\documentclass[preprint, 12pt]{siamonline190516}

\usepackage[english]{babel}
\usepackage[latin1]{inputenc}
\usepackage{amsmath,amssymb,latexsym,amsfonts}
\usepackage{bm}
\usepackage{bbm}
\usepackage{mathtools}
\usepackage{mathrsfs}
\usepackage{graphics}

\usepackage{algorithm}
\usepackage{algorithmic}
\usepackage{color}
\usepackage{float}
\usepackage{floatpag}
\floatpagestyle{empty}
\usepackage{booktabs}

\usepackage{graphicx} \usepackage{verbatim}
\usepackage[margin=1in]{geometry}
\usepackage[font=footnotesize,labelfont=footnotesize]{caption}
\usepackage[font=scriptsize,labelfont=scriptsize]{subcaption}
\usepackage{enumerate}
\usepackage{tabulary}
\usepackage{multirow}
\usepackage{adjustbox}
\allowdisplaybreaks
\usepackage[bb=boondox]{mathalfa}
 
\renewenvironment{equation*}{\[}{\]\ignorespacesafterend}

\newcommand{\C}{\mathcal{C}}

\newcommand{\DSD}{\mathcal{D}}

\newcommand{\Lrw}{L_{\text{RW}}}

\newcommand{\sinf}{s^{\infty}}
\newcommand{\0}{\mathbb{0}}
\newcommand{\1}{\mathbbm{1}}

\newcommand{\numclust}{K} 
\newcommand{\Dbtw}{D_{t}^{\text{btw}}}
\newcommand{\Din}{D_{t}^{\text{in}}}
\newcommand{\G}{\mathcal{G}}

\DeclareMathOperator*{\Lsym}{L_{SYM}}

\usepackage{hyperref}

\newtheorem{thm}{Theorem}[section]
\newtheorem{cor}[thm]{Corollary}

\newtheorem{defn}[thm]{Definition}

\makeatletter
\newcommand{\distas}[1]{\mathbin{\overset{#1}{\kernZ@\sim}}}%
\newsavebox{\mybox}\newsavebox{\mysim}
\newcommand{\distras}[1]{%
  \savebox{\mybox}{\hbox{\kern3pt$\scriptstyle#1$\kern3pt}}%
  \savebox{\mysim}{\hbox{$\sim$}}%
  \mathbin{\overset{#1}{\kernZ@\resizebox{\wd\mybox}{\ht\mysim}{$\sim$}}}%
}
\makeatother

\makeatletter
\let\c@equation\c@thm
\makeatother
\numberwithin{equation}{section}

\makeatother

\begin{document}

\title{Diffusion State Distances: Multitemporal Analysis, Fast Algorithms, and Applications to Biological Networks}

\author{Lenore Cowen\thanks{Department of Computer Science, Tufts University, Medford, MA 02155, USA \email{(cowen@cs.tufts.edu,kapil.devkota@tufts.edu})}
	\and Kapil Devkota\footnotemark[1]
	\and Xiaozhe Hu\thanks{Department of Mathematics, Tufts University, Medford, MA 02155, USA (\email{xiaozhe.hu@tufts.edu, jm.murphy@tufts.edu, kaiyi.wu@tufts.edu})}
        \and James M. Murphy\footnotemark[2]
        \and Kaiyi Wu\footnotemark[2]
        }
\maketitle

\begin{abstract}
Data-dependent metrics are powerful tools for learning the underlying structure of high-dimensional data.  This article develops and analyzes a data-dependent metric known as \emph{diffusion state distance (DSD)}, which compares points using a data-driven diffusion process.  Unlike related diffusion methods, DSDs incorporate information across time scales, which allows for the intrinsic data structure to be inferred in a parameter-free manner.  This article develops a theory for DSD based on the multitemporal emergence of mesoscopic equilibria in the underlying diffusion process.  New algorithms for denoising and dimension reduction with DSD are also proposed and analyzed. These approaches are based on a weighted spectral decomposition of the underlying diffusion process, and experiments on synthetic datasets and real biological networks illustrate the efficacy of the proposed algorithms in terms of both speed and accuracy.  Throughout, comparisons with related methods are made, in order to illustrate the distinct advantages of DSD for datasets exhibiting multiscale structure.
\end{abstract}

\section{Introduction}

Metrics for pairwise comparisons of data points $X=\{x_{i}\}_{i=1}^{n}\subset\mathbb{R}^{D}$ are an essential tool in a wide range of data analysis tasks including classification, regression, clustering, and visualization \cite{Friedman2001_Elements}.  Euclidean distances, and more generally data-independent metrics that depend only on the original coordinates of the data, may be inadequate for high-dimensional data due to the curse of dimensionality \cite{Vershynin2018_High} or for low-dimensional data with nonlinear correlations.  In order to address these concerns, metrics derived from the global structure of the data are necessary.  

A family of data-dependent metrics based on local similarity graphs has been developed to address this problem \cite{Klein1993_Resistance, Tenenbaum2000_Global, Roweis2000_Nonlinear, Belkin2002_Laplacian, Belkin2003_Laplacian, Donoho2003_Hessian, Coifman2005_Geometric, Coifman2006_Diffusion}.  These methods typically construct an undirected, weighted graph $\G$ with nodes corresponding to $X$ and weights between nodes $x_{i},x_{j}$ given by $W_{ij}=\mathcal{K}(x_{i},x_{j})$ for a suitable kernel function that is usually radial and rapidly decaying.  Though constituted from \emph{local} relationships among the points of $X$ (due to the rapid decay of $\mathcal{K}$), the \emph{global} features of $X$ may be gleaned from $\G$ by considering partial differential operators on $\G$ (e.g. Laplacian or Schr\"odinger operators), geodesics in the path space, or diffusion processes on $X$.  These new metrics may vastly improve over data-independent metrics for a range of tasks including supervised classification and regression, unsupervised clustering, low-dimensional embeddings, and the visualization of high-dimensional data. 

While powerful, these methods may not adequately capture multiscale structure in the data.  For example, when there are both small and large clusters in the underlying data, or geometric features at multiple scales, it may be challenging to select without supervision the optimal low-dimensional representation or parameters for these data-dependent metrics.  In particular, when considering diffusion processes on graphs, the time scale crucially determines the granularity of the subsequent analysis.  Indeed, for a suitable Markov random matrix $P\in\mathbb{R}^{n\times n}$ derived from an underlying graph $\G$ on data $X$, $P^{t}$ captures latent structure in the data.  Small scale features are homogenized in favor of coarse structures as $t$ increases \cite{Nadler2007_Fundamental, Maggioni2019_Learning} before $P^{t}$ converges to the rank one matrix corresponding to the stationary distribution as $t\rightarrow\infty$.  Developing diffusion methods that efficiently account for multiscale structure in the data is essential both to mitigate the  practical dependence on the time parameter $t$ and to adequately capture both fine and coarse-scale structures.  The recently proposed \emph{diffusion state distances (DSD)} \cite{Cao2013_Going, Cao2014_New} account for multiscale structure of the underlying data by aggregating the behavior of the diffusion process across time scales.  DSD has shown state-of-the-art performance for the analysis of protein-protein interaction networks (PPI) \cite{Choobdar2019_Assessment}, and randomized methods have allowed it to scale to large networks \cite{Lin2018_Computing}.  

\subsection{Summary of Contributions}

This article makes three major contributions.  First, \emph{a multitemporal theory for DSD is developed}.  Unlike related graph-based methods such as Laplacian eigenmaps and diffusion distances, DSDs incorporate information across all time scales of a data-driven diffusion process.  Our analysis illustrates that diffusion patterns persisting across long time scales contribute more to the DSDs than diffusion patterns that emerge and disappear rapidly; in this sense, DSDs \emph{aggregate geometric structures across time scales}.  This is of particular importance for data that exhibits multiscale cluster structure, in which a hierarchy of cluster patterns across time scales emerge.  We prove that the DSD synthesizes the resultant mesoscopic equilibria into a single metric that accounts for the temporal longevity of these patterns.

Second, we \emph{develop a dimension-reduction framework for the DSD}, in which the data is projected into a small number of coordinates, such that Euclidean distances in the embedded space approximate DSD in the original space.  Based on well-known ideas pertaining to the graph Laplacian and random walks on graphs, this embedding serves several important purposes.  First, it \emph{allows for the fast, faithful computation of the DSD} by considering only the top eigenvectors of the underlying random walk.  Second, it \emph{denoises the DSD} by truncating the eigenexpansion and discarding high-frequency eigenvectors that are typically corrupted by noise in the finite sample setting.  Finally, it yields an \emph{interpretation of the DSD in terms of the principal eigenvectors of the random walk graph Laplacian}, which complements the proposed analysis in terms of mesoscopic equilibria.  Numerical experiments verify the usefulness of this low-dimensional coordinate embedding.  

Third, the proposed algorithms are deployed for the \emph{analysis of real protein-protein interaction (PPI) networks}.  The impact of the spectral decomposition is shown to be substantial, yielding improved empirical performance and also reduced runtime.  DSD are known to be at the state-of-the-art for analysis of PPI networks, so improvements in their accuracy and complexity are a boon to better understanding of these large, complex networks.

The article is organized as follows.  Background on diffusion processes on graphs and related data-dependent metrics are discussed in Section \ref{sec:Background}.  DSD are analyzed in a multitemporal framework in Section \ref{sec:MultitemporalAnalysis}.  The spectral formulation of DSD and related dimension reduction and denoising techniques are proposed and analyzed in Section \ref{sec:SpectralAnalysis}.  Numerical experiments on synthetic and real biological networks are shown in Section \ref{sec:NumericalExperiments}.  Conclusions and future research directions are discussed in Section \ref{sec:FutureDirections}.

Throughout, we will use the notation presented in Table \ref{tab:Notation}.

\begin{table}[!htb]
\begin{center}
\begin{tabular}{|r c | p{10cm} }
\hline
Notation & Meaning\\
\hline
$X=\{x_{i}\}_{i=1}^{n}\subset\mathbb{R}^{D}$ & data points\\ 
$\G=(X,W)$ & graph $\G$ with nodes corresponding to $X$ and weights W\\
$D$ & degree matrix for weight matrix $W$\\
$\Lrw, \Lsym$ & random walk Laplacian, symmetric normalized Laplacian \\
$M^{\dagger}$ & pseudoinverse of a matrix $M$\\
 $P$ & Markov transition matrix \\
 $S$ & matrix of stochastic complements \\
 $\pi$ & stationary distribution of $P$\\
 $t$ & time scale of random walk $P$\\
 $p_{t}(x_{i},x_{j})$ & $(P^{t})_{ij}$, the probability of transition from $x_{i}$ to $x_{j}$ at time $t$\\
$\{(\phi_{\ell},\lambda_{\ell})\}_{\ell=1}^{n}$ & spectral decomposition of $\Lsym$ \\
 $\ell^{p}(w)$ & space of $p$-integrable sequences with respect to the weight $w$; $\ell^{p}=\ell^{p}(\1)$\\ 
 $\| x \|_{\ell^{p}(w)}$ & the $\ell^{p}$ norm of $x$ with respect to the weight $w$; $\| x \|_{p}=\| x \|_{\ell^{p}(\1)}$\\
$ \|M\|_{\infty}$ & $\max_{1\le i \le m}\sum_{j=1}^{n}|M_{ij}|, \ M\in\mathbb{R}^{m\times n}$.\\
 $\0$, $\1$ & all 0s or 1s vector or matrix, depending on context\\
 $e_{i}$ & the vector $(0,0,\dots,0,1,0,\dots,0)$ with a 1 in the $i^{th}$ coordinate\\
 $\rho$ & general metric\\
 $D_{t}(x_{i},x_{j})$ & diffusion distance between $x_{i}$ and $x_{j}$ at time $t$\\
 $\DSD (x_{i},x_{j})$ & DSD between $x_{i}$ and $x_{j}$\\
 \hline
\end{tabular}
\end{center}
\caption{Notation used throughout the article.\label{tab:Notation}}

\end{table}

\section{Background on Diffusion Geometry and Related Data-Dependent Metrics}
\label{sec:Background}

Given a dataset $X=\{x_{i}\}_{i=1}^{n}\subset\mathbb{R}^{D}$, a natural notion of \emph{diffusion on $X$} is given by random walks on a weighted graph $\G$ with nodes corresponding to the points of $X$ and edges between $x_{i}, x_{j}\in X$ with weight $W_{ij}\in [0,1]$.  The matrix $W\in\mathbb{R}^{n\times n}$ is the \emph{weight matrix}, and is symmetric in the case that $\G$ is an undirected graph.  It is common to construct $W$ using a (radial, rapidly decaying) kernel $\mathcal{K}:\mathbb{R}^{D}\times\mathbb{R}^{D}\rightarrow [0,1]$ as $W_{ij}=\mathcal{K}(x_{i},x_{j})$.  A common choice of kernel is $\mathcal{K}(x_{i},x_{j})=\exp(-\rho(x_{i},x_{j})^{2}/\sigma^{2})$, where $\rho:\mathbb{R}^{D}\times\mathbb{R}^{D}\rightarrow [0,\infty)$ is a metric, for example $\rho(x_{i},x_{j})=\|x_{i}-x_{j}\|_{2}$, and $\sigma>0$ is a scaling parameter.  Given $W$, one can construct a diffusion process with underlying state space $X$ by normalizing $W$ to be row stochastic: $P=D^{-1}W$, where $D\in\mathbb{R}^{n\times n}$ is the diagonal degree matrix with $D_{ii}=\sum_{j=1}^{n}W_{ij}$.  Under the assumptions that $P$ is irreducible (i.e. $\G$ is a connected graph) and aperiodic, $P$ admits a unique stationary distribution $\pi$ satisfying $\pi P=\pi$ \cite{Chung1997}. 

The Markov process encoded by $P$ evolves with time.  The $i^{th}$ row of $P^{t}$ encodes the transition probabilities of $x_{i}$ at time $t$.  \emph{Diffusion distances} \cite{Coifman2005_Geometric, Coifman2006_Diffusion} use the probability profiles encoded by $P^{t}$ to compute pairwise distances between points in the underlying dataset $X$:

\begin{defn}\label{defn:DD}Let $P$ be an irreducible, aperiodic Markov transition matrix on $X=\{x_{i}\}_{i=1}^{n}\subset\mathbb{R}^{D}$ with unique stationary distribution $\pi$.  Let $p_{t}(x_{i},x_{j})=(P^{t})_{ij}$.  The \emph{diffusion distance} between $x_{i}, x_{j}$ at time $t$ is \begin{equation}\label{eqn:DD}D_{t}(x_{i},x_{j})=\|(e_{i}-e_{j})P^{t}\|_{\ell^{2}(1/\pi)}=\sqrt{\sum_{\ell=1}^{n}\left(p_{t}(x_{i},x_{\ell})-p_{t}(x_{j},x_{\ell})\right)^{2}\frac{1}{\pi(\ell)}},\end{equation}where $e_{i}$ is the $i^{th}$ canonical basis vector of $\mathbb{R}^{n}$.
\end{defn}

Diffusion distances are parametrized by a time scale parameter $t$, which corresponds to how long the diffusion process has run on the underlying graph.  Note that $\lim_{t\rightarrow\infty}P^{t}=\1\pi$, so that for all $i,j,\ell$, $\lim_{t\rightarrow\infty}(p_{t}(x_{i},x_{\ell})-p_{t}(x_{j},x_{\ell}))^{2}=0$; the rate of convergence is uniform and exponential in the second largest eigenvalue of $P$ \cite{Maggioni2019_Learning}.  In particular, $D_{t}(x_{i},x_{j})\rightarrow 0$ at an exponential rate that is uniform in $x_{i},x_{j}$.  Equivalently, $\lim_{t\rightarrow\infty}\pi_{0}P^{t}=\pi$ for any choice of initial probability distribution $\pi_{0}$.  For $t$ not too small and not too large, $D_{t}$ accounts for mesoscopic cluster structure in the data \cite{Maggioni2019_Learning}.  For data sampled from a common manifold, there is an asymptotic relationship as $n\rightarrow\infty$ between the scale parameter $\sigma$ used for constructing the underlying graph and the time parameter $t$ \cite{Coifman2006_Diffusion}.  Diffusion distances have been applied to a range of problems including molecular dynamics \cite{Rohrdanz2011_Determination, Zheng2011_Polymer}, learning of dynamical systems \cite{Nadler2006_Diffusion, Coifman2008_Diffusion, Singer2009_Detecting}, latent variable estimation \cite{Lederman2018_Learning, Katz2019_Alternating, Shnitzer2019_Recovering}, remote sensing image processing \cite{Czaja2016_Fusion, Murphy2018_Iterative, Murphy2019_Unsupervised, Murphy2019_Spectral}, and medical signal processing \cite{Lederman2015_Alternating, Wu2014_Assess, Li2017_Efficient, Alagapan2018_Diffusion}.

The diffusion distance is one of a family of powerful data-dependent metrics, including resistance distance \cite{Klein1993_Resistance}, Laplacian eigenmaps \cite{Belkin2002_Laplacian, Belkin2003_Laplacian}, path distances \cite{Fischer2003_Path, Borgwardt2005_Shortest, Little2020_Path} and related methods of nonlinear dimension reduction \cite{Tenenbaum2000_Global, Donoho2003_Hessian, Maaten2008_Visualizing}.  It is natural to compare diffusion distances with Laplacian eigenmaps, as both may be understood in terms of the spectral decomposition of an underlying operator on the graph.  Indeed, although $P$ is not symmetric, it is diagonally conjugate to the symmetric matrix $D^{1/2}PD^{-1/2}=D^{-1/2}WD^{-1/2}$, which admits spectral decomposition $\{(\lambda_{\ell},\phi_{\ell})\}_{\ell=1}^{n}$, where $1=\lambda_{1}>|\lambda_{2}|\ge\dots\ge|\lambda_{n}|\ge 0$.  Then $P$ has the same eigenvalues as $W$, with left and right eigenvectors $\{\varphi_{\ell}\}_{\ell=1}^{n}$ and $\{\psi_{\ell}\}_{\ell=1}^{n}$ where $\forall \ell=1,\dots,n$, $\varphi_{\ell}=\sqrt{\pi}\phi_{\ell}, \quad \psi_{\ell}=\phi_{\ell}/\sqrt{\pi},$ respectively.  This allows (\ref{eqn:DD}) to be written in terms of $\{(\lambda_{\ell},\psi_{\ell})\}_{\ell=1}^{n}$: \[D_{t}(x_{i},x_{j})=\sqrt{\sum_{\ell=1}^{n}\lambda_{\ell}^{2t}(\psi_{\ell}(x_{i})-\psi_{\ell}(x_{j}))^{2}},\]where $\psi_{\ell}(x_{i})=(\psi_{\ell})_{i}$.  In this sense, diffusion distances are Euclidean distances in the $n$-dimensional coordinate system $x_{i}\mapsto(\lambda_{1}^{t}\psi_{1}(x_{i}),\lambda_{2}^{t}\psi_{2}(x_{i}),\dots,\lambda_{n}^{t}\psi_{n}(x_{i})).$  Moreover, these \emph{diffusion maps} may be truncated after $M\le n$ terms, yielding a low-dimensional coordinate representation:  \begin{align}\label{eqn:DiffusionMaps_LowDimensional}x_{i}\mapsto (\lambda_{1}^{t}\psi_{1}(x_{i}),\lambda_{2}^{t}\psi_{2}(x_{i}),\dots,\lambda_{M}^{t}\psi_{M}(x_{i})),
\end{align} and a corresponding approximation of the diffusion distances as \[D_{t}(x_{i},x_{j})\approx\sqrt{\sum_{\ell=1}^{M}\lambda_{\ell}^{2t}(\psi_{\ell}(x_{i})-\psi_{\ell}(x_{j}))^{2}}.\]It is natural to compare (\ref{eqn:DiffusionMaps_LowDimensional}) to the embedding of Laplacian eigenmaps, namely $x_{i}\mapsto \{\phi_{\ell}(x_{i})\}_{\ell=1}^{M}$.  Indeed, diffusion maps differs primarily in its weighting of the eigenvectors by the eigenvalues; $\psi_{\ell}$ differs from $\phi_{\ell}$ only by normalizing by the square root of the stationary distribution.  As $t$ increases, only the lowest-frequency eigenvectors (i.e. those with largest eigenvalues in modulus) contribute to the low dimensional representation.  A detailed analysis of the impact of time has been developed by linking the Markov transition matrix to the Fokker-Planck stochastic differential equation in the continuum limit \cite{Nadler2007_Fundamental}, and in terms of latent cluster structure in discrete data \cite{Maggioni2019_Learning}.

\subsection{Diffusion State Distances}
\label{subsec:DSD_Background}

The \emph{diffusion state distance (DSD)} is a graph-driven metric originally proposed in the context of PPI networks \cite{Cao2013_Going}, and generalized to graphs with weighted edges in \cite{Cao2014_New}.  It was designed to capture complex structures in a manner robust to high-degree nodes that ruin the discriminative ability of classical shortest-path metrics in small world networks \cite{Newman2018_Networks}.  

\begin{defn}\label{defn:DSD} Let $P$ be a Markov transition matrix on $X=\{x_{i}\} _{i=1}^{n}\subset\mathbb{R}^{D}$.  The \emph{diffusion state distance between $x_{i}$ and $x_{j}$ with respect to the weight $w$} is \begin{equation}\label{eqn:DSD}\DSD (x_{i},x_{j})=\left\|(e_{i}-e_{j})\sum_{t=0}^{\infty}P^{t}\right\|_{\ell^{2}(w)}=\sqrt{\sum_{\ell=1}^{n}\left(\sum_{t=0}^{\infty}P^{t}_{i\ell}-P^{t}_{j\ell}\right)^{2}w(\ell)},\end{equation}where $e_{i}$ is the $i^{th}$ canonical basis vector of $\mathbb{R}^{n}$.  

\end{defn}

We note that \cite{Cao2013_Going} originally introduced the DSD in the $\ell^{1}$ norm; for purposes of dimension reduction, convergence, and comparability with related data-driven distances, we formulate Definition \ref{defn:DSD} in the $\ell^{2}$ norm.  Moreover, DSD was originally proposed with $w=\1$.  We introduce the weighting by $w$, which will be helpful when considering the connection between DSD and the spectral decomposition of $P$.

The DSD may be computed in terms of a regularized inverse Laplacian \cite{Cao2013_Going}.  Indeed, let $\Lrw=I-P=I-D^{-1}W$ be the random walk graph Laplacian \cite{Chung1997}.  Expanding in Neumann series, \begin{align}\begin{split}\label{eqn:NeumannSeries}(e_{i}-e_{j})(\Lrw+\1\pi)^{-1}=&(e_{i}-e_{j})(I-P+\1\pi)^{-1}\\
=&(e_{i}-e_{j})\sum_{t=0}^{\infty}(P-\1\pi)^{t}\\
=&(e_{i}-e_{j})\sum_{t=0}^{\infty}(P^t-\1\pi)\\
=&(e_{i}-e_{j})\sum_{t=0}^{\infty}P^t,\end{split}\end{align}where the final equality follows from the fact that for all $i,j$, $(e_{i}-e_{j})\1\pi=\pi-\pi=0$.  In particular, with weight vector $w=\1$, $\DSD (x_{i},x_{j})=\|(e_{i}-e_{j})(\Lrw+\1\pi)^{-1}\|_{2}.$

\subsubsection{Comparison of DSD to Related Metrics}

Unlike diffusion distances, which fix a time scale $t$, the DSD sums across all time scales.  Intuitively, this suggests that the DSD synthesizes multitemporal structures (e.g. clusters) in the data so that structures that persist across long time scales are emphasized, while structures that persist across shorter time scales are less significant.  This intuition is made precise by multiscale cluster analysis in Section \ref{sec:MultitemporalAnalysis}.  The lack of a time parameter may be considered an advantage of DSD over diffusion distances, particularly in unsupervised machine learning when no training data is available for cross validation to determine a good choice of $t$.

We note that both diffusion distances and DSD overcome an important weakness of shortest path metrics \cite{Tenenbaum2000_Global, Little2020_Path}, which are not discriminative in small world networks when all points have short paths between them.  In some sense, diffusion distances and DSD average across all paths in the data, so that the existence of a single short path between two points is insufficient for them to be close in these metrics.  

The notions of resistance and commute distance between $x_{i}$ and $x_{j}$ are also related to the DSD.  The commute distance between $x_{i},x_{j}$ is the expected time it takes to go from $x_{i}$ to $x_{j}$ and back to $x_{i}$, according to the random walk $P$.  This can be computed as $\mathscr{C}(x_{i},x_{j})=\langle e_{i}-e_{j},\Lsym^{\dagger}(e_{i}-e_{j})\rangle$, where $\Lsym^{\dagger}$ is the pseudoinverse of $\Lsym$.  While intuitively appealing, under several models of random graphs (e.g. Erd\H{o}s-R\'{e}yni graphs, $k$-NN graphs, and $\epsilon$-graphs), $\mathcal{C}(x_{i},x_{j})\rightarrow C_{n}|\frac{1}{D_{i}}-\frac {1}{D_{j}}|$ as the number of nodes $n\rightarrow\infty$, where $D_{i}$ is the degree of $x_{i}$ and $C_{n}$ is a scaling constant dependent on $n$ and the ambient structure of the random graph model, but is crucially independent of $x_{i}, x_{j}$ \cite{VonLuxburg2014_Hitting}.  This suggests that the commute distance is not an informative metric for statistical and machine learning on these graph models for $n$ sufficiently large, as it degenerates to the relative difference in degrees, which is completely local and captures no interesting structures or patterns in the data.  On the other hand, DSD are derived from diffusion processes on graphs, which converge in a meaningful sense to solutions of certain Fokker-Planck equations as $n\rightarrow\infty$, under suitable scalings of $\sigma\rightarrow 0^{+}$ \cite{Nadler2006_Diffusion_NIPS, Nadler2006_Diffusion}.  We shall further compare DSD to commute distances in Section \ref{subsec:LargeSampleLimits}

\section{Diffusion State Distances and Mesoscopic Equlibria}
\label{sec:MultitemporalAnalysis}

DSDs sum across all time scales of a diffusion operator, and it is natural to analyze the role of this time parameter in terms of the multitemporal behavior of the underlying Markov chain.  Indeed, in order to prove that (time dependent) diffusion distances $D_{t}$ capture cluster structure of data $X$ at a particular time scale, the underlying Markov matrix $P\in\mathbb{R}^{n\times n}$ may be analyzed hierarchically \cite{Maggioni2019_Learning}.  More precisely, $P$ may be decomposed as 

\begin{equation}
P=\begin{bmatrix}\label{eqn:Ppartition}
    P_{11}  & P_{12}  & \dots & P_{1\numclust} \\
   P_{21}  & P_{22}  & \dots & P_{2\numclust} \\
   \vdots & \vdots & \ddots & \vdots \\
   P_{\numclust 1}  & P_{\numclust 2}  & \dots & P_{\numclust \numclust}
\end{bmatrix},
\end{equation}where each $P_{kk}$ is square and $\numclust\le n$.  Let $C_{k}\subset\{1,\dots,n\}$ be the indices corresponding to $P_{kk}$.  Intuitively, if the mass of $P$ concentrates on the diagonal blocks $\{P_{kk}\}_{k=1}^{\numclust}$, then it is unlikely that the random walker transitions in short time between the $\{C_{k}\}_{k=1}^{\numclust}$.  Moreover, if these blocks are in some sense coherent and indivisible, then the random walker explores them quickly, then waits a long time to transition to another block.  

In order to make this intuition precise, we introduce the notion of \emph{stochastic complement} \cite{Meyer1989, Maggioni2019_Learning}. 

\begin{defn}  Let $P\in\mathbb{R}^{n\times n}$ be an irreducible Markov matrix partitioned as in (\ref{eqn:Ppartition}).  For a given index $k\in\{1,\dots,\numclust\}$, let $P_{k}$ denote the principal block submatrix generated by deleting the $k^{\text{th}}$ row and $k^{\text{th}}$ column of blocks from (\ref{eqn:Ppartition}), and let \begin{align*}P_{*k}&=\begin{bmatrix}P_{1k}   P_{2k} \dots P_{k-1,k}  P_{k+1,k} \dots  P_{\numclust k}\end{bmatrix}^{\top}, \\P_{k*}&=\begin{bmatrix} P_{k1} \ P_{k2}\  \dots \ P_{k,k-1} \ P_{k,k+1} \ \dots \ P_{k\numclust} \end{bmatrix}.\end{align*}  The \emph{stochastic complement of $P_{kk}$} is the matrix $S_{kk}=P_{kk}+P_{k*}(I-P_{k})^{-1}P_{*k}.$
\end{defn}

The stochastic complement $S_{kk}$ is the transition matrix for a Markov chain with state space $C_{k}$ obtained from the original chain by re-routing transitions into or out of $C_{k}$.  More precisely, the reduced chain $S_{kk}$ captures two possibilities for transitions between points in $C_{k}$: a transition is either direct in $C_{k}$ or indirect by first exiting $C_{k}$, then moving through points in $C_{k}^{c}$, then finally entering back into $C_{k}$ at some future time.  Indeed, the term $P_{k*}(I-P_{k})^{-1}P_{*k}$ in the definition of $S_{kk}$ accounts for leaving $C_{k}$ (the factor $P_{k*}$), traveling for some time in $C_{k}^{c}$ (the factor $(I-P_{k})^{-1}$), then re-entering $C_{k}$ (the factor $P_{*k}$).  Note that the factor $(I-P_{k})^{-1}$ may be expanded in Neumann series as $(I-P_{k})^{-1}=\sum_{t=0}^{\infty}P_{k}^{t},$ showing that it accounts for exiting and returning to $I_k$ after a finite but otherwise arbitrary number of steps outside of it.

The notion of stochastic complement quantifies the emergence of \emph{mesoscopic equilibria} of $P$.  We say $P$ is \emph{primitive} if it is non-negative, irreducible and aperiodic.  

\begin{thm}\label{thm:NearReducibility}\cite{Meyer1989}  Suppose $P\in\mathbb{R}^{n\times n}$ is an irreducible row-stochastic matrix partitioned into $K^{2}$ square block matrices as in (\ref{eqn:Ppartition}) and let $S$ be the completely reducible row-stochastic matrix with the stochastic complements of the diagonal blocks of $P$ on the diagonal:

\[ 
P=\begin{bmatrix}
    P_{11}  & P_{12}  & \dots & P_{1\numclust} \\
   P_{21}  & P_{22}  & \dots & P_{2\numclust} \\
   \vdots & \vdots & \ddots & \vdots \\

   P_{\numclust 1}  & P_{\numclust 2}  & \dots & P_{\numclust \numclust }
\end{bmatrix}, \
\
S=\begin{bmatrix}
    S_{11}  & \0  & \dots & \0 \\
   \0  & S_{22}  & \dots & \0 \\
   \vdots & \vdots & \ddots & \vdots \\

   \0  & \0  & \dots & S_{\numclust \numclust}
\end{bmatrix}.\]Let each $S_{kk}$ be primitive, so that the eigenvalues of $S$ are such that $\lambda_{1}=\lambda_{2}=\dots=\lambda_{\numclust}=1>\lambda_{\numclust+1}\ge \lambda_{\numclust+2}\ge \dots>-1.$  Let $Z$ diagonalize $S$, and let \[ 
S^{\infty}=\lim_{t\rightarrow\infty}S^{t}=\begin{bmatrix}
    \1 \pi^{1}  & \0  & \dots & \0 \\
   \0  & \1\pi^{2}  & \dots & \0 \\
   \vdots & \vdots & \ddots & \vdots \\

   \0  & \0  & \dots & \1\pi^{K}
\end{bmatrix},\] for $\pi^{k}$ the stationary distribution of $S_{kk}$.  Then $\|P^{t}-S^{\infty}\|_{\infty}\le \delta t + \kappa|\lambda_{\numclust+1}|^{t},$ where $\delta=2\max_{k}\|P_{k*}\|_{\infty}$ and $\kappa=\|Z^{-1}\|_{\infty}\|Z\|_{\infty}$.  
Moreover, for any initial distribution $\pi_{0}$ and $s=\lim_{t\rightarrow\infty}\pi_{0}S^{t}=\pi_0S^\infty$, $\|\pi_{0}P^{t}-s\|_{1}\le \delta t +\kappa|\lambda_{\numclust+1}|^{t}.$
\end{thm}A proof of this result appears in the Appendix.

In the estimate $\|P^{t}-S^{\infty}\|_{\infty}\le \delta t+\kappa|\lambda_{\numclust+1}|^{t}$, the right hand side consists of two terms, which reference different aspects of the underlying dynamics of $P$.  The $\delta t$ term bounds $\|P^{t}-S^{t}\|_{\infty}$, which accounts for the approximation of $P^{t}$ by the reducible Markov chain $S^{t}$.  In the context in which the partition $\{C_{k}\}_{k=1}^{\numclust}$ is interpreted as an unsupervised clustering of the data $X$, this term accounts for the between-cluster connections in $P$.  The term $\kappa |\lambda_{\numclust+1}|^{t}$ bounds $\|S^{t}-S^{\infty}\|_{\infty}$, which accounts for the propensity of mixing within the distinct elements of the partition $\{C_{k}\}_{k=1}^{\numclust}$.  In the clustering context, this term quantifies the within-cluster connections.  

Theorem \ref{thm:NearReducibility} implies that for $\epsilon>0$ large enough, there is a range of $t$ for which the dynamics of $P^{t}$ are $\epsilon$-close to the dynamics of the reducible, stationary Markov chain $S^{\infty}$:

\begin{cor}\label{cor:CriticalTimeRange}Let $P,S^{\infty}, s,\lambda_{\numclust+1}, \delta, \kappa$ be as in Theorem \ref{thm:NearReducibility}.  Suppose that for some $\epsilon>0$, \[{\ln\left(\frac{2\kappa}{\epsilon}\right)}/{\ln\left(\frac{1}{|\lambda_{\numclust+1}|}\right)}<t<\frac{\epsilon}{2\delta}.\]  Then $\|P^{t}-S^{\infty}\|_{\infty}<\epsilon,$ and for every initial distribution $\pi_{0}$, $\|\pi_{0}P^{t}-s\|_{1}<\epsilon.$
\end{cor}

The values $\lambda_{\numclust+1}, \delta, \kappa$ may be understood as latent geometric parameters of the dataset and the underlying partition, which determine the range of times $t$ at which mesoscopic equilibria are reached with respect to the partition $\{C_{k}\}_{k=1}^{\numclust}$.  More precisely, $\delta,\kappa$ converge as $n\rightarrow\infty$ to natural quantities independent of $n$, and \cite{Trillos2019_Error} proved that as $n\rightarrow\infty$, there is a scaling for $\sigma\rightarrow 0^{+}$ in which the (random) eigenvalues of $P$ converge to the (deterministic) eigenvalues of a corresponding continuum operator.  Thus, the parameters of Theorem \ref{cor:CriticalTimeRange} may be understood as random fluctuations of geometrically intrinsic quantities which admit the following interpretations:

\vspace{5pt}

\begin{itemize}
\item $\lambda_{\numclust+1}$ is the largest eigenvalue of $S$ not equal to 1.  Since $S$ is block diagonal and each $S_{kk}$ is primitive, it follows that $\lambda_{\numclust+1}=\max_{k=1,\dots,\numclust}\lambda_{2}(S_{kk})$, where $\lambda_{2}(S_{kk})$ is the second largest (and first not equal to 1) eigenvalue of $S_{kk}$.  The conductance \cite{Chung1997} and the mixing time \cite{Levin2009} of the random walk restricted to $S_{kk}$ are closely related to $\lambda_{2}(S_{kk})$.\vspace{5pt}

\item The quantity $\delta=2\max_{k=1,\dots,\numclust}\|P_{k*}\|_{\infty}$ is controlled by the largest interaction between the clusters $\{C_{k}\}_{k=1}^{\numclust}$.  \vspace{5pt}

\item The quantity $\kappa=\|Z^{-1}\|_{\infty}\|Z\|_{\infty}$, with $Z=\left(\phi_{1}|\dots|\phi_{n}\right)$, is  a measure of the condition number of $S$. If $Z,Z^{-1}$ are orthogonal matrices, then each row of $Z,Z^{-1}$ has $\ell^{2}$ norm 1, hence  $\kappa\le n$.  
\end{itemize}

\vspace{5pt}

Analysis of the stochastic complement shows that, for some range of time $t\in [\tau_{1},\tau_{2}]$, the diffusion distances at time $t$ for points in the same $C_{k}$ are uniformly smaller than the diffusion distances between any points in distinct $C_{k}$ \cite{Maggioni2019_Learning}.  Indeed, let 

\[\Din=\max_{k=1,\dots,\numclust}\max_{x,y\in C_{k}}D_{t}(x,y), \quad \Dbtw=\min_{\substack{k,k'=1,\dots,\numclust, \\ k\neq k'}}\min_{x\in C_{k}, y\in C_{k'}}D_{t}(x,y).\]Theorem \ref{thm:LUND} is stated in $\ell^{2}$, which necessitates the following definition.

\begin{defn}
Let $P, S^{\infty}\in \mathbb{R}^{n\times n}$ be as in Theorem \ref{thm:NearReducibility} and set $p_{t}(x_{i},x_{j})=(P^{t})_{ij}, \ s^{\infty}(x_{i},x_{j})=(S^{\infty})_{ij}$.  Define \[\gamma(t)=\max_{x\in X}\left(1-\frac{1}{2}\sum_{u\in X}\left|\frac{|p_{t}(x,u)-\sinf(x,u)|}{\|p_{t}(x,\cdot)-\sinf(x,\cdot)\|_{2}}-\frac{1}{\sqrt{n}}\right|^{2}\right)^{-1}.\]
\end{defn}It is not hard to show \cite{Botelho2019_Exact} that for any vector $u\in\mathbb{R}^{n}$, $\|u\|_{2}=\frac{c_{u}}{\sqrt{n}}\|u\|_{1},$ where \[c_{u}=\left(1-\frac{1}{2}\sum_{i=1}^{n}\left|\frac{|u_{i}|}{\|u\|_{2}}-\frac{1}{\sqrt{n}}\right|^{2}\right)^{-1}.\]  So, $\gamma(t)$ is the maximum of $c_{u}$ when $u$ is chosen among the rows of $P^{t}-S^{\infty}$.  In this sense, $\gamma(t)\in [1, \sqrt{n}]$ measures how the $\ell^{1}$ norm differs from the $\ell^{2}$ norm across all rows of $P^{t}-S^{\infty}$.

\begin{thm}\label{thm:LUND}\cite{Maggioni2019_Learning}  Let $X$ be data partitioned as $\{C_{k}\}_{k=1}^{\numclust}$ and let $P$ be a corresponding Markov transition matrix on $X$.  Let $\delta, \lambda_{\numclust+1}, \kappa,S^{\infty}$ be as in Theorem \ref{thm:NearReducibility} and let $s^{\infty}(x_{i},x_{j})=(S^{\infty})_{ij}$. Let $D_{t}$ be the diffusion distance associated to $P$ and counting measure $\nu$.  If $t,\epsilon$ satisfy \[\frac{\ln\left(\frac{2\kappa}{\epsilon}\right)}{\ln\left(\frac{1}{\lambda_{\numclust+1}}\right)}<t<\frac{\epsilon}{2\delta}\,,\] then $\displaystyle\Din\le 2\frac{\epsilon}{\sqrt{n}}\gamma(t)$ and $\displaystyle\Dbtw\ge 2\left(\min_{y\in X}\|\sinf(y,\cdot)\|_{\ell^{2}(\nu)}-\frac{\epsilon}{\sqrt{n}}\gamma(t)\right)$.

\end{thm}

These theorems are powerful tools for guaranteeing performance properties of diffusion distances, since they guarantee good behavior for clustering by choosing \[t\in [\tau_{1},\tau_{2}]=\left[\frac{\ln\left(\frac{2\kappa}{\epsilon}\right)}{\ln\left(\frac{1}{\lambda_{\numclust+1}}\right)},\frac{\epsilon}{2\delta}\right].\]  However, due to the dependence on a fixed time parameter, they are unable to account for multitemporality in the data, where different clusters emerge at different time scales.  Synthetic examples of data in $\mathbb{R}^{2}$ exhibiting this multiscale phenomena appear in Figures \ref{fig:Multitemporality_Gaussians} and \ref{fig:Multitemporality_Nonlinear}. 

\subsection{Multitemporal Analysis of Diffusion State Distances}
\label{subsec:MultitemporalTheory}

The DSD, however, aggregates across time scales.  This allows it to capture cluster structures that emerge at many different time scales.  Indeed, suppose $\{\C_{r}\}_{r=1}^{R}$ is a family of clusterings, with $\C_{r}=\{C_{r,j}\}_{j=1}^{K_{r}}$ and where $K_{r}$ is the number of clusters at scale $r\in\{1,\dots,\numclust\}$.  For each scale $r$, one can construct the stochastic complement $S_{r}$ of $P$ with respect to the partition $\C_{r}$, which has $K_{r}$ diagonal blocks and an associated equilibrium distribution $S_{r}^{\infty}$.  For each $r$, Theorem \ref{thm:NearReducibility} yields a result depending on $(\delta_{r},\lambda_{r}^{*},\kappa_{r})$, where $\delta_{r}$ and $\kappa_{r}$ are the $\delta$ and $\kappa$ parameters associated to the clustering $\C_{r}$ and $\lambda_{r}^{*}$ denotes the $(K_{r}+1)^{st}$ largest eigenvalue of $S_{r}^{\infty}$.  Given such a family of partitions, the following estimates on DSD hold.

\begin{thm}\label{thm:Main}Let $P\in\mathbb{R}^{n\times n}$ be a Markov matrix for data $X=\{x_{i}\}_{i=1}^{n}\subset\mathbb{R}^{D}$, and let $\{\C_{r}\}_{r=1}^{R}$, $\C_{r}=\{C_{r,j}\}_{j=1}^{K_{r}}$ be an associated family of partitions on $X$.  Let $(\delta_{r},\lambda_{r}^{*},\kappa_{r})$ be the parameters associated to partition $\C_{r}$, as defined in Theorem \ref{thm:NearReducibility}.  Let $r:[0,\infty)\rightarrow \{1,\dots,R\}$ be a map assigning each time scale to a partition.  Then:

\vspace{5pt}

\begin{enumerate}[(a)]

\item $\displaystyle\left\|\sum_{t=1}^{\infty}(e_{i}-e_{j})P^{t}\right\|_{1}\le  \left\|\sum_{t=1}^{\infty}(e_{i}-e_{j})S_{r(t)}^{\infty}\right\|_{1} +  2\sum_{t=1}^{\infty}\left(\delta_{r(t)} t+\kappa_{r(t)}|\lambda_{r(t)}^{*}|^{t}\right).$

\vspace{5pt}

\item $\displaystyle\left\|\sum_{t=1}^{\infty}(e_{i}-e_{j})P^{t}\right\|_{1}\ge \left\|\sum_{t=1}^{\infty}(e_{i}-e_{j})S_{r(t)}^{\infty}\right\|_{1}-2\sum_{t=1}^{\infty}\left(\delta_{r(t)}t+\kappa_{r(t)}|\lambda_{r(t)}^{*}|^{t}\right).$

\end{enumerate}

\end{thm}\begin{proof}To see (a), notice that we may estimate \begin{align*}\left\|\sum_{t=1}^{\infty}(e_{i}-e_{j})P^{t}\right\|_{1}=&\left\|\sum_{t=1}^{\infty}\left((e_{i}-e_{j})S_{r(t)}^{\infty}+(e_{i}-e_{j})\left(P^{t}-S_{r(t)}^{\infty}\right)\right)\right\|_{1}\\
\le & \left\|\sum_{t=1}^{\infty}(e_{i}-e_{j})S_{r(t)}^{\infty}\right\|_{1}+\left\|\sum_{t=1}^{\infty}(e_{i}-e_{j})\left(P^{t}-S_{r(t)}^{\infty}\right)\right\|_{1}\\
\le& \left\|\sum_{t=1}^{\infty}(e_{i}-e_{j})S_{r(t)}^{\infty}\right\|_{1}+2\sum_{t=1}^{\infty}\left\|P^{t}-S_{r(t)}^{\infty}\right\|_{\infty}\\
\le& \left\|\sum_{t=1}^{\infty}(e_{i}-e_{j})S_{r(t)}^{\infty}\right\|_{1}+2\sum_{t=1}^{\infty}\left(\delta_{r(t)}t+\kappa_{r(t)}|\lambda_{r(t)}^{*}|^{t}\right).\end{align*}On the other hand, to see (b), we estimate
\begin{align*}\left\|\sum_{t=1}^{\infty}(e_{i}-e_{j})P^{t}\right\|_{1}=&\left\|\sum_{t=1}^{\infty}\left((e_{i}-e_{j})S_{r(t)}^{\infty}+(e_{i}-e_{j})\left(P^{t}-S_{r(t)}^{\infty}\right)\right)\right\|_{1}\\
\ge & \left\|\sum_{t=1}^{\infty}(e_{i}-e_{j})S_{r(t)}^{\infty}\right\|_{1}-\left\|\sum_{t=1}^{\infty}(e_{i}-e_{j})\left(P^{t}-S_{r(t)}^{\infty}\right)\right\|_{1}\\
\ge & \left\|\sum_{t=1}^{\infty}(e_{i}-e_{j})S_{r(t)}^{\infty}\right\|_{1}-\sum_{t=1}^{\infty}\left\|(e_{i}-e_{j})\left(P^{t}-S_{r(t)}^{\infty}\right)\right\|_{1}\\
\ge & \left\|\sum_{t=1}^{\infty}(e_{i}-e_{j})S_{r(t)}^{\infty}\right\|_{1}-2\sum_{t=1}^{\infty}\left\|P^{t}-S_{r(t)}^{\infty}\right\|_{\infty}\\
\ge & \left\|\sum_{t=1}^{\infty}(e_{i}-e_{j})S_{r(t)}^{\infty}\right\|_{1}-2\sum_{t=1}^{\infty}\left(\delta_{r(t)}t+\kappa_{r(t)}|\lambda_{r(t)}^{*}|^{t}\right).
\end{align*}
\end{proof}

\begin{cor}Let $P\in\mathbb{R}^{n\times n}$ be a Markov matrix for data $X=\{x_{i}\}_{i=1}^{n}\subset\mathbb{R}^{D}$, and let $\{\C_{r}\}_{r=1}^{R}$, $\C_{r}=\{C_{r,j}\}_{j=1}^{K_{r}}$ be an associated family of partitions on $X$.  Let $(\delta_{r},\lambda_{r}^{*},\kappa_{r})$ be the parameters associated to partition $\C_{r}$, as defined in Theorem \ref{thm:NearReducibility}.  Let $r:[0,\infty)\rightarrow \{1,\dots,R\}$ be a map assigning each time scale to a choice of partition.  Let $x_{i},x_{j}\in X$ be such that $x_{i},x_{j}$ are in the same cluster in $C_{r(t)}$, for all $t\ge T_{*}$.  Then:

\vspace{5pt}

\begin{enumerate}[(a)]
\item  \[\left\|\sum_{t=1}^{\infty}(e_{i}-e_{j})P^{t}\right\|_{1}\le \left\|\sum_{t=1}^{T_{*}}(e_{i}-e_{j})S_{r(t)}^{\infty}\right\|_{1}+2\sum_{t=1}^{\infty}\left(\delta_{r(t)}t+\kappa_{r(t)}|\lambda_{r(t)}^{*}|^{t}\right).\]
 
 \item  \[\left\|\sum_{t=1}^{\infty}(e_{i}-e_{j})P^{t}\right\|_{1}\ge \left\|\sum_{t=1}^{T_{*}}(e_{i}-e_{j})S_{r(t)}^{\infty}\right\|_{1}-2\sum_{t=1}^{\infty}\left(\delta_{r(t)}t+\kappa_{r(t)}|\lambda_{r(t)}^{*}|^{t}\right).\]

 \end{enumerate}

\end{cor}

\begin{proof}This follows from Theorem \ref{thm:Main}, together with the observation that if $x_{i},x_{j}$ are in the same cluster at scale $r(t)$, then $(e_{i}-e_{j})S_{r(t)}^{\infty}=\0$.
\end{proof}

Note that if $x_{i},x_{j}$ are in different clusters at scale $r(t)$, \[(e_{i}-e_{j})S_{r(t)}^{\infty}=[\0,\0,\dots,s_{r(t),\ell_{i}}^{\infty},\0,\dots,\0,-s_{r(t),\ell_{j}}^{\infty},\0,\dots,\0],\]where at scale $r(t)$ $x_{i}$ is in the $\ell_{i}^{th}$ cluster, $x_{j}$ is in the $\ell_{j}^{th}$ cluster, and $s_{r(t),\ell}^{\infty}$ is the equilibrium distribution on the $\ell^{th}$ cluster at scale $r(t)$.  Hence, assuming some regularity on the size of the clusters at each scale, \[\left\|\sum_{t=1}^{T_{*}}(e_{i}-e_{j})S_{r(t)}^{\infty}\right\|_{1}\approx T_{*}.\]Moreover, if the hierarchical clustering at scale $r_{*}$ is ``good", in the sense that the stochastic complement $S_{r_{*}}$ well approximates $P$, then $t\delta_{r_{*}}+\kappa_{r_{*}}|\lambda_{r_{*}}^{*}|^{t}$ is small for $t$ such that $r(t)=r_{*}$.  In particular, if all the hierarchical partitions are good, then \[\sum_{t=1}^{\infty}\left(\delta_{r(t)}t+\kappa_{r(t)}|\lambda_{r(t)}^{*}|^{t}\right)\] can be made close to 0 by choosing a good partition at each time step.  This suggests the DSD captures the intrinsic hierarchical cluster structure in data: if two points persist in the same cluster across many time scales, they will be very close in DSD, compared to points that only exist in the same cluster at later time scales.  This multiscale hierarchical phenomenon is illustrated in Figures \ref{fig:Multitemporality_Gaussians}, \ref{fig:Multitemporality_Nonlinear}.  Note that as $t\rightarrow\infty$, $P^{t}$ converges to a rank 1 matrix, and so the trivial partition with all points in the same cluster yields the optimal partition, since in this case $\delta=0$.

We remark that for simplicity, all analysis in this section is done in $\ell^{1}$;  by considering $\gamma(t)$, all these results pass to the $\ell^{2}$ case, which is more convenient for computation, as will be shown in Section \ref{sec:SpectralAnalysis}.

\begin{figure}[!htb]
\centering
\begin{subfigure}[t]{.32\textwidth}
\includegraphics[width=\textwidth]{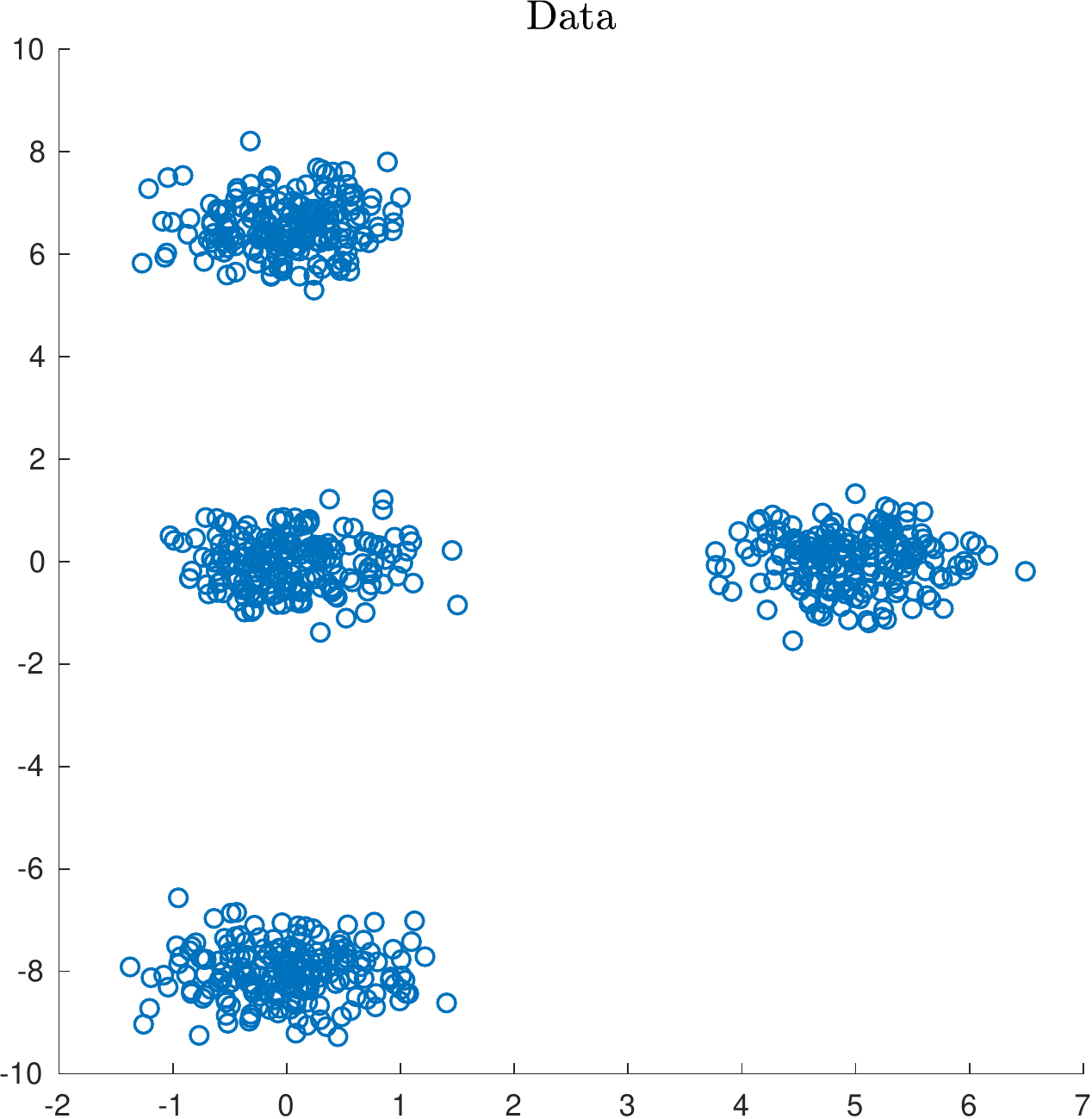}
\subcaption{Gaussian data}
\end{subfigure}
\begin{subfigure}[t]{.32\textwidth}
\includegraphics[width=\textwidth]{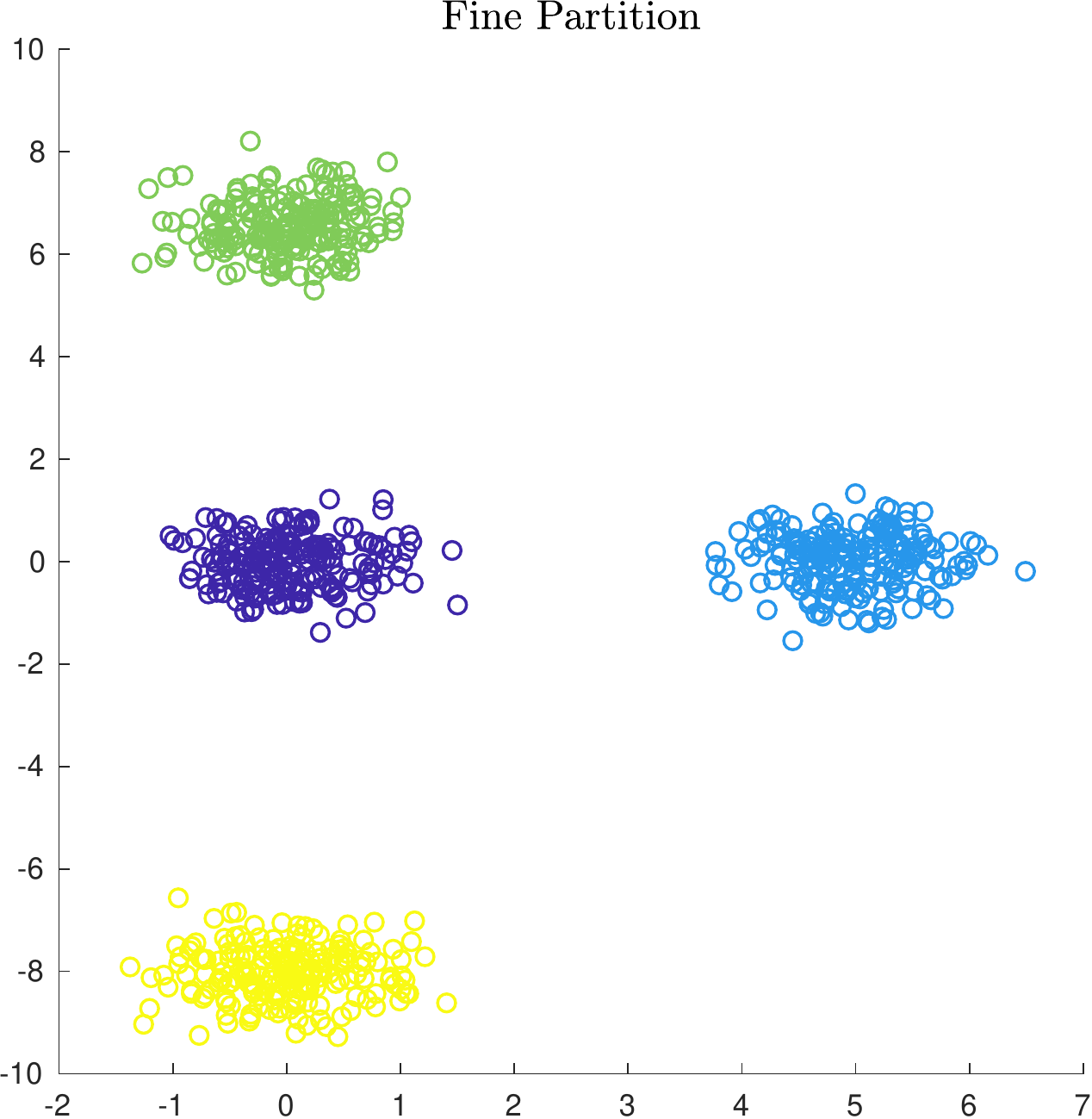}
\subcaption{Fine scale partition}
\end{subfigure}
\begin{subfigure}[t]{.32\textwidth}
\includegraphics[width=\textwidth]{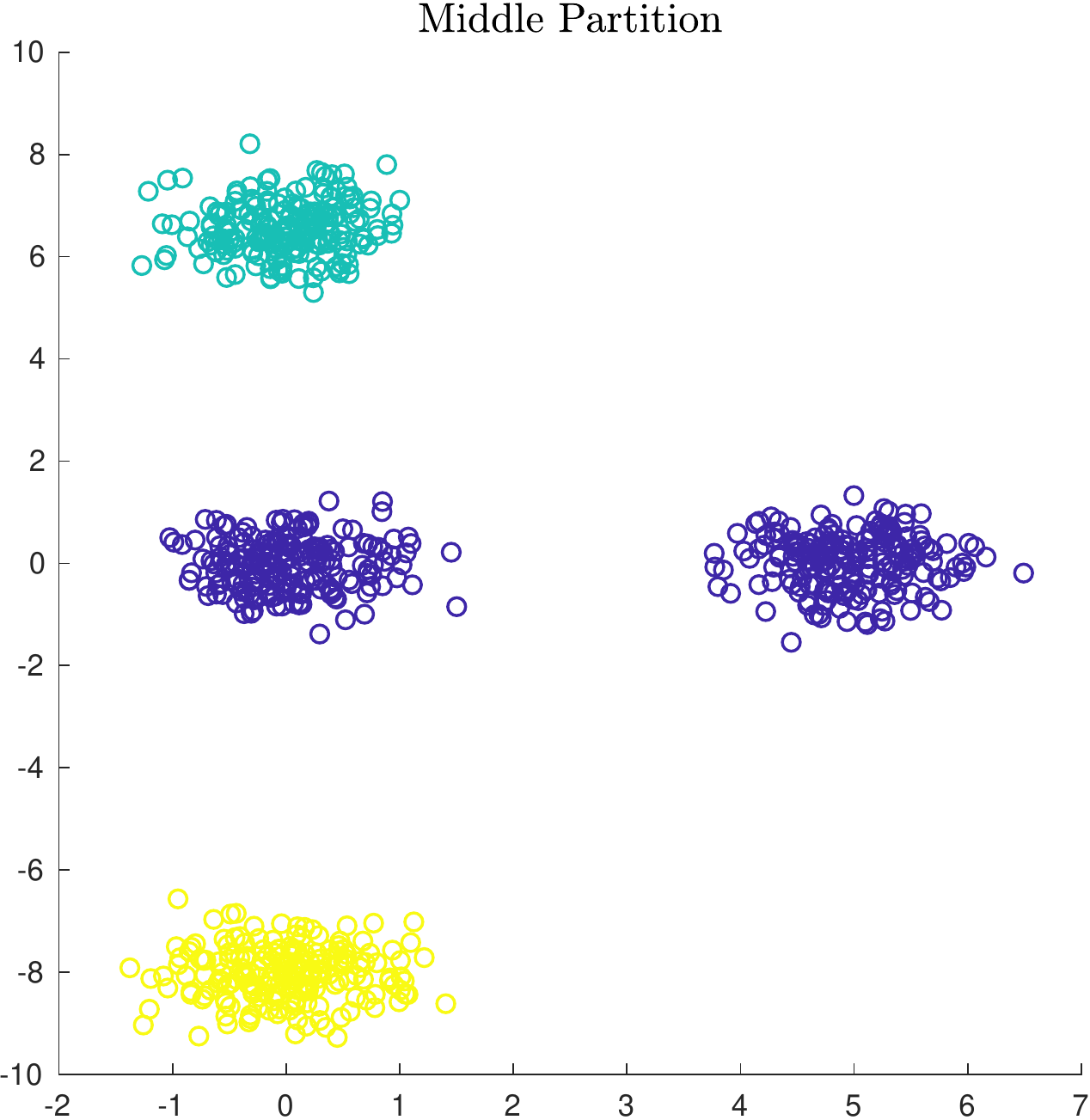}
\subcaption{Middle scale partition}
\end{subfigure}
\begin{subfigure}[t]{.32\textwidth}
\includegraphics[width=\textwidth]{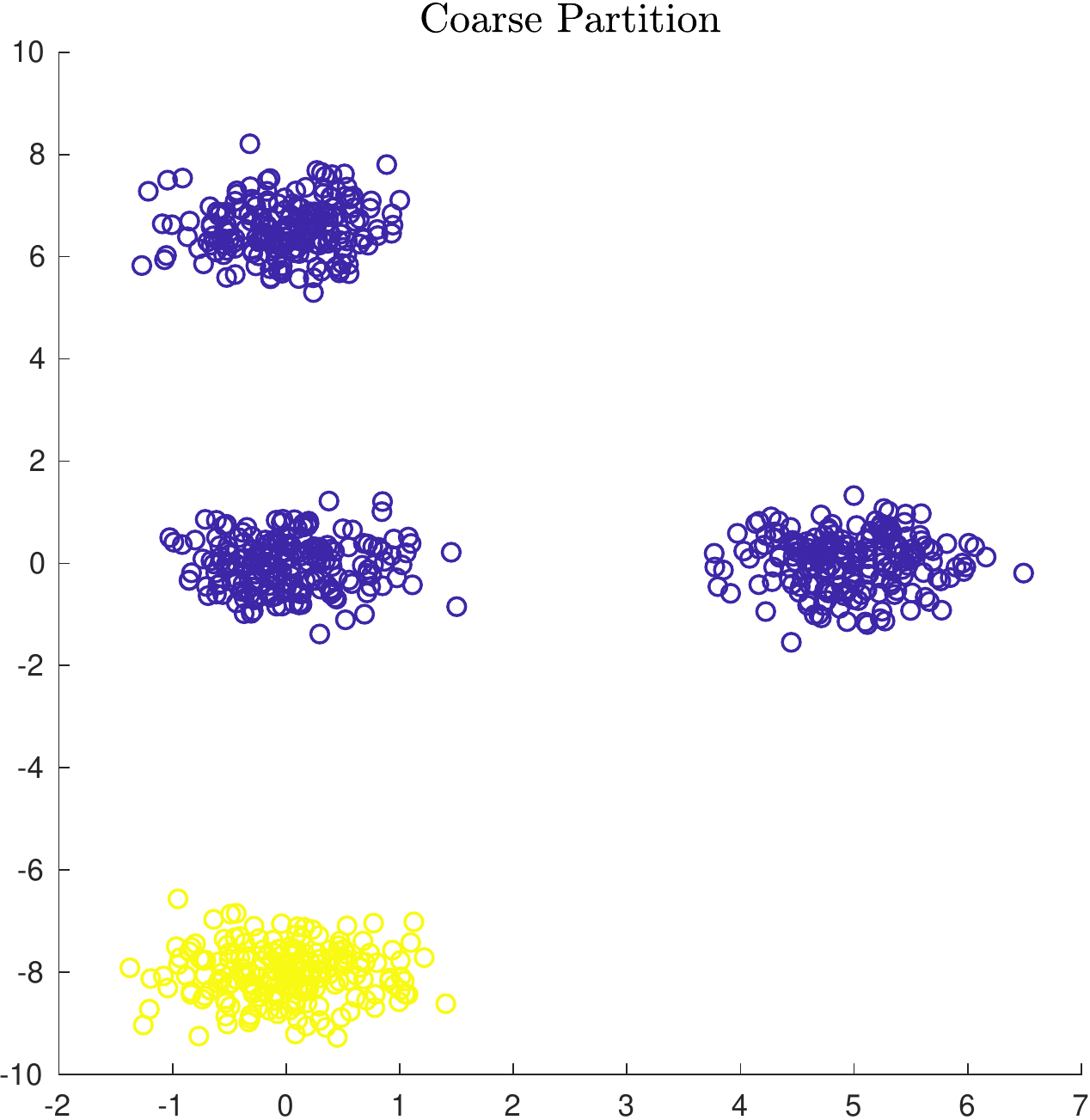}
\subcaption{Coarse scale partition}
\end{subfigure}
\begin{subfigure}[t]{.32\textwidth}
\includegraphics[width=\textwidth]{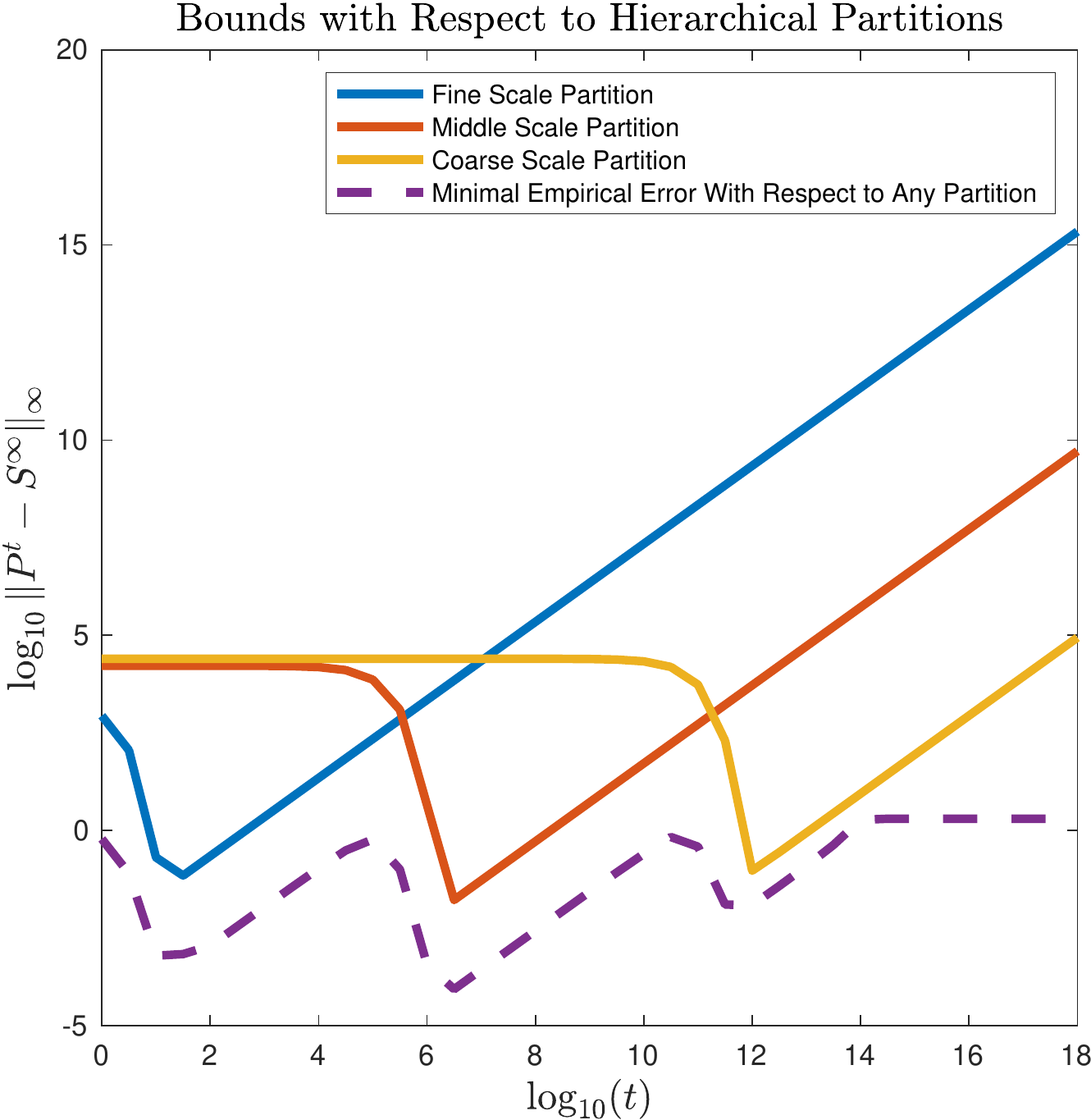}
\subcaption{Estimation of $\|P^{t}-S^{\infty}\|_{\infty}$}
\end{subfigure}

\caption{\emph{(a)}: Mixture of Gaussian data in $\mathbb{R}^{2}$ exhibiting multiscale structure.  The means of the Gaussians are at (0,0), (5,0), (0,6.5), (0,-8).    All Gaussians have covariance matrix $\frac{1}{4} I$.  The transition matrix $P$ was constructed with the Gaussian kernel with $\sigma=1$.  \emph{(b), (c), (d)}: There are four natural clusterings in the data, since the Gaussians are not equidistant.  \emph{(e)}:  The estimates on $\|P^{t}-S^{\infty}\|_{\infty}$ with respect to the fine, middle, and coarse scale partitions.  We see that as time progresses, increasingly coarser partitions fit the data.  Moreover, if we minimize the errors induced by these partitions (shown in the dotted line), we see that these mesoscopic equilibria characterize essentially the full multitemporal behavior of $P$.}
\label{fig:Multitemporality_Gaussians}		
\end{figure}

\begin{figure}[!htb]
\begin{subfigure}[t]{.49\textwidth}
\includegraphics[width=\textwidth]{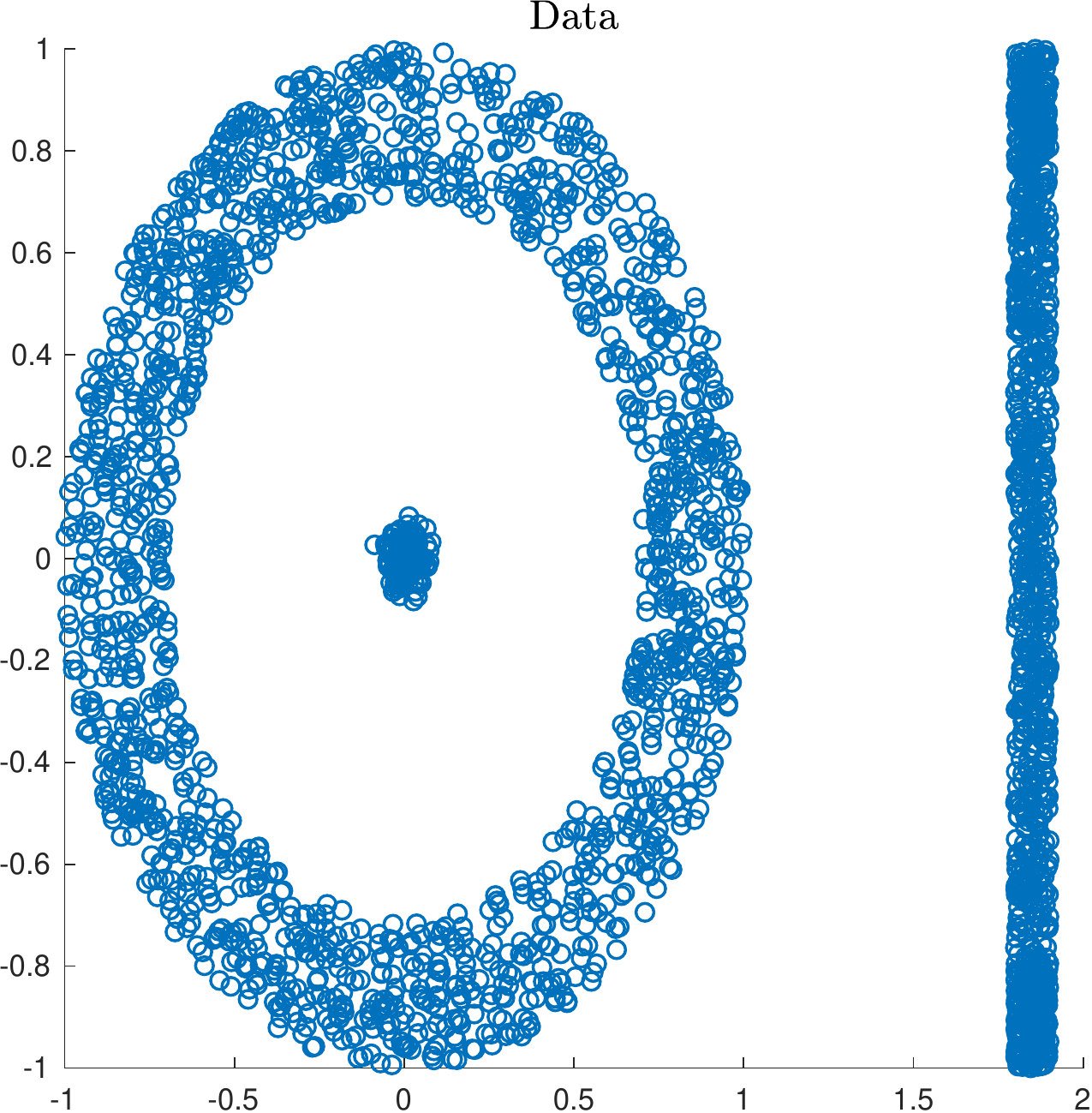}
\subcaption{Geometric data}
\end{subfigure}
\begin{subfigure}[t]{.49\textwidth}
\includegraphics[width=\textwidth]{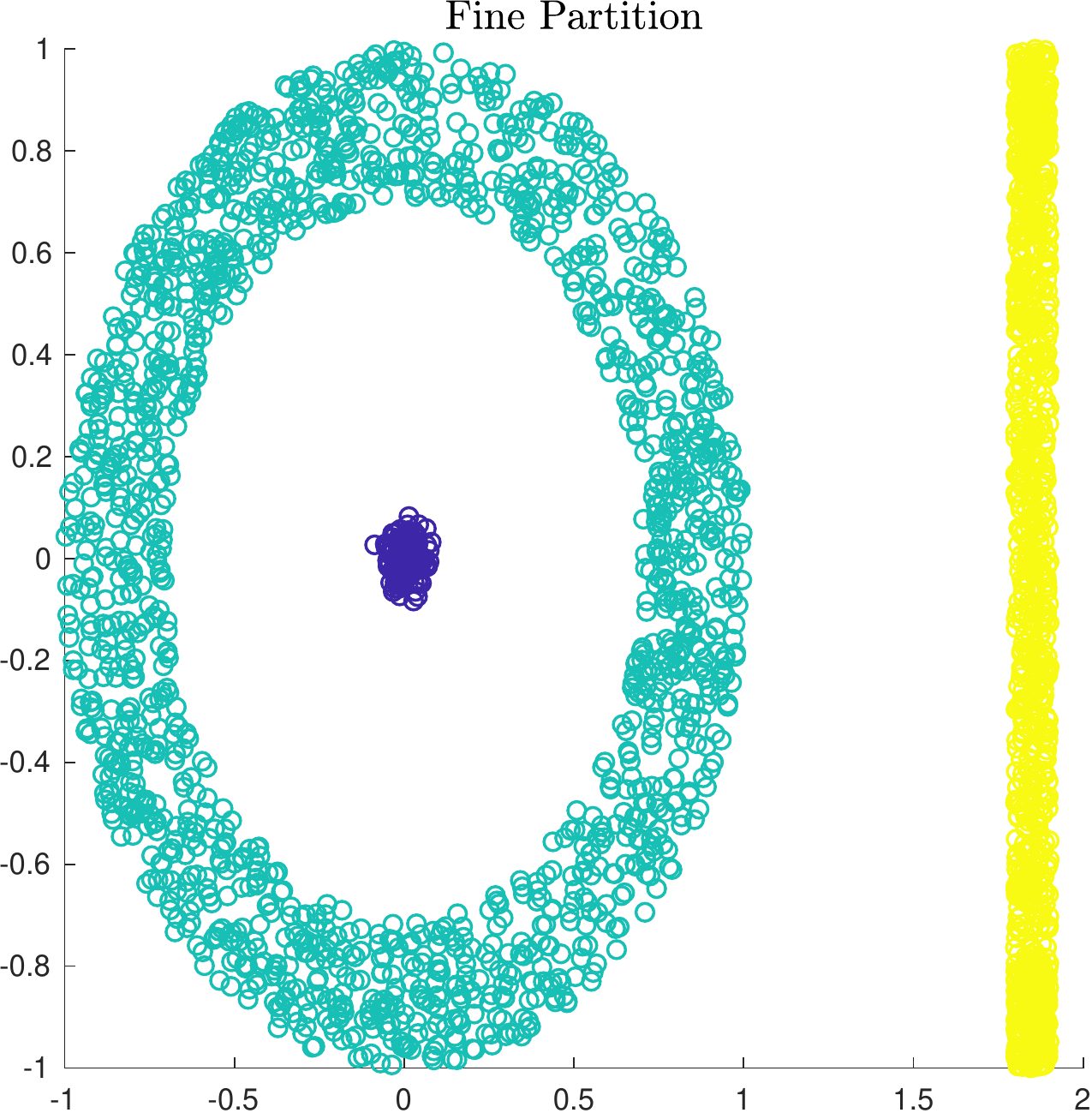}
\subcaption{Fine scale partition}
\end{subfigure}
\begin{subfigure}[t]{.49\textwidth}
\includegraphics[width=\textwidth]{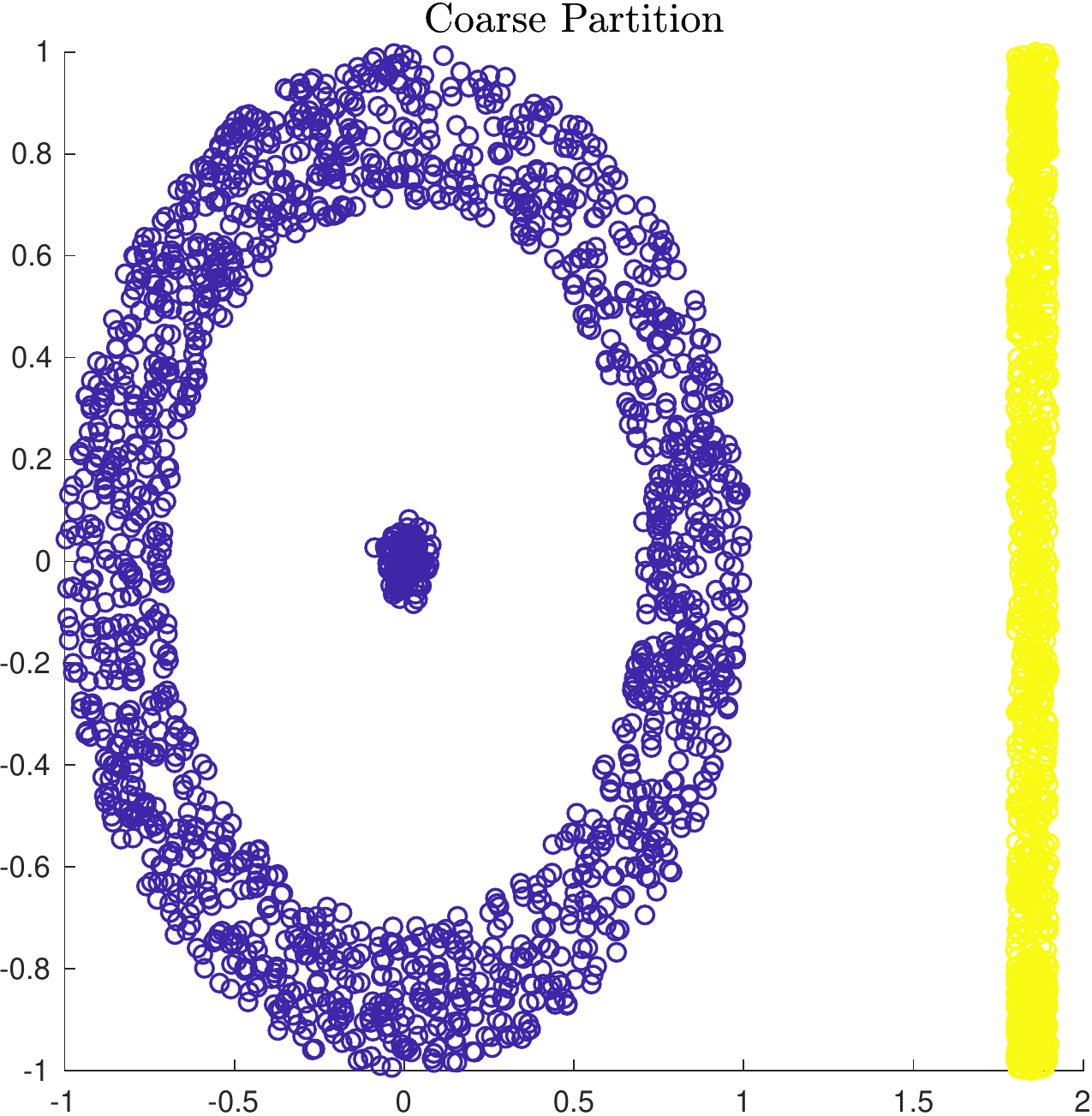}
\subcaption{Coarse scale partition}
\end{subfigure}
\begin{subfigure}[t]{.49\textwidth}
\includegraphics[width=\textwidth]{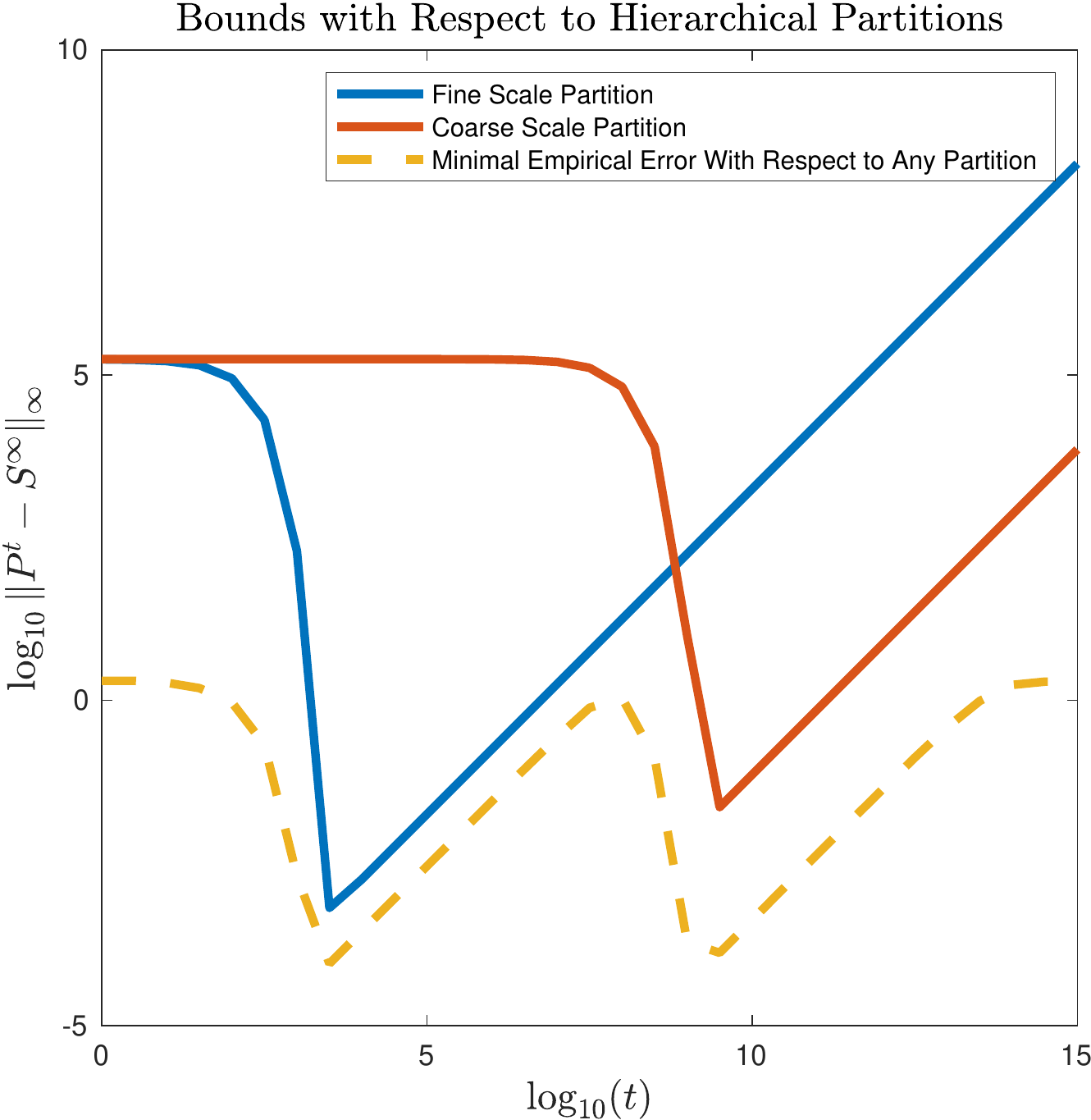}
\subcaption{Estimation of $\|P^{t}-S^{\infty}\|_{\infty}$}
\end{subfigure}
\caption{\emph{(a):}  Data in $\mathbb{R}^{2}$ showing geometric structure.  The transition matrix $P$ was constructed with the Gaussian kernel with $\sigma=.16$.  \emph{(b), (c)}:  We consider two natural clusterings in the data: a fine clustering consisting of three clusters, with each of the ring, the Gaussian, and the bar as their own cluster; and a coarse clustering consisting of two clusters with the ring and the Gaussian together as one cluster and the bar as its own cluster.  \emph{(d)}:  The estimates on $\|P^{t}-S^{\infty}\|_{\infty}$ with respect to the partition with three clusters, and with respect to the coarser, two-scale clustering, are shown.  We see that for increasing $t$, increasingly coarser partitions fit the data.  Moreover, if we minimize the errors induced by these partitions (shown in the dotted line), we see that these mesoscopic equilibria characterize essentially the full multitemporal behavior of $P$.}
\label{fig:Multitemporality_Nonlinear}		
\end{figure}

\subsection{Large Sample Limits}
\label{subsec:LargeSampleLimits}

The results of Section \ref{subsec:MultitemporalTheory} are for finite datasets $X=\{x_{i}\}_{i=1}^{n}\subset\mathbb{R}^{D}$, from which a discrete Markov diffusion matrix $P$ is constructed.  It is natural to ask about the large sample limit $n\rightarrow\infty$.  When the data $X$ is sampled from a manifold $\mathcal{M}$ with density function $p(x)>0$, it is known that the matrix $P$ and corresponding graph Laplacian converge to a continuum object under certain scaling regimes.  In particular, the governing parameters of the multitemporal theory, namely $\delta_{r}, \lambda_{r}^{*}, \kappa_{r}$, all have continuum analogues. 

It is known \cite{VonLuxburg2014_Hitting} that the related notion of commute distance $\mathscr{C}(x_{i},x_{j})=\langle (e_{i}-e_{j}),\Lsym^{\dagger}(e_{i}-e_{j})\rangle$ converges under this data model to a degenerate limit as $n\rightarrow\infty$, despite the fact that the pseudoinverse of $\Lsym$ converges to a well-defined limit.  Indeed, $\mathscr{C}(x_{i},x_{j})$ converges, after an appropriate scaling, to $|\frac{1}{D_{i}}-\frac{1}{D_{j}}|$, where $D_{i}, D_{j}$ are the degrees of $x_{i}, x_{j}$, respectively.  As such, commute distances become uninteresting in the asymptotic limit $n\rightarrow\infty$ under these random graph models.  It is thus natural to consider whether DSD becomes similarly degenerate as $n\rightarrow\infty$.  

It is precisely because the vectors which determine DSD become identical upon convergence of $P$ to stationarity, that the DSD formulation avoids this degeneracy: all the action of the DSD is on the early portion of the walk before it reaches stationarity.  Moreover, the analysis presented in Section \ref{subsec:MultitemporalTheory} is sufficiently general to apply beyond the types of random graph models for which commute distances are known to be degenerate.  Indeed, the conditions necessary for Theorem \ref{thm:Main} to be illuminating require the largest eigenvalues of $P$ to be close to 1, in order to have a sufficient time scale at which mixing within clusters has occurred, but between cluster mixing has not.  In the asymptotic analysis of \cite{VonLuxburg2014_Hitting}, the graphs under consideration all have relatively large spectral gap $1-\lambda_{2}(P)$.  In particular, the kinds of graphs for which commute distances degenerate are different from those which we consider.  In this sense, the results of Section  \ref{subsec:MultitemporalTheory} are most illuminating for data sampled from a mixture of manifolds (modeling clusters), and a background noise region with much lower density (modeling transition regions between the clusters) \cite{Little2020_Path}.  Developing a large sample limit theory for such a model of clustered data is the subject of ongoing research.  

\section{Spectral Analysis of Diffusion State Distances}
\label{sec:SpectralAnalysis}

The analysis of DSD in terms of multitemporal mesoscopic equilibria presented in Section \ref{sec:MultitemporalAnalysis} manifests itself further when considering the associated spectral decomposition.  Though DSD can be defined in any $\ell^{p}$ space, $p\ge 1$ (and the results of Section \ref{sec:MultitemporalAnalysis} focus on the case $p=1$), Section \ref{sec:SpectralAnalysis} focuses on the case $p=2$, in order to exploit the Hilbert space structure of $\ell^{2}$.  As discussed in Section \ref{sec:SpectralAnalysis}, the factor $\gamma(t)$ quantifies the trade-off between $p=1,2$.

Recall that the DSD with respect to the weight $w$ may be written as \[\DSD (x_{i},x_{j})=\left\|\sum_{t=0}^{\infty}p_{t}(x_{i},\cdot)-p_{t}(x_{j},\cdot)\right\|_{\ell^{2}(w)}.\]  As shown in Section \ref{sec:Background}, $P$ is diagonally conjugate to the symmetric matrix $D^{-\frac{1}{2}}WD^{-\frac{1}{2}}$, and a computation \cite{Coifman2006_Diffusion} reveals that $P$ may be written as a sum of outer products \[P=\sum_{\ell=1}^{n}\lambda_{\ell}\psi_{\ell}\varphi_{\ell}^{\top},\] where $\psi_{\ell}(x)=\phi_{\ell}(x)/\sqrt{\pi(x)}, \varphi_{\ell}(y)=\phi_{\ell}(y)\sqrt{\pi(y)},$ for $\pi$ the stationary distribution of $P$ and $\{(\phi_{\ell},\lambda_{\ell})\}_{\ell=1}^{n}$ the eigenvectors and eigenvalues of the symmetric matrix $D^{-\frac{1}{2}}WD^{-\frac{1}{2}}=D^{\frac{1}{2}}PD^{-\frac{1}{2}}$, understood as column vectors.  Equivalently, $p_{t}(x,y)=\sum_{\ell=1}^{n}\lambda_{\ell}^{t}\psi_{\ell}(x)\varphi_{\ell}(y)$ .  In particular, $\{\phi_{\ell}\}_{\ell=1}^{n}$ is an orthonormal basis for $\ell^{2}$ and $\{\varphi_{\ell}\}_{\ell=1}^{n}$ is an orthonormal basis for $\ell^{2}(1/\pi)$.  With this decomposition, the DSD may be re-formulated as follows.

\begin{thm}\label{thm:SpectralAnalysisDSD}The diffusion state distance with respect to the weight $w=1/\pi$ admits the decomposition \[\DSD (x_{i},x_{j})=\displaystyle\left\|(e_{i}-e_{j})\sum_{t=0}^{\infty}P^{t}\right\|_{\ell^{2}(1/\pi)}=\sqrt{\sum_{\ell=1}^{n}\left(\frac{1}{1-\lambda_{\ell}}\right)^{2}(\psi_{\ell}(x_{i})-\psi_{\ell}(x_{j}))^{2}},\]where $\{(\lambda_{\ell},\psi_{\ell})\}_{\ell=1}^{n}$ are the eigenvalues and right eigenvectors of $P$.
\end{thm}

\begin{proof}  Decomposing $p_{t}$, noting that $\psi_{1}=\1$, and summing over $t$ as above yields

\begin{align*}\DSD (x_{i},x_{j})=&\left\|\sum_{t=0}^{\infty}p_{t}(x_{i},\cdot)-p_{t}(x_{j},\cdot)\right\|_{\ell^{2}(1/\pi)}\\
=&\left\|\sum_{t=0}^{\infty}\sum_{\ell=1}^{n}\lambda_{\ell}^{t}\psi_{\ell}(x_{i})\varphi_{\ell}(\cdot)-\lambda_{\ell}^{t}\psi_{\ell}(x_{j})\varphi_{\ell}(\cdot)\right\|_{\ell^{2}(1/\pi)}\\
=&\left\|\sum_{t=0}^{\infty}\sum_{\ell=2}^{n}\lambda_{\ell}^{t}(\psi_{\ell}(x_{i})-\psi_{\ell}(x_{j}))\varphi_{\ell}(\cdot)\right\|_{\ell^{2}(1/\pi)}\\
=&\left\|\sum_{\ell=2}^{n}\sum_{t=0}^{\infty}\lambda_{\ell}^{t}(\psi_{\ell}(x_{i})-\psi_{\ell}(x_{j}))\varphi_{\ell}(\cdot)\right\|_{\ell^{2}(1/\pi)}\\
=&\left\| \sum_{\ell=2}^{n}\frac{1}{1-\lambda_{\ell}}(\psi_{\ell}(x_{i})-\psi_{\ell}(x_{j}))\varphi_{\ell}(\cdot)\right\|_{\ell^{2}(1/\pi)}
\end{align*}Recalling that $\{\varphi_{\ell}\}_{\ell=1}^{n}$ is an orthonormal basis for $\ell^{2}(1/\pi)$, the result follows from  Parseval's theorem.
\end{proof}


We remark that the spectral decomposition of $P$ is related  \cite{Boehnlein2014_Computing, Beveridge2016_Hitting} to the Green's function $G$ for the inhomogeneous Laplace's equation $\Lrw f=g$.  Indeed, let the $(i,j)^{th}$ entry of $G\in\mathbb{R}^{n\times n}$ be \[G_{ij}=\sum_{\ell=2}^{n}\frac{1}{1-\lambda_{\ell}}\psi_{\ell}(x_{i})\varphi_{\ell}(x_{j}).\]  Then \begin{align*}G\Lrw=&\sum_{\ell=2}^{n}\left(\frac{1}{1-\lambda_{\ell}}\psi_{\ell}\varphi_{\ell}^{\top}\right)\left(\sum_{j=1}^{n}(1-\lambda_{j})\psi_{j}\varphi_{j}^{\top}\right)\\=&\sum_{\ell=2}^{n}\psi_{\ell}\varphi_{\ell}^{\top}\\=&I-\pi\1.\end{align*}

\subsection{Diffusion State Distance and Dimension Reduction} \label{subsec:DSD-DR}

Theorem \ref{thm:SpectralAnalysisDSD} implies that the DSD is in fact the Euclidean distance in a new coordinate basis, given by the transformation \[x\mapsto \left(\frac{1}{1-\lambda_{2}}\psi_{2}(x),\frac{1}{1-\lambda_{3}}\psi_{3}(x),\cdots,\frac{1}{1-\lambda_{n}}\psi_{n}(x)\right).\]  This representation naturally lends itself to dimension reduction in the following sense.  Suppose that for some $M$, $1/(1-\lambda_{M})\gg1/(1-\lambda_{M+1}).$  Then the truncated representation \begin{align}\label{eqn:DSD_Embedding}x\mapsto \left(\frac{1}{1-\lambda_{2}}\psi_{2}(x),\frac{1}{1-\lambda_{3}}\psi_{3}(x),\cdots,\frac{1}{1-\lambda_{M}}\psi_{M}(x)\right)\end{align} is faithful in the sense that the discarded coordinates are comparatively insignificant.  Moreover, depending on the number of data points, the spectral decomposition is unreliably noisy after some $M\ll n$ \cite{Trillos2019_Error}, in which case the truncation of the coordinate change has the added benefit of denoising.  This approach to computing the DSD is an alternative to randomized approaches \cite{Lin2018_Computing}.  

To understand the condition $1/(1-\lambda_{M})\gg1/(1-\lambda_{M+1})$, let $\lambda_{\ell}=1-\mu_{\ell}$ for $0=\mu_{1}\le \mu_{2}\le\cdots\le \mu_{n}$.  Then the condition becomes $\frac{1}{\mu_{M}}\gg\frac{1}{\mu_{M+1}}$, which is a measure of a gap between in the near-reducibility of certain subsets of the underlying Markov transition matrix \cite{Maggioni2019_Learning}.  Indeed, suppose that $\lambda_{M}=1-e^{-q_{1}}, \lambda_{M+1}=1-e^{-q_{2}}$ for $q_{1}\ge q_{2}\ge0$.  If $q_{1}=q_{2}$, then the corresponding eigenvectors $\psi_{M}, \psi_{M+1}$ count equally in the DSD, since the weights $1/(1-\lambda_{\ell}), \ell=M,M+1$ are equal.  On the other hand, if $q_{1}\gg q_{2}$, then there is a large gap between $\mu_{M}$ and $\mu_{M+1}$:  $\mu_{M}/\mu_{M+1}=e^{-q_{1}+q_{2}}\ll 1$; in this case, the DSD may be truncated after the $M^{th}$ eigenvector while preserving much of the information present in the full DSD.

\subsection{Relationship Between Diffusion State Distances and Inverse Laplacians}

The DSD is closely related to the inverse of a regularized Laplacian \cite{Cao2013_Going, Boehnlein2014_Computing}.  Indeed, \[\DSD (x_{i},x_{j})=\|(e_{i}-e_{j})(I-P+\1\pi)^{-1}\|_{\ell^{2}(w)}\] as shown in Section \ref{subsec:DSD_Background}.  It is enlightening to consider the Neumann series expansion, noting that $\lambda_{1}=1, \psi_{1}=\1, \varphi_{1}=\pi$:
\begin{align*}(I-P+\1\pi)^{-1}=&\left(I-(P-\1\pi\right))^{-1}\\
=&\sum_{t=0}^{\infty}(P-\1\pi)^{t}\\
=&\sum_{t=0}^{\infty}(P^{t}-\1\pi)\\
=&\sum_{t=0}^{\infty}\left(\left(\sum_{\ell=1}^{n}\lambda_{\ell}^{t}\psi_{\ell}\varphi_{\ell}\right)-\1\pi\right)\\
=&\sum_{t=0}^{\infty}\left(\sum_{\ell=2}^{n}\lambda_{\ell}^{t}\psi_{\ell}\varphi_{\ell}\right)\\
=&\sum_{\ell=2}^{n}\left(\sum_{t=0}^{\infty}\lambda_{\ell}^{t}\right)\psi_{\ell}\varphi_{\ell}\\
=&\sum_{\ell=2}^{n}\frac{1}{1-\lambda_{\ell}}\psi_{\ell}\varphi_{\ell}.\end{align*}

Hence, 

\begin{align*}\DSD (x_{i},x_{j})=&\|(e_{i}-e_{j})(I-P+\1\pi)^{-1}\|_{\ell^{2}(w)}\\
=& \left\|\sum_{\ell=2}^{n}\frac{1}{1-\lambda_{\ell}}(\psi_{\ell}(x_{i})-\psi_{\ell}(x_{j}))\varphi_{\ell}(\cdot)\right\|_{\ell^{2}(w)}.
\end{align*}

Noting that $\{\varphi_{\ell}\}_{\ell=1}^{n}$ is an orthonormal basis for $\ell^{2}(1/\pi)$, we see that this expansion is of particular use when $w=1/\pi$.  We remark that the original formulation of DSD used the weight $w=\1$, i.e. constant weight.  If the degrees of the underlying graph generating $P$ are nearly constant (i.e. if $P$ is nearly bistochastic), then the difference between this original formulation and the DSD with weight $1/\pi$ is small in a multiplicative sense.  The more disparate the degree structure, the more the $1/\pi$ factor in the proposed DSD will correct for this degree imbalance, and the original formulation and $1/\pi$-weighted one will differ.  

\subsection{Computational Considerations}
\label{subsec:CC}
Based on the discussion in the previous sections, we can improve the computational complexity of computing DSD approximately in practice.  Indeed, a naive direct computation of all pairwise DSD when the underlying graph on $n$ nodes is dense is $O(n^{3})$, which is unacceptably inefficient for large datasets.

Recall that the right eigenvectors $\{\psi_{\ell}\}_{\ell=1}^{n}$ of P are, after scaling by $\sqrt{\pi}$, eigenvectors of the symmetric normalized graph Laplacian $\Lsym = I - D^{-\frac{1}{2}}W D^{-\frac{1}{2}}$ with corresponding eigenvalues $\mu_{\ell} = 1 - \lambda_{\ell}$.  We can thus compute the DSD using $\Lsym$ according to Algorithm \ref{alg:DSD}.

\

\begin{algorithm}[htp!]
	\caption{Compute DSD} \label{alg:DSD}
	\begin{algorithmic}[1]
		\STATE Compute the eigenpairs $\{(\mu_\ell, \phi_{\ell})\}_{\ell=1}^{n}$ of the symmetric normalized graph Laplacian $\Lsym$.
		\STATE Compute $Y = \Sigma \Phi D^{-\frac{1}{2}}$, where $\Sigma = \operatorname(\mu_2^{-1}, \mu_3^{-1}, \cdots, \mu_n^{-1})$ and $\Phi = (\phi_2, \phi_3, \cdots, \phi_n)$.
		\STATE Compute $\DSD (x_i, x_j) = \| Y(:,i) - Y(:,j) \|_{2} $.
	\end{algorithmic}
\end{algorithm}

\

Note that in Algorithm~\ref{alg:DSD}, we first compute $Y$, which can be considered as the coordinates suggested in Section~\ref{subsec:DSD-DR}. This $Y$ can also be saved for later usage in machine learning tasks.  Moreover, we can further apply the dimension reduction technique discussed in Section~\ref{subsec:DSD-DR}, which basically implies that we only compute the first $M$ eigenpairs and then form a truncated version of $Y$.  This is summarized in Algorithm~\ref{alg:approx-DSD}.

\

\begin{algorithm}[htp!]
	\caption{Compute approximate DSD} \label{alg:approx-DSD}
	\begin{algorithmic}[1]
		\STATE Compute first $M$ eigenpairs $\{(\mu_\ell, \phi_{\ell})\}_{\ell=1}^{M}$, of the symmetric normalized graph Laplacian $\Lsym$.\STATE Compute $\widetilde{Y} = \Sigma_M \Phi_M D^{-\frac{1}{2}}$, where $\Sigma_M = \operatorname(\mu_2^{-1}, \mu_3^{-1}, \cdots, \mu_M^{-1})$ and $\Phi_M = (\phi_2, \phi_3, \cdots, \phi_M)$.
		\STATE Compute $\widetilde{\DSD}(x_i, x_j) = \| \widetilde{Y}(:,i) - \widetilde{Y}(:,j) \|_{2} $.
	\end{algorithmic}
\end{algorithm}

\

The main computational cost in both algorithms is the first step, i.e., computing the eigenpairs of the symmetric normalized graph Laplaican $\Lsym$.  To do this, we use the preconditioned Krylov-based eigensolver, such as the implementation in ARPACK \cite{lehoucq1998arpack}.  The performance of such a Krylov-based eigensolver depends on the choices of the preconditioner.  Since we are working with the normalized graph Laplacian, the algebraic multigrid method (AMG) is used as a preconditioner \cite{xu2017algebraic}.  More precisely, we use the aggregation-based AMG method, which is a suitable choice for solving graph Laplacian problems as shown in~\cite{livne2012lean,brannick2013algebraic,Lin2018_Computing,hu2019adaptive} due to its (nearly) optimal computational complexity in practice.

\section{Computational and Numerical Experiments}
\label{sec:NumericalExperiments}

In order to validate the results of Section \ref{sec:MultitemporalAnalysis} and \ref{sec:SpectralAnalysis}, we perform experiments on synthetic and real datasets.  In Section \ref{subsec:SyntheticDatasets}, experiments on synthetic data demonstrate that DSD captures multiscale cluster structure, unlike traditional diffusion or Laplacian embeddings, and that the spectral decomposition of DSD allows for fast computations that also denoise.  In particular, we shall compare the DSD embedding (\ref{eqn:DSD_Embedding}) with the classical Laplacian eigenmaps (LE) \cite{Ng2002_Spectral, Belkin2003_Laplacian} embedding $x\mapsto (\psi_{2}(x),\dots,\psi_{M}(x))$.  Note that it is not necessary to consider $\psi_{1}$, as it is constant by construction.  In Section \ref{subsec:BiologicalNetworks}, we consider the DSD for the analysis of real biological networks.

\subsection{Synthetic Datasets}
\label{subsec:SyntheticDatasets}

We consider three synthetic datasets that characterize clustered data.  All examples consist of 3 communities each of size 100, 2 of which are closer together than either is to a third.  For all examples, we describe how to construct $W$.  The Markov diffusion matrix $P$ is then constructed by normalizing $W$ to be row-stochastic: $P=D^{-1}W$.  The datasets considered are:\\\begin{enumerate}[(a)]

\item A \emph{hierarchical stochastic block model (HSBM)} \cite{Peixoto2014_Hierarchical, Lyzinski2017_Community} $P$ with different levels of affinity between the different communities.  The general \emph{stochastic block model (SBM)} \cite{Abbe2017_Community} considers random graphs on $n$ nodes, partitioned into communities $\{C_{k}\}_{k=1}^{\numclust}$.  For $k,k'\in\{1,\dots,\numclust\}$, let $p_{kk'}\in [0,1]$.  We generate an unweighted random graph by sampling edges in an i.i.d. fashion by putting an edge between $x\in C_{k}, y\in C_{k'}$ with probability $p_{kk'}$.  The resulting weight matrix $\widetilde{W}$ is symmetrized by considering $W_{ij}=\max\{\widetilde{W}_{ij},\widetilde{W}_{ji}\}$.  In our case, $\numclust=3$ and we consider $p_{11}=p_{22}=p_{33}=.5, p_{12}=p_{21}=p_{13}=p_{31}=.001, p_{23}=p_{32}=.01$.  This SBM is hierarchical in the sense that there is strong separation between the three communities $C_{1},C_{2},C_{3}$, but also between the two communities consisting of $C_{1}$ and $C_{2}\cup C_{3}$.\vspace{5 pt}  

\item A \emph{low-rank block} $P$.  We construct the matrix $P$ to be constant on each of three diagonal blocks, with weak connections between the blocks.  In order to illustrate how the DSD captures aggregate behavior across time scales, the connections between the second and third clusters are made slightly stronger than those between the first and the others.  That is, for the same choice of $\{p_{kk'}\}_{k,k'=1}^{3}$ as above, we set $W_{ij}=p_{k(i)k(j)}$, where $x_{i}\in C_{k(i)}$.  This can be interpreted as the expected value of the hierarchical stochastic block matrices described in example (a).\vspace{5 pt}  

\item A \emph{random geometric} $P$ generated by randomly sampling 100 times from each of three different Gaussians with identity covariance matrices and means $(0,0), (4,0),$ and $(2,6)$, respectively.  This yields a sample $\{x_{i}\}_{i=1}^{300}$.  We then set $W_{ij}=\exp(-\|x_{i}-x_{j}\|_{2}^{2})$.\vspace{5 pt}  

\end{enumerate}

\subsubsection{Embeddings for Multiscale Data}

Given the transition matrix $P$, one can construct a low-dimensional embedding using the eigenvectors of $P$, as in LE, or the weighted eigenvectors as in (\ref{eqn:DSD_Embedding}).  The impact of weighting by the eigenvalues can be significant---eigenvectors with eigenvalues close to 1 contribute more significantly than those far from 1.  It is natural to value such eigenvectors more, since an eigenvalue close to 1 indicates that $P$ is close to being reducible.  In that case, the corresponding eigenvector is expected to localize on the two nearly disconnected components of the underlying graph.  In all cases, we truncate the eigenvector expansion at $M=3$, corresponding to the observation that there are 3 clusters in a sense particular to the different methods of generating the data.

Figure \ref{fig:TransitionMatricesAndEmbeddings} shows the impact of the different embeddings for the three multiscale datasets considered.  The DSD embedding is more faithful to the multiscale hierarchy in the dataset; this structure is preserved when the embedding is weighted as in (\ref{eqn:DSD_Embedding}).  When the unweighted LE embedding is used, all the eigenvectors count equally---they all have $\ell^{2}$ norm 1---and the hierarchical structure captured in this case by the weights $(1-\lambda_{\ell})^{-1}$ is washed out.  Indeed, in the LE plots, it is hard to tell that two of the clusters are closer to each than either is to a third---all one can observe is three equally well-separated clusters.  When using DSD, however, the well-separated clusters are arranged in such a way that two clusters are closer together than either is to a third.  In this sense, the DSD are more informative than the LE.  

\begin{figure}[!htb]
\centering
\begin{subfigure}[b]{.26\textwidth}
\captionsetup{width=.95\linewidth}
		\includegraphics[width=\textwidth]{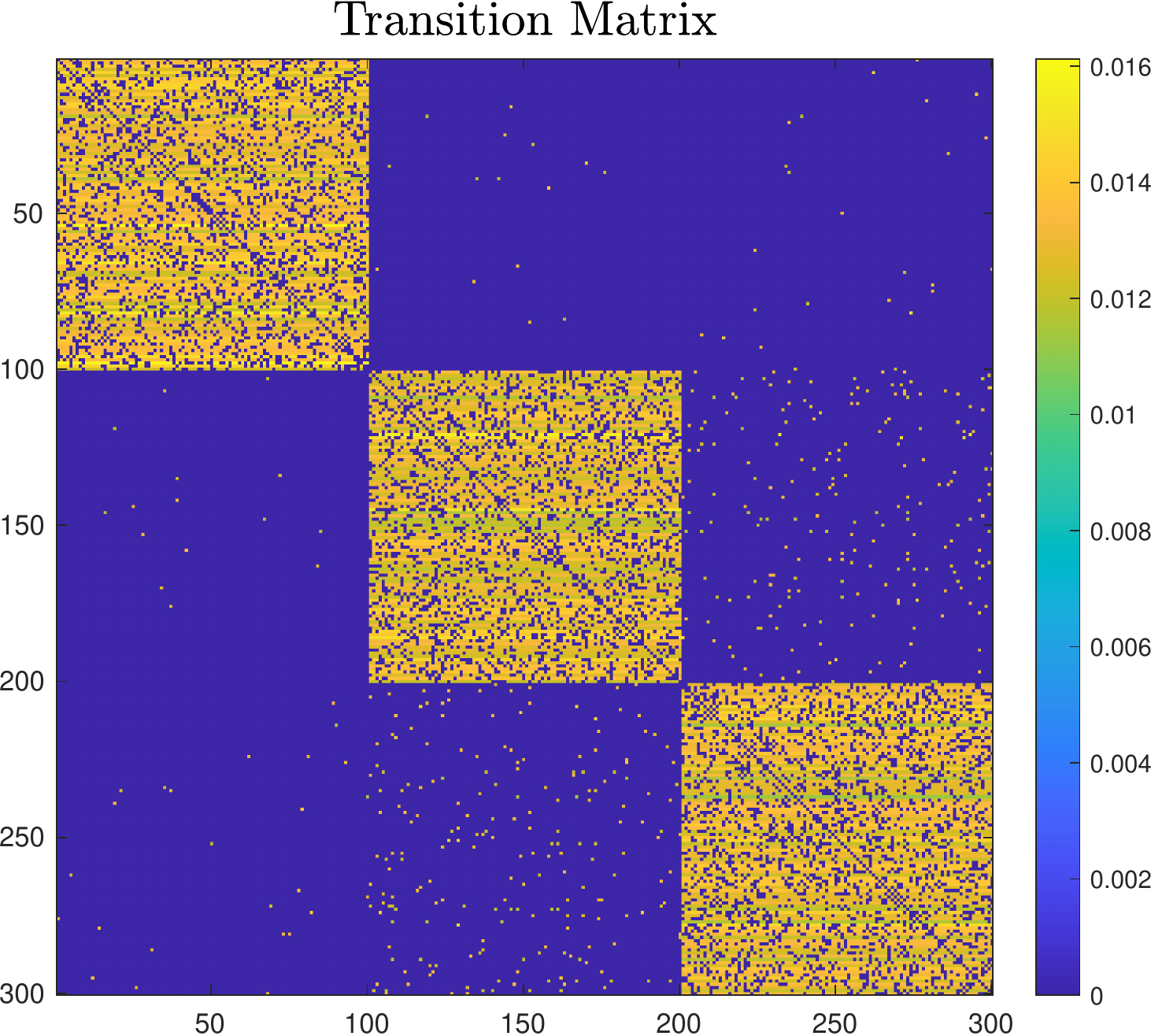}
		\includegraphics[width=\textwidth]{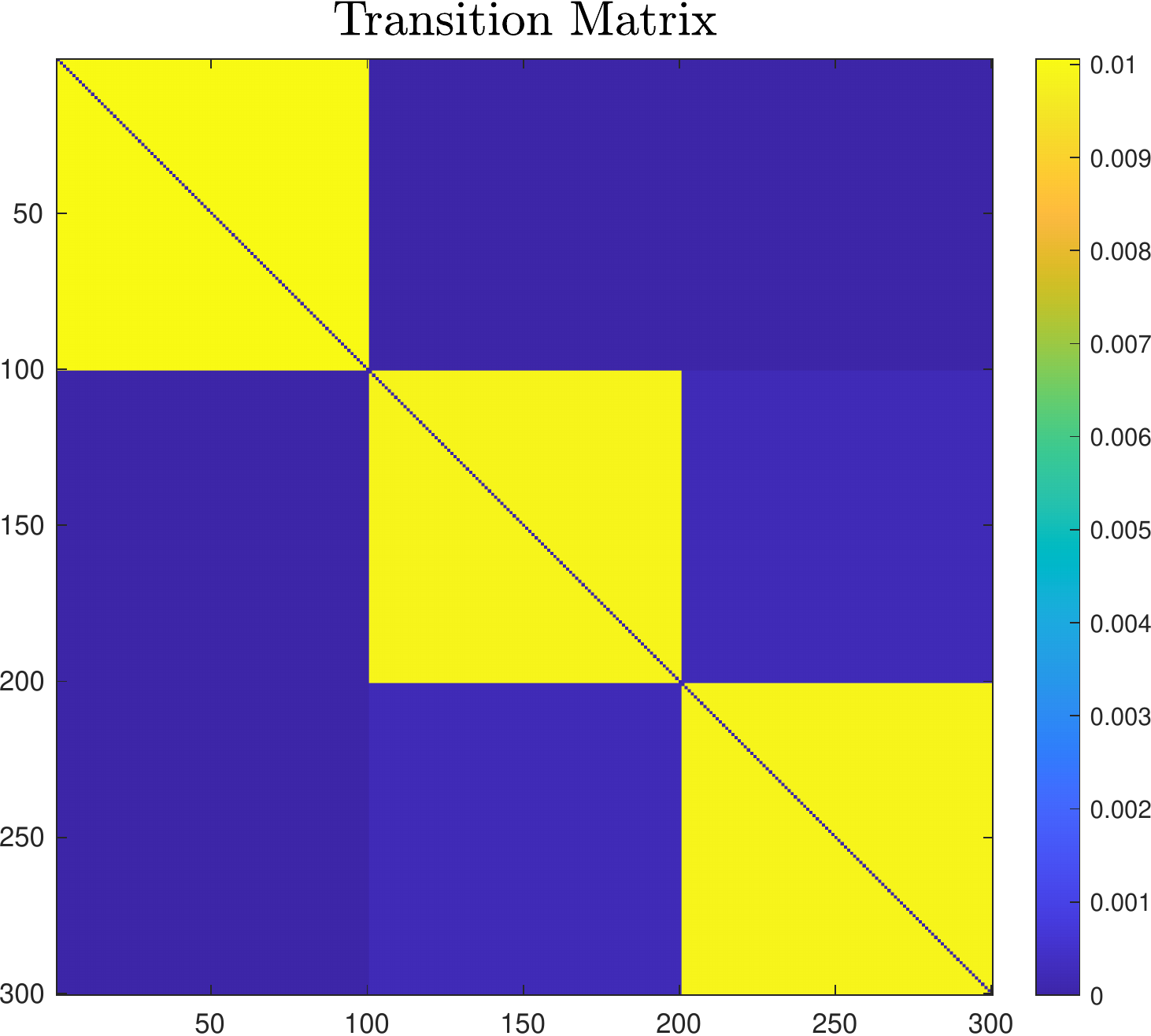}
		\includegraphics[width=\textwidth]{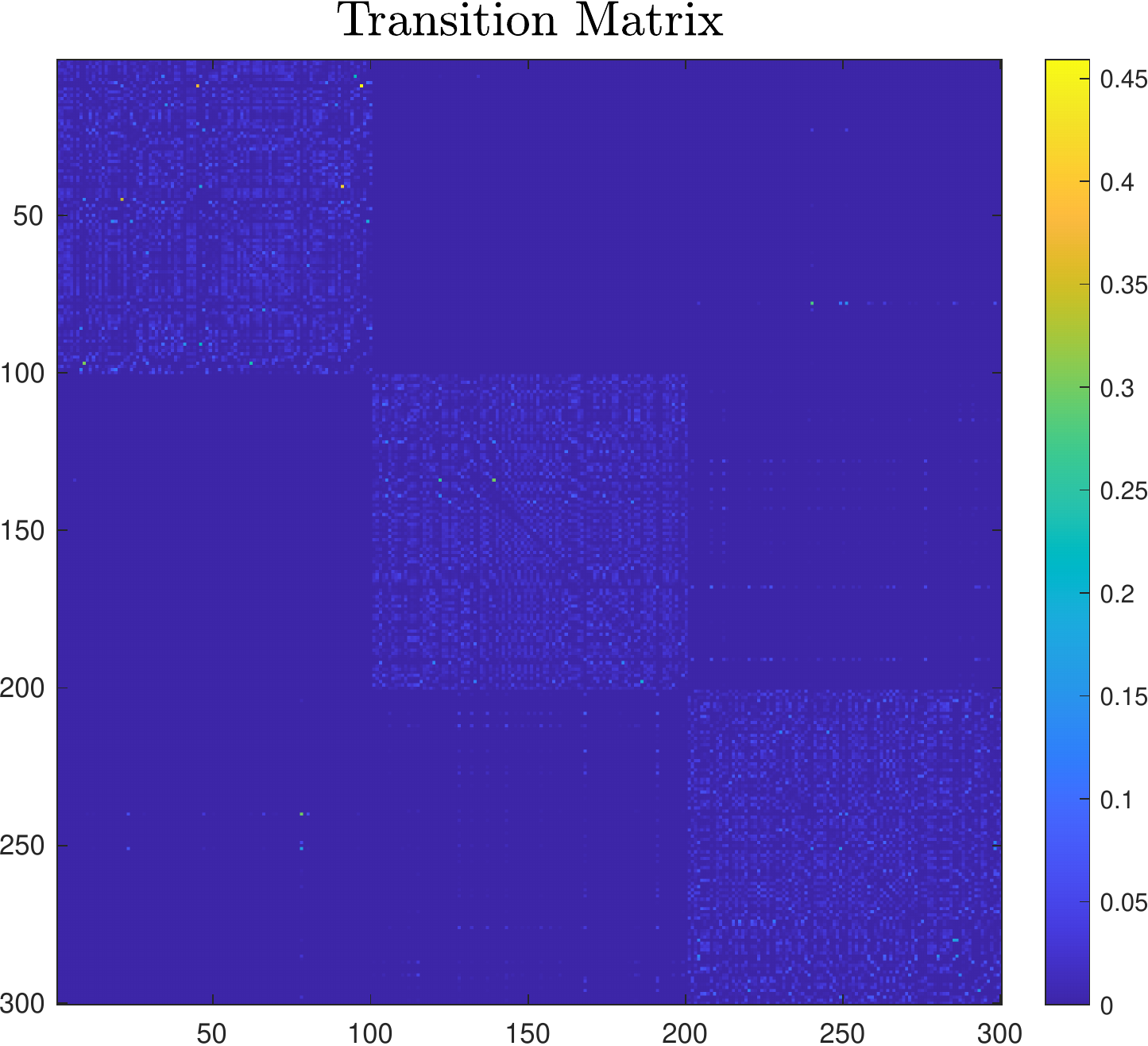}
\subcaption{Transition Matrix}
\end{subfigure}
\begin{subfigure}[b]{.23\textwidth}
\captionsetup{width=.95\linewidth}
	\includegraphics[width=\textwidth]{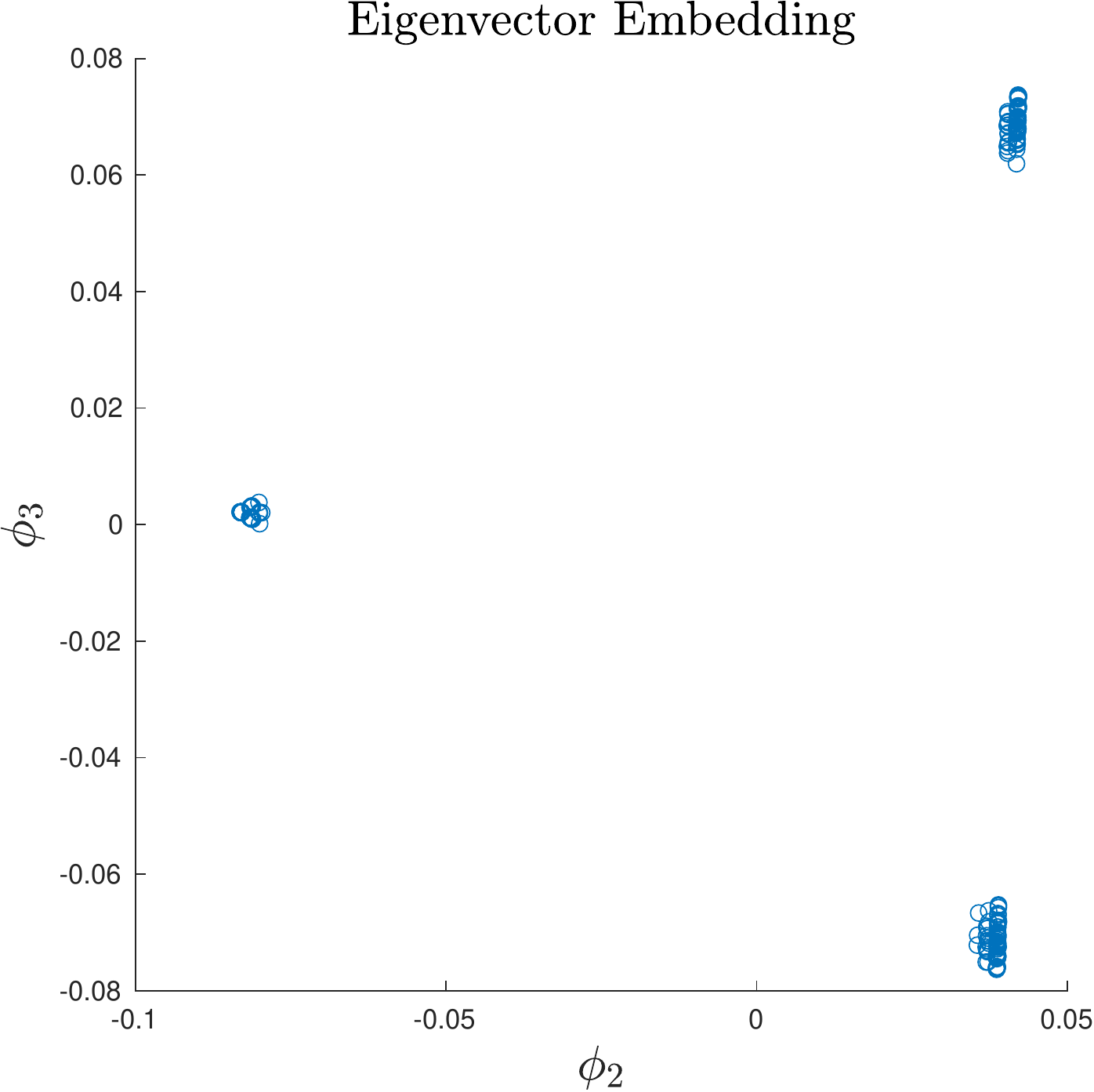}
	\includegraphics[width=\textwidth]{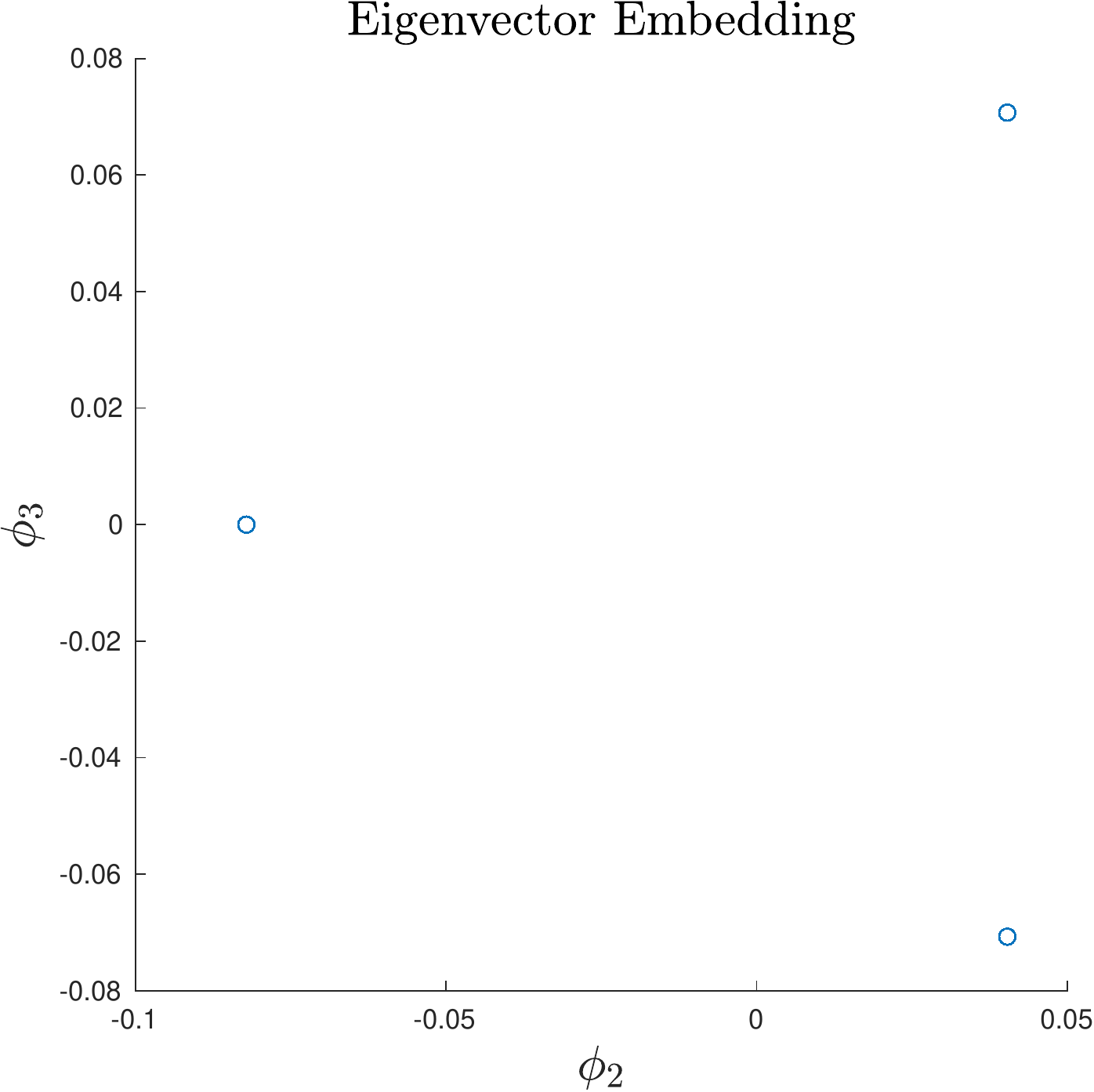}
	\includegraphics[width=\textwidth]{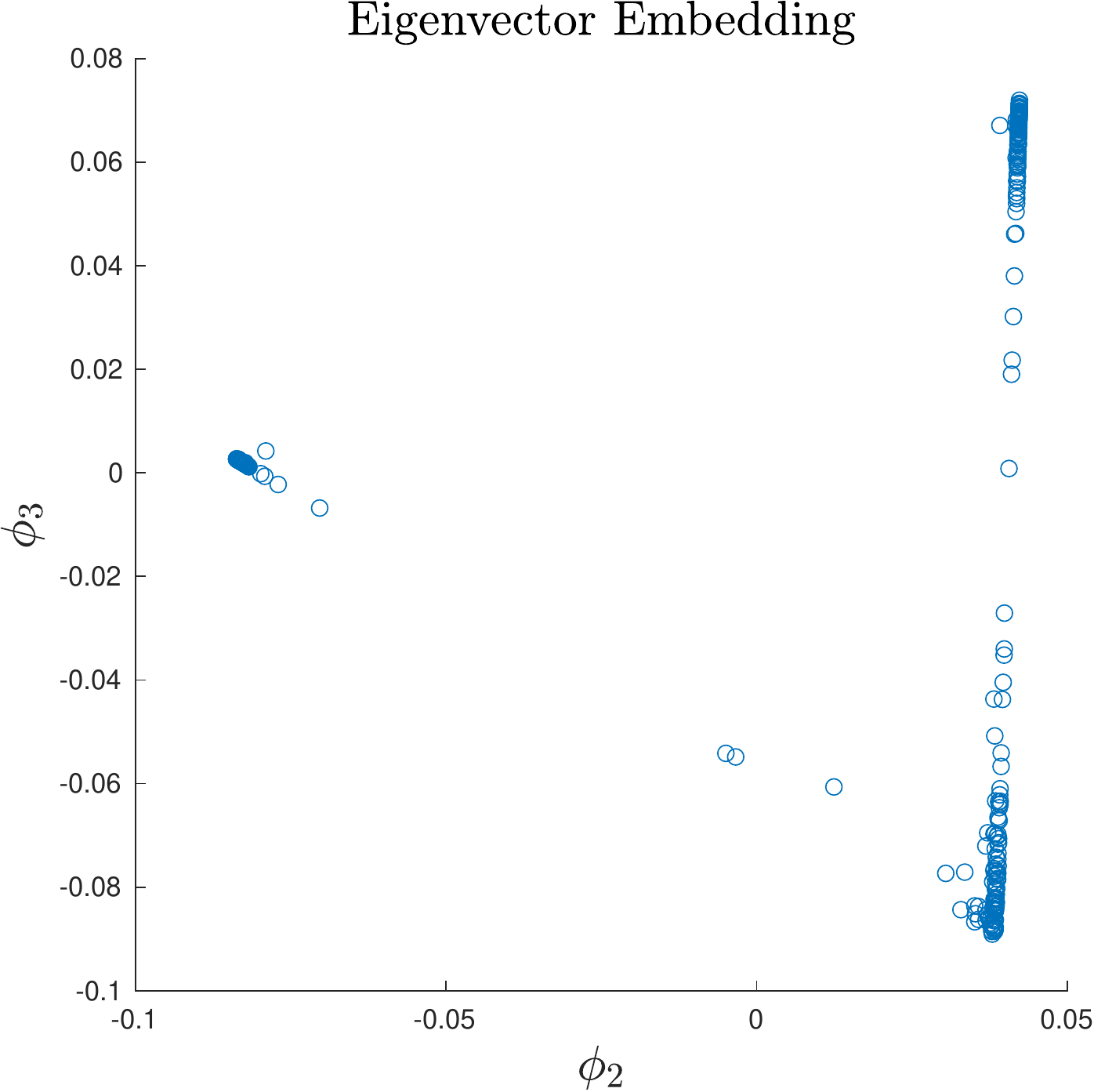}
\subcaption{LE Embedding}
\end{subfigure}
\begin{subfigure}[b]{.22\textwidth}
\captionsetup{width=.95\linewidth}
	\includegraphics[width=\textwidth]{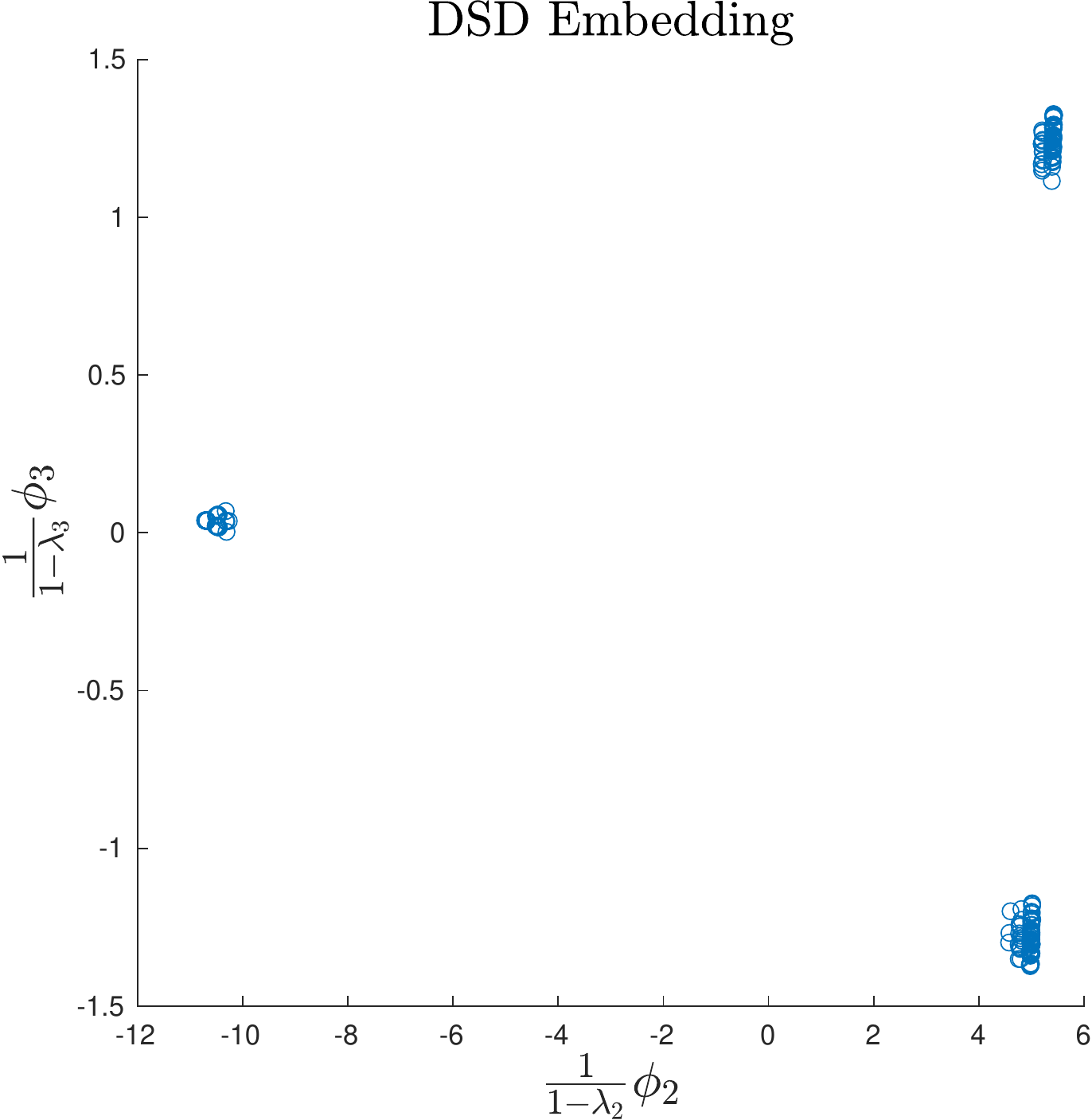}
	\includegraphics[width=\textwidth]{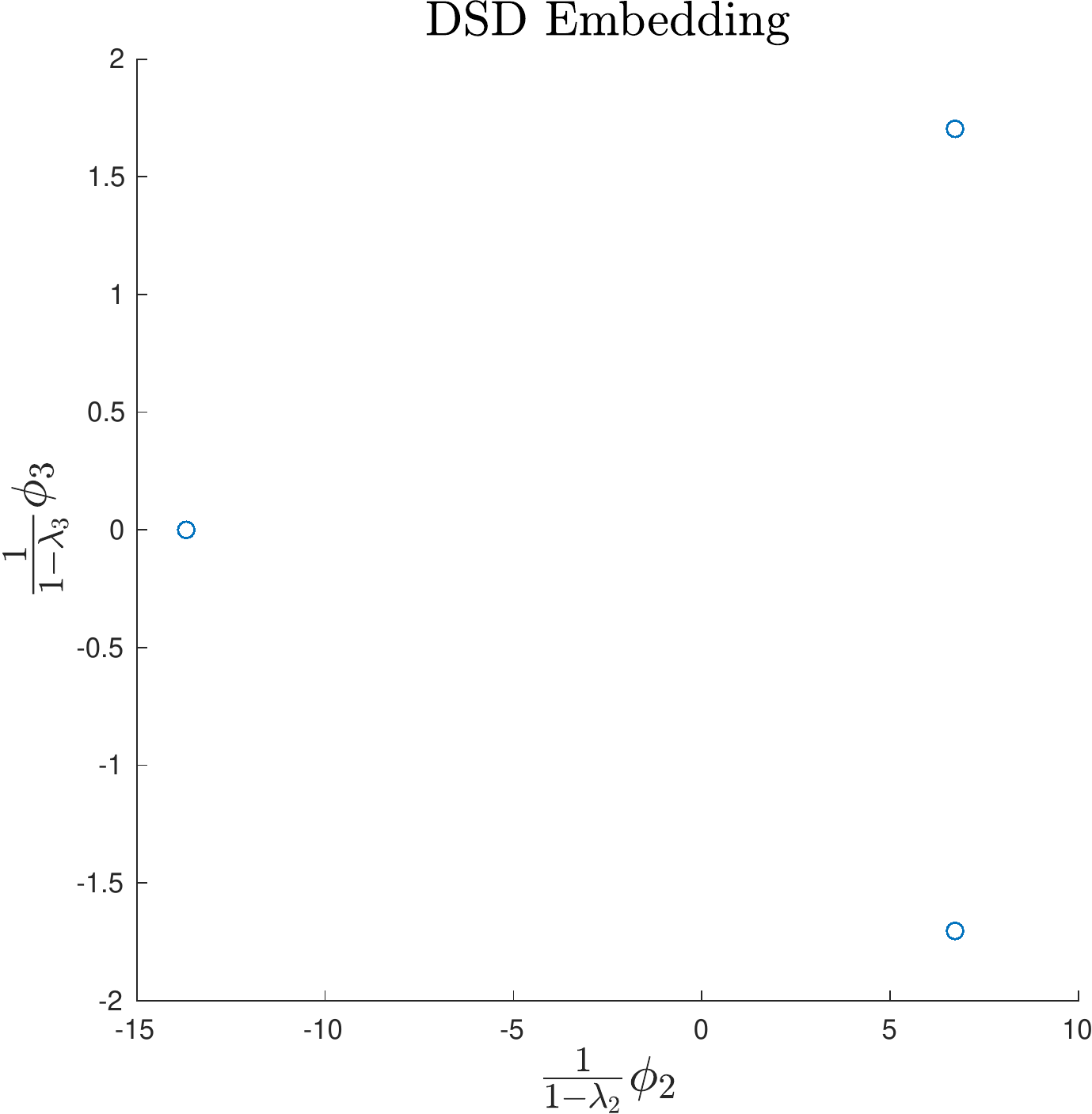}
	\includegraphics[width=\textwidth]{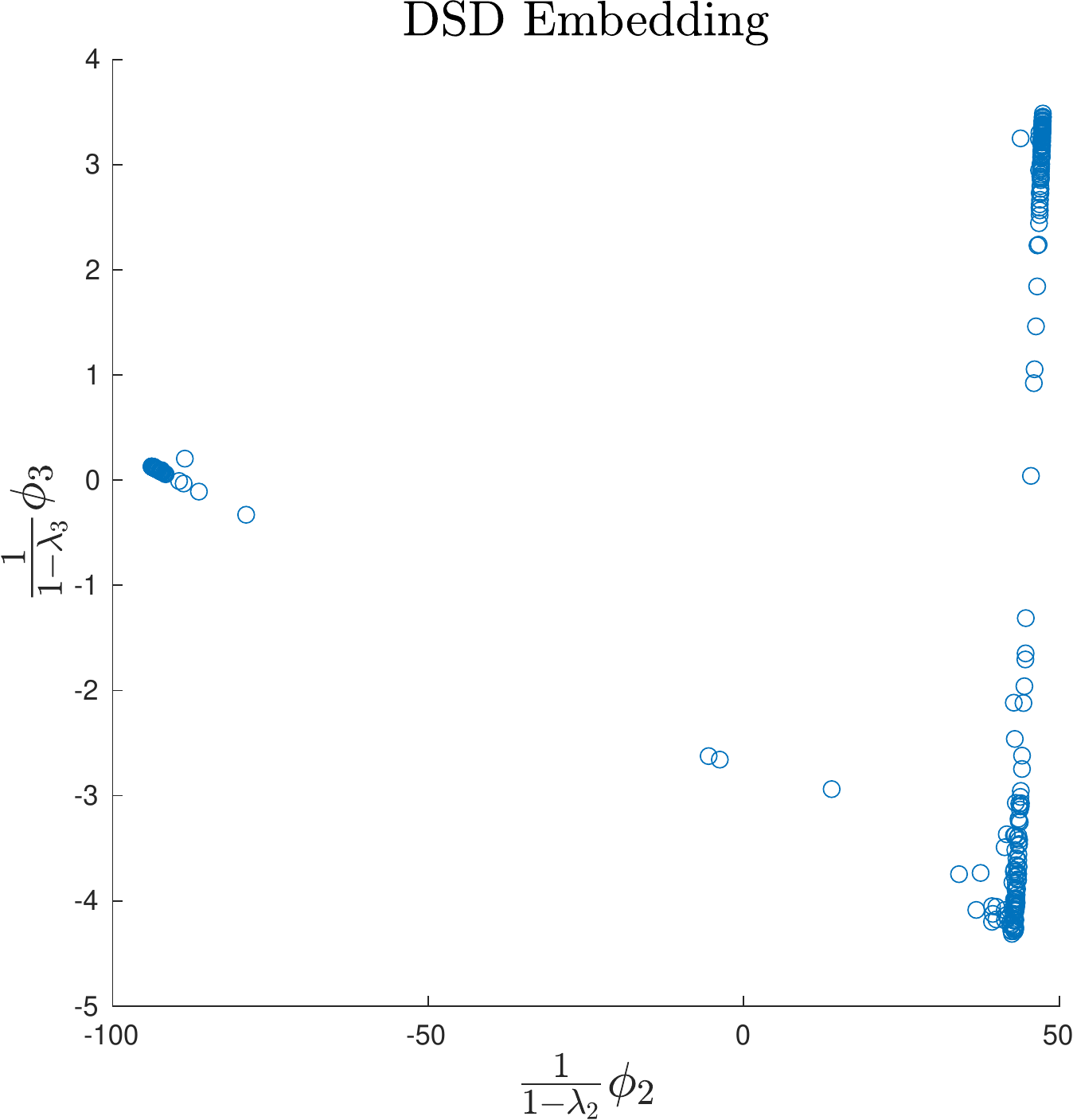}
\subcaption{DSD Embedding}
\end{subfigure}
\begin{subfigure}[b]{.23\textwidth}
\captionsetup{width=.95\linewidth}
	\includegraphics[width=\textwidth]{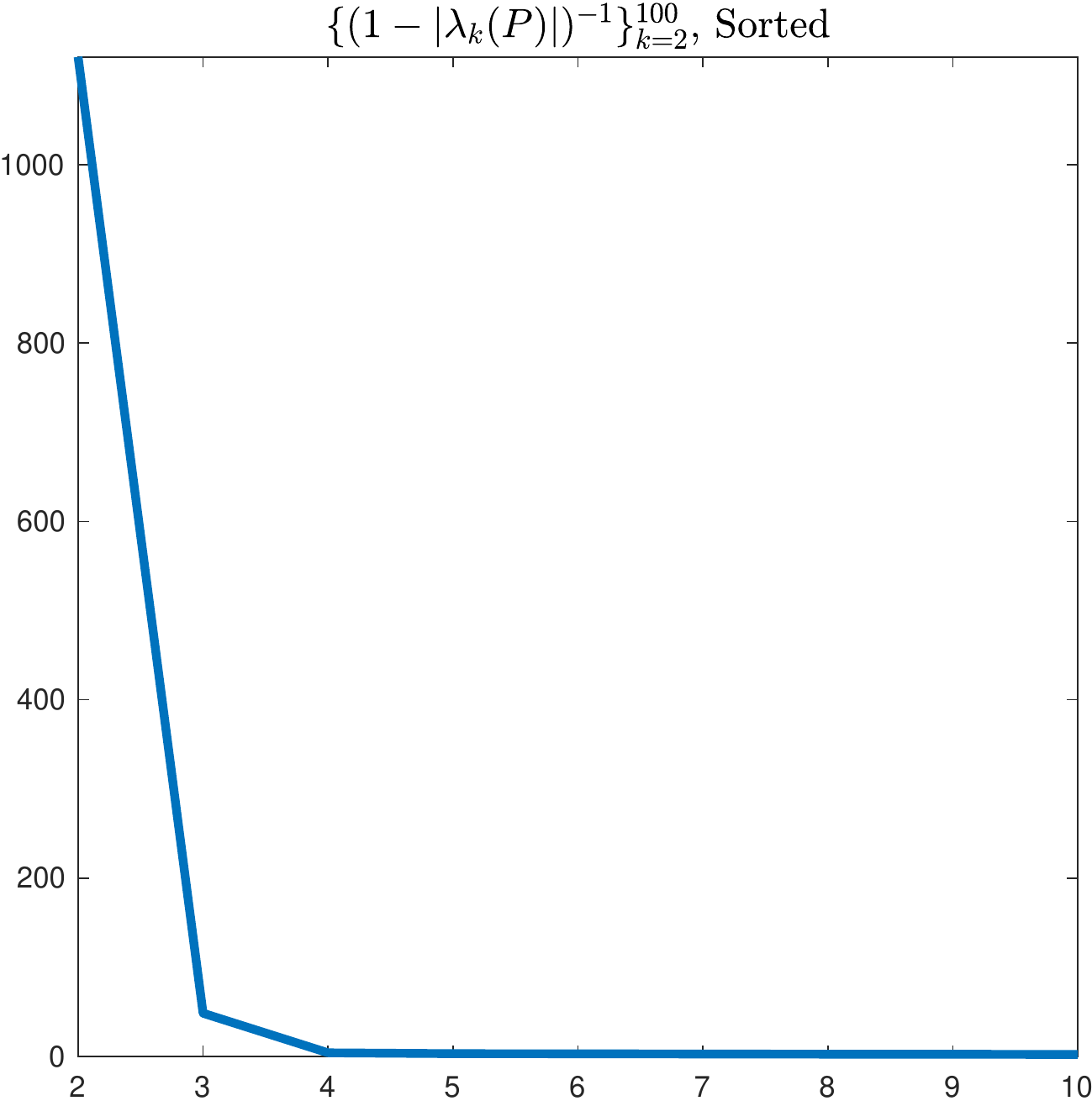}
	\includegraphics[width=\textwidth]{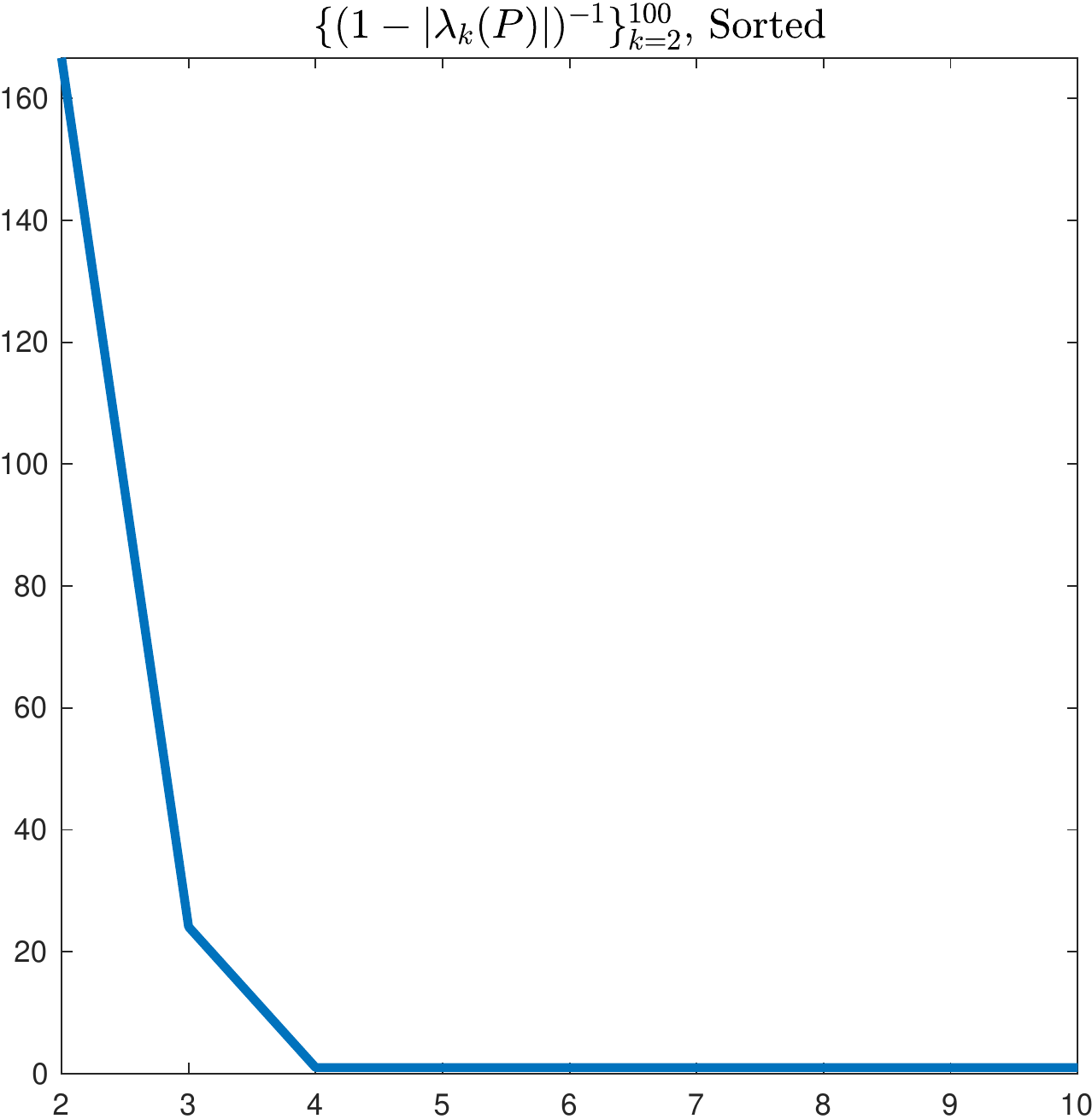}
	\includegraphics[width=\textwidth]{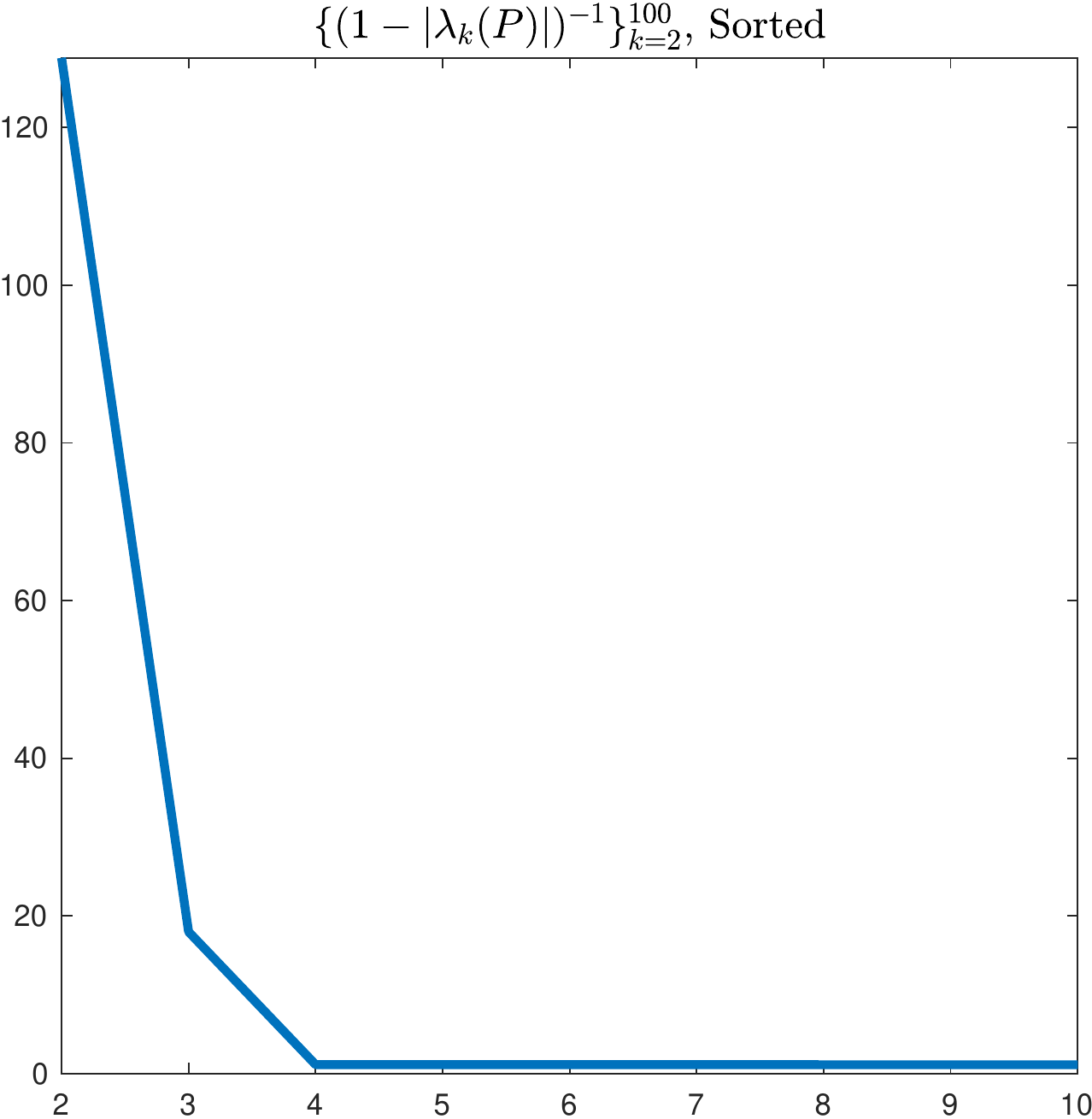}
	\subcaption{Reciprocal Eigenvalues}

\end{subfigure}
\caption{\emph{First row:} HSBM; \emph{second row:} low rank;  \emph{third row:} random geometric.  \emph{(a):} Transition matrices for each of the three proposed data models.  \emph{(b):}  Embedding the data using the $(\psi_{2},\psi_{3})$, i.e. the first non-trivial right eigenvectors of $P$.  \emph{(c):}  Embedding the data using the DSD coordinates $\left(\frac{1}{1-\lambda_{2}}\psi_{2},\frac{1}{1-\lambda_{3}}\psi_{3}\right)$.  In all cases, the DSD embedding produces coordinates that preserve the property that one pair of clusters are closer to each other than either is to the third.   \emph{(d):}  The largest values of $|(1-\lambda)^{-1}|$ for each $P$, sorted.  All data models have eigenvalues exhibiting fast decay toward zero, suggesting the utility of the low-rank representation.}
\label{fig:TransitionMatricesAndEmbeddings}		
\end{figure}

\subsubsection{Impact of Truncation on Denoising}

The coordinate representation (\ref{eqn:DSD_Embedding}) also has the impact of denoising the DSD.  Indeed, the high frequency eigenpairs (those with eigenvalue close to 0 in modulus) are statistically unreliable in the finite sample regime, and moreover do not correspond to geometrically meaningful features in the continuum domain.  In Figure \ref{fig:DSD_Matrices}, the DSD distance matrices for the multiscale datasets are computed directly, with the full spectral decomposition, and with the truncated (and therefore denoised) spectral decomposition using the first 10 eigenvectors.  We see that the denoised DSD distance matrices improve the contrast in separation between the clusters.

We also show the distance matrices when using the truncated LE embedding with the first 10 eigenvectors, and also the degree distances $d(x_{i},x_{j})=|\frac{1}{D_{i}}-\frac{1}{D_{j}}|$, to illustrate the distinct nonlocal behavior captured by the DSD on these data.  Both the LE distance and the degree distance are less informative than DSD, and often quite noisy.

\begin{figure}[!htb]
\centering
\begin{subfigure}[t]{.24\textwidth}
\captionsetup{width=.95\linewidth}
	\includegraphics[width=\textwidth]{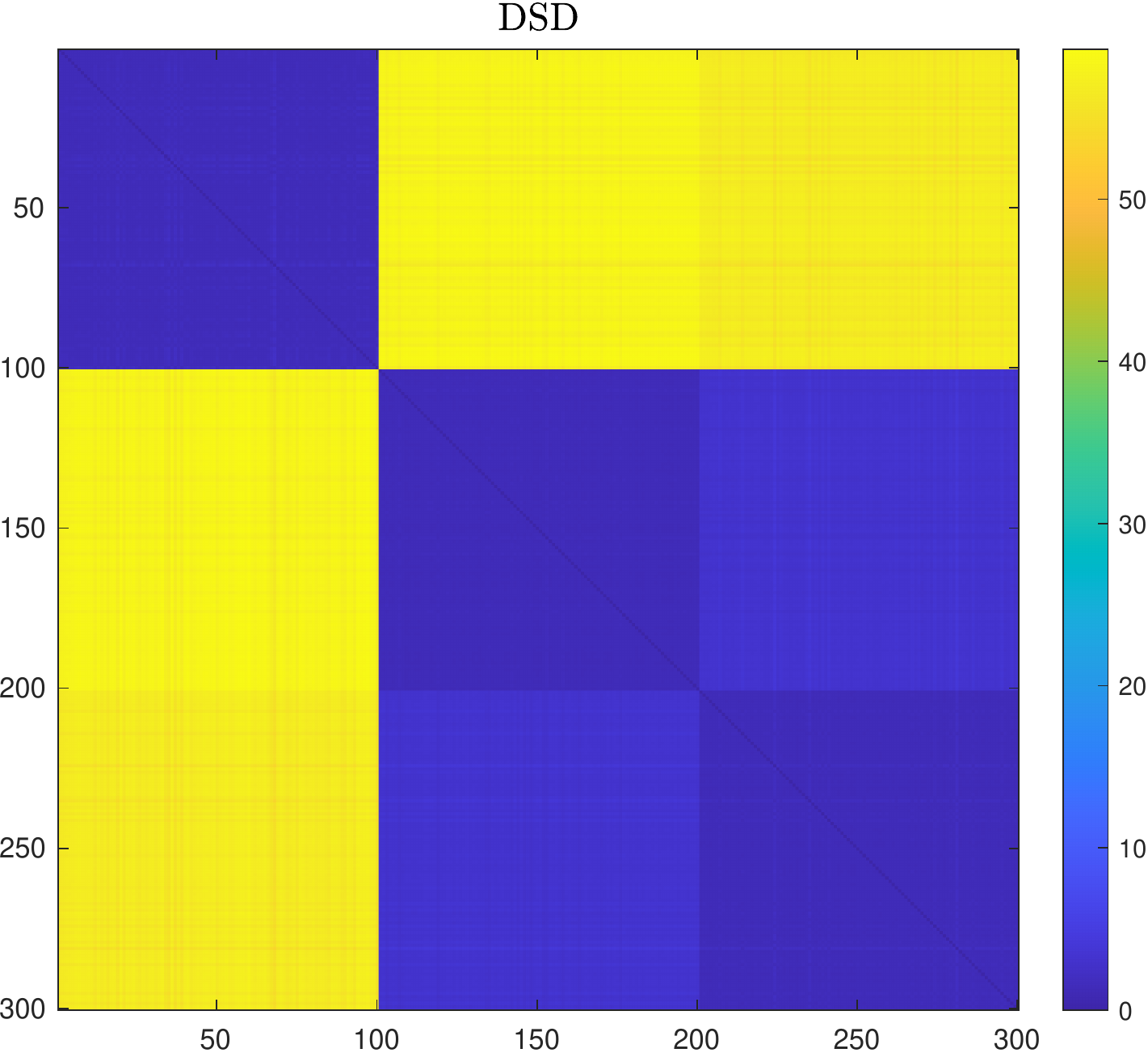}
	\includegraphics[width=\textwidth]{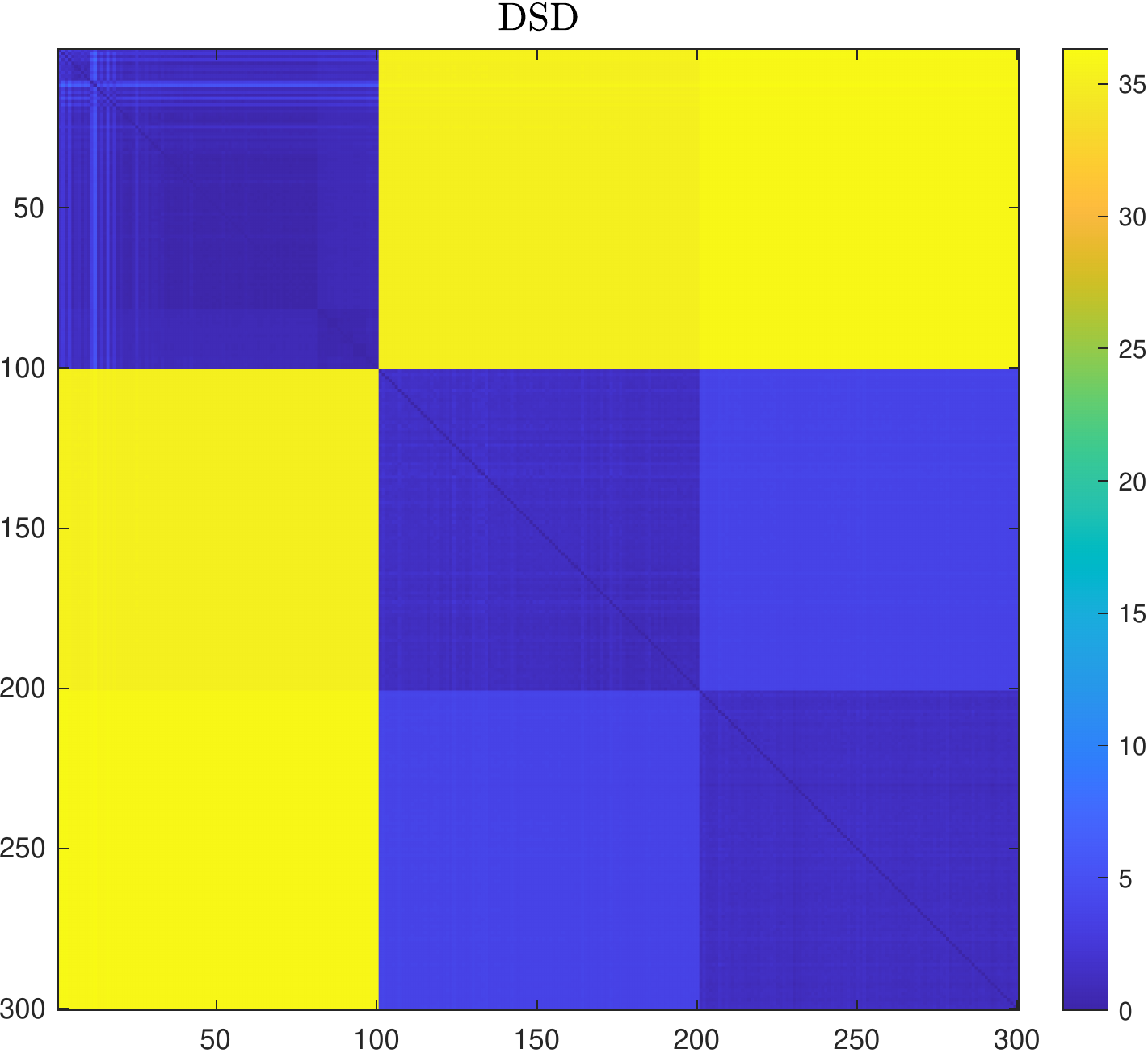}
	\includegraphics[width=\textwidth]{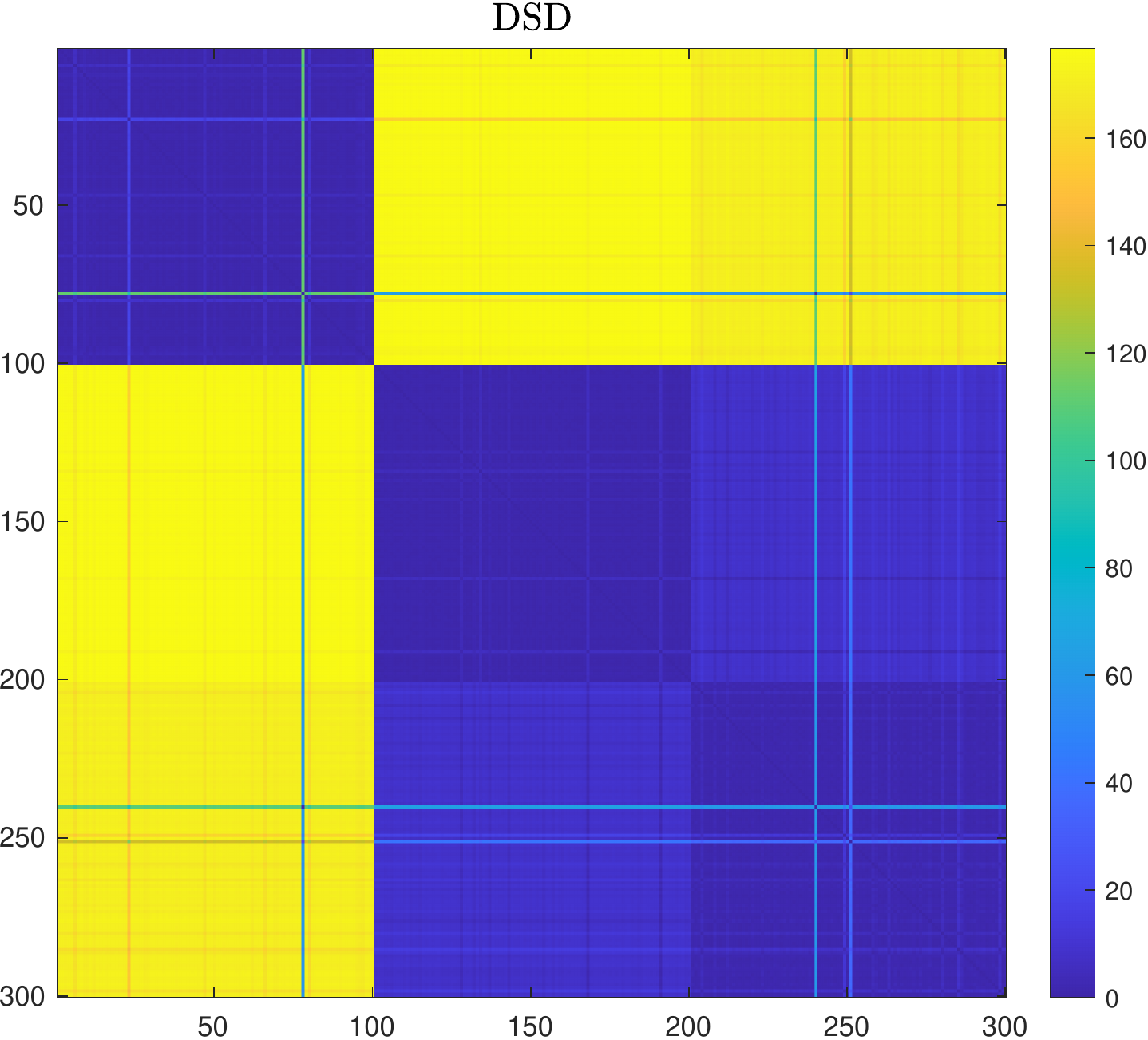}
\subcaption{DSD Distance Matrix}
\end{subfigure}
\begin{subfigure}[t]{.24\textwidth}
\captionsetup{width=.95\linewidth}
	\includegraphics[width=\textwidth]{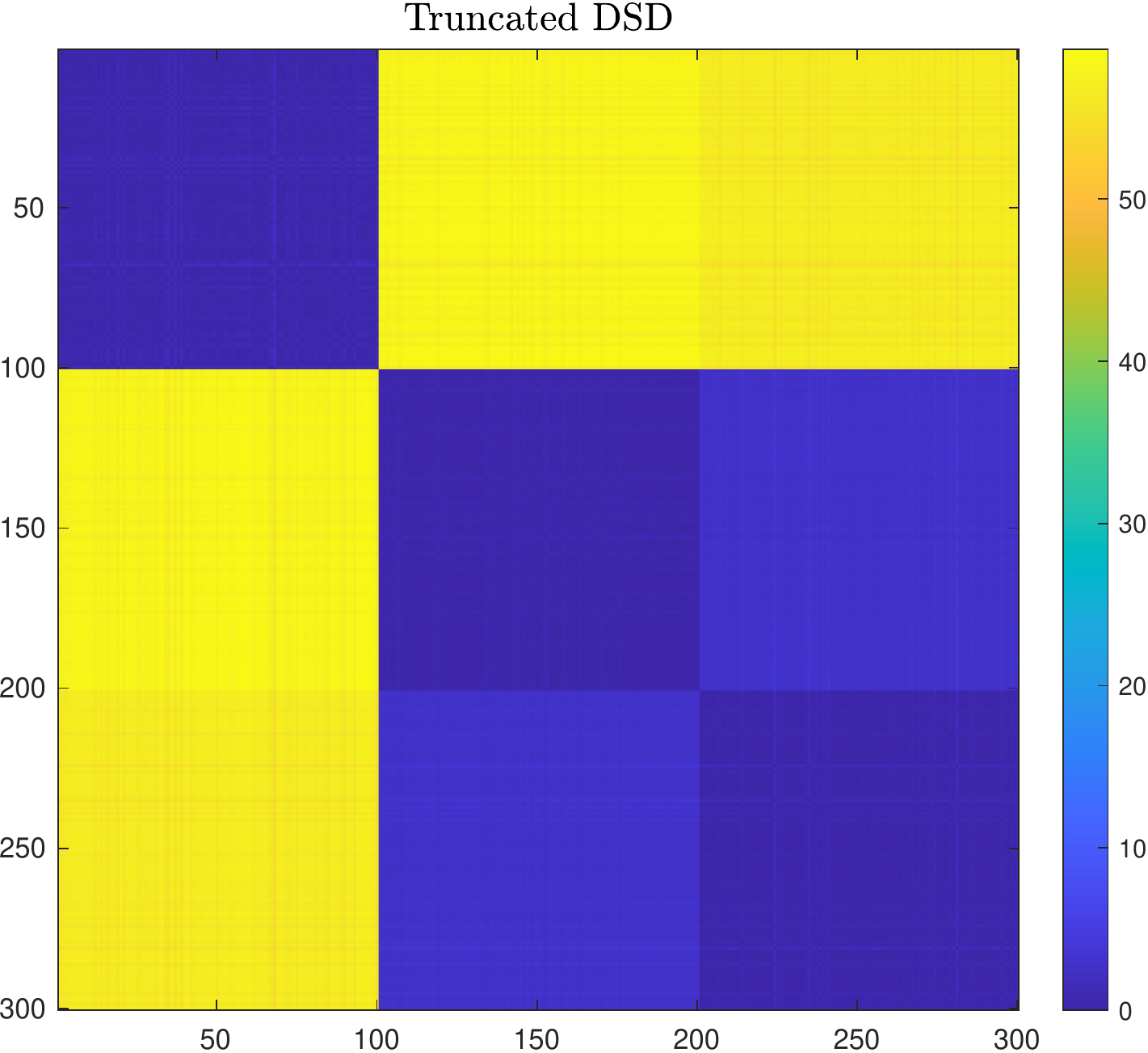}
	\includegraphics[width=\textwidth]{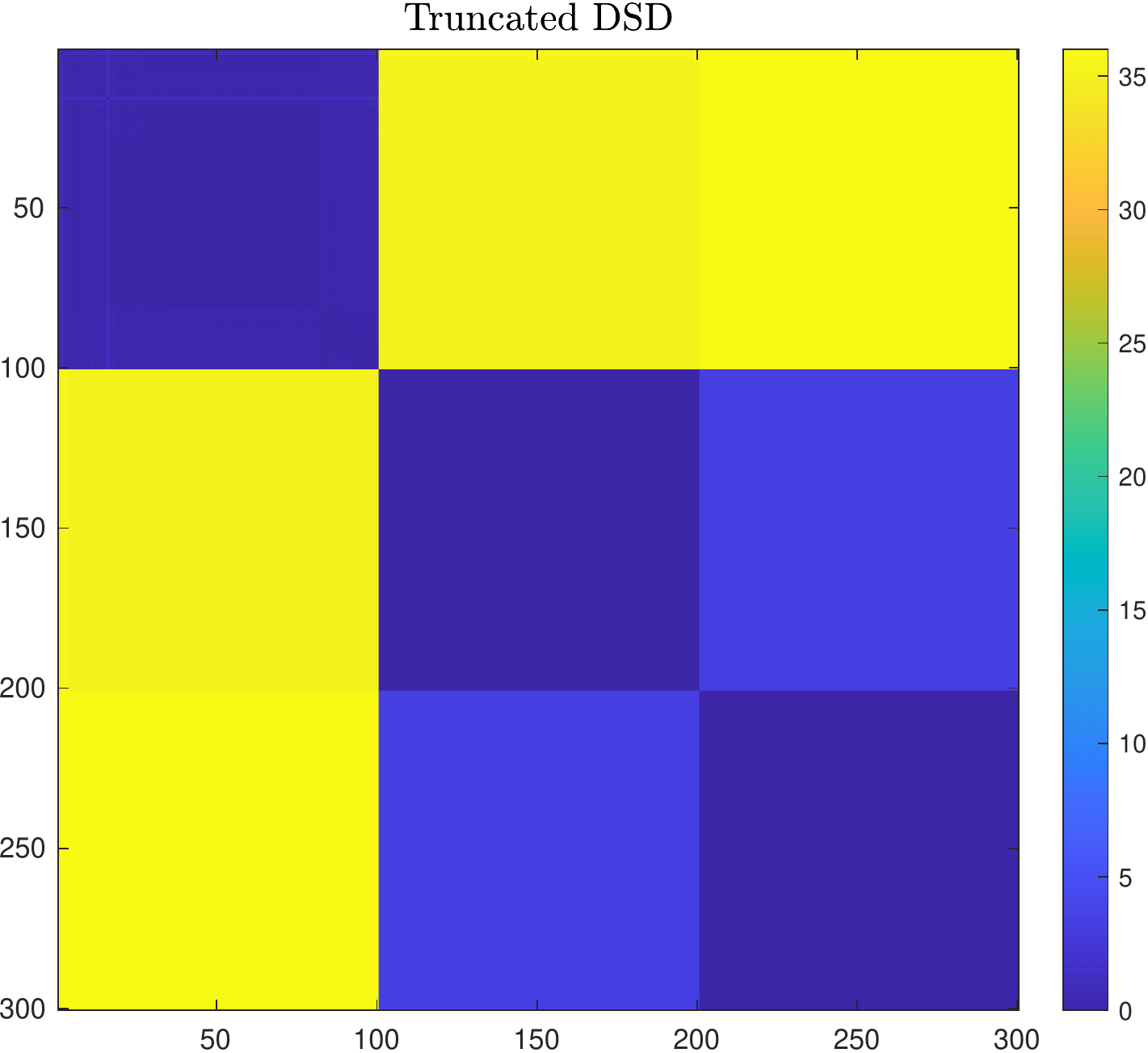}
	\includegraphics[width=\textwidth]{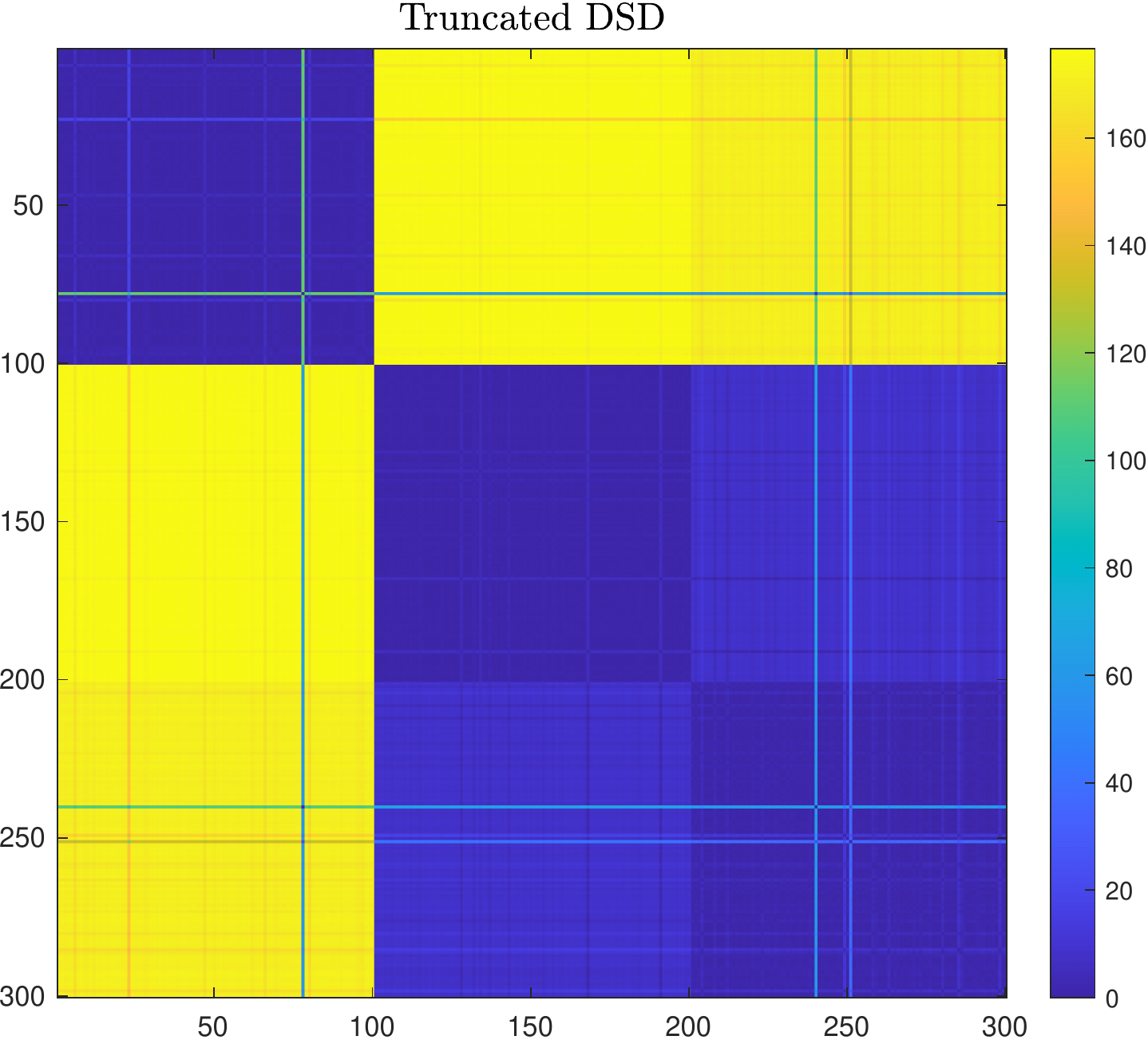}
\subcaption{Truncated DSD Distance Matrix}
\end{subfigure}
\begin{subfigure}[t]{.24\textwidth}
\captionsetup{width=.95\linewidth}	
	\includegraphics[width=\textwidth]{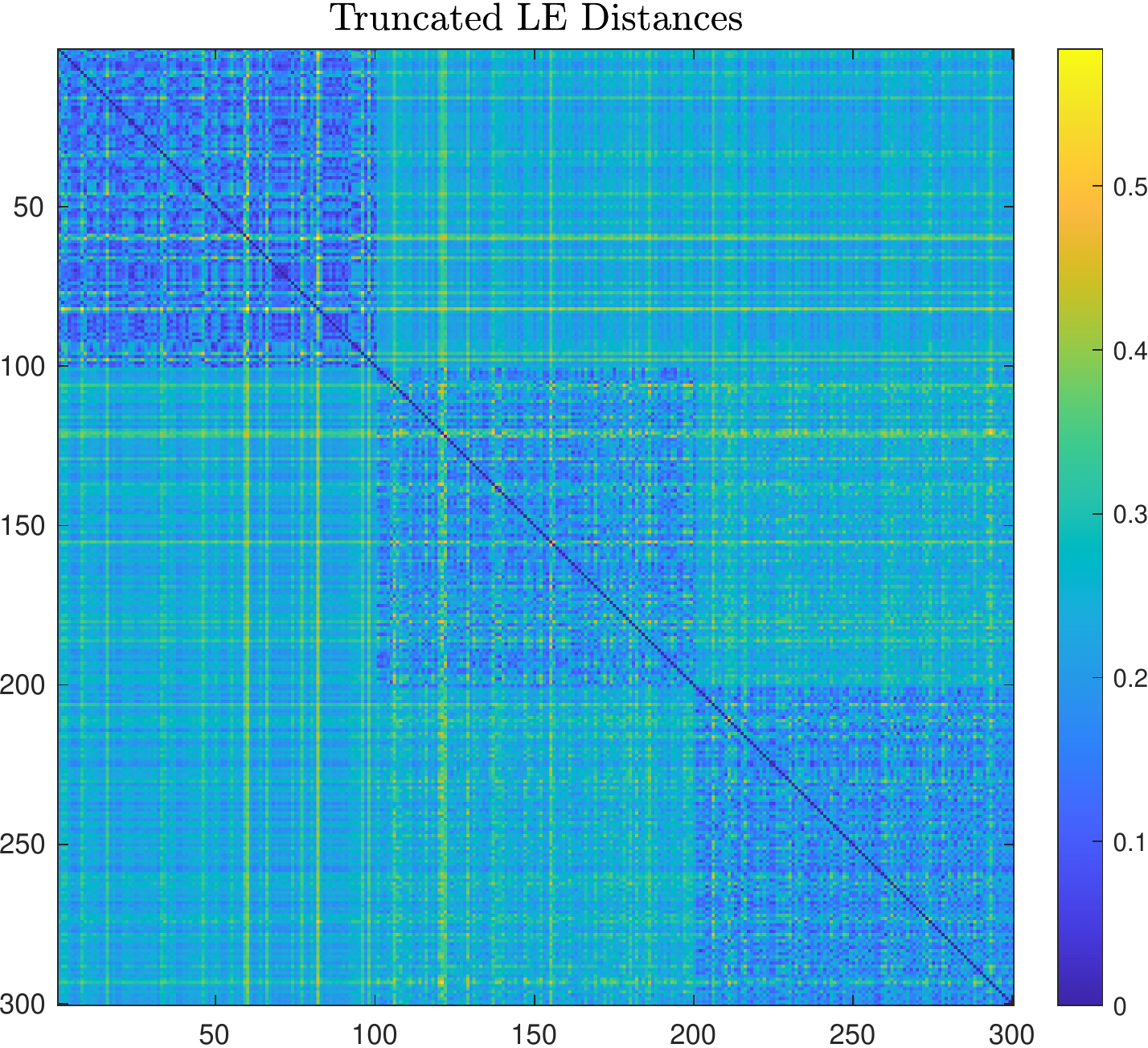}
	\includegraphics[width=\textwidth]{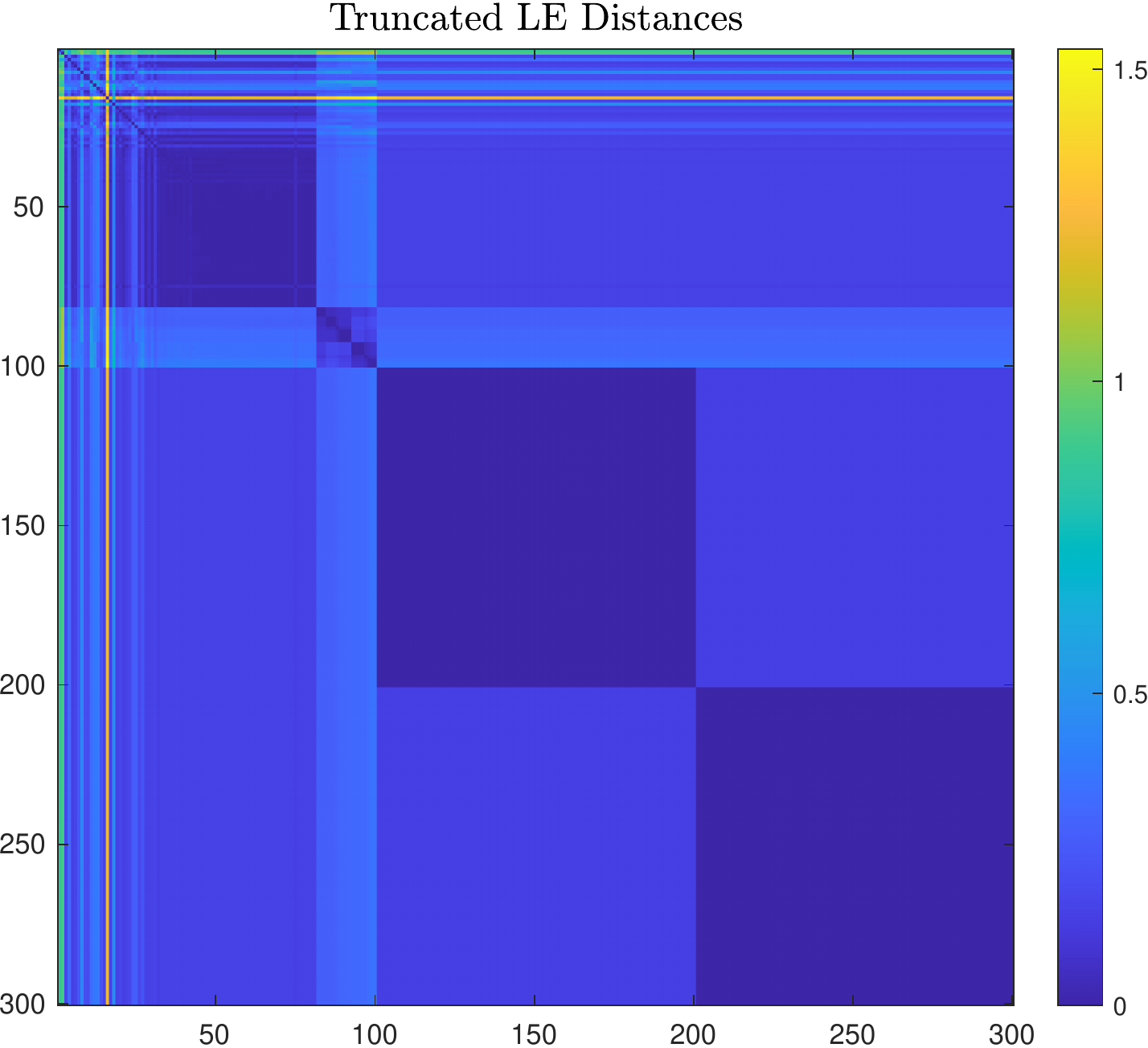}
	\includegraphics[width=\textwidth]{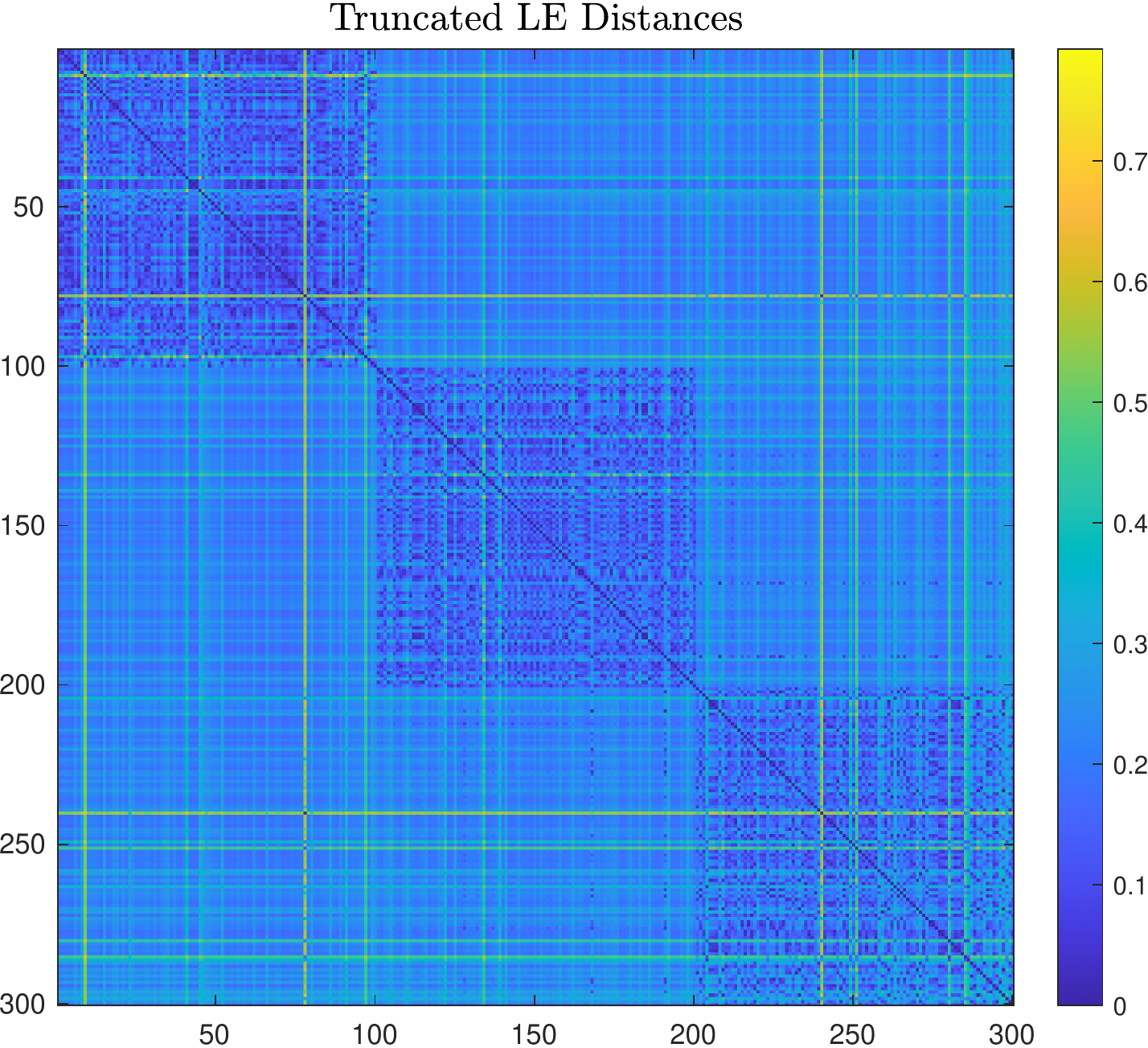}
\subcaption{Truncated LE Distance Matrix}
\end{subfigure}
\begin{subfigure}[t]{.24\textwidth}
\captionsetup{width=.95\linewidth}	
	\includegraphics[width=\textwidth]{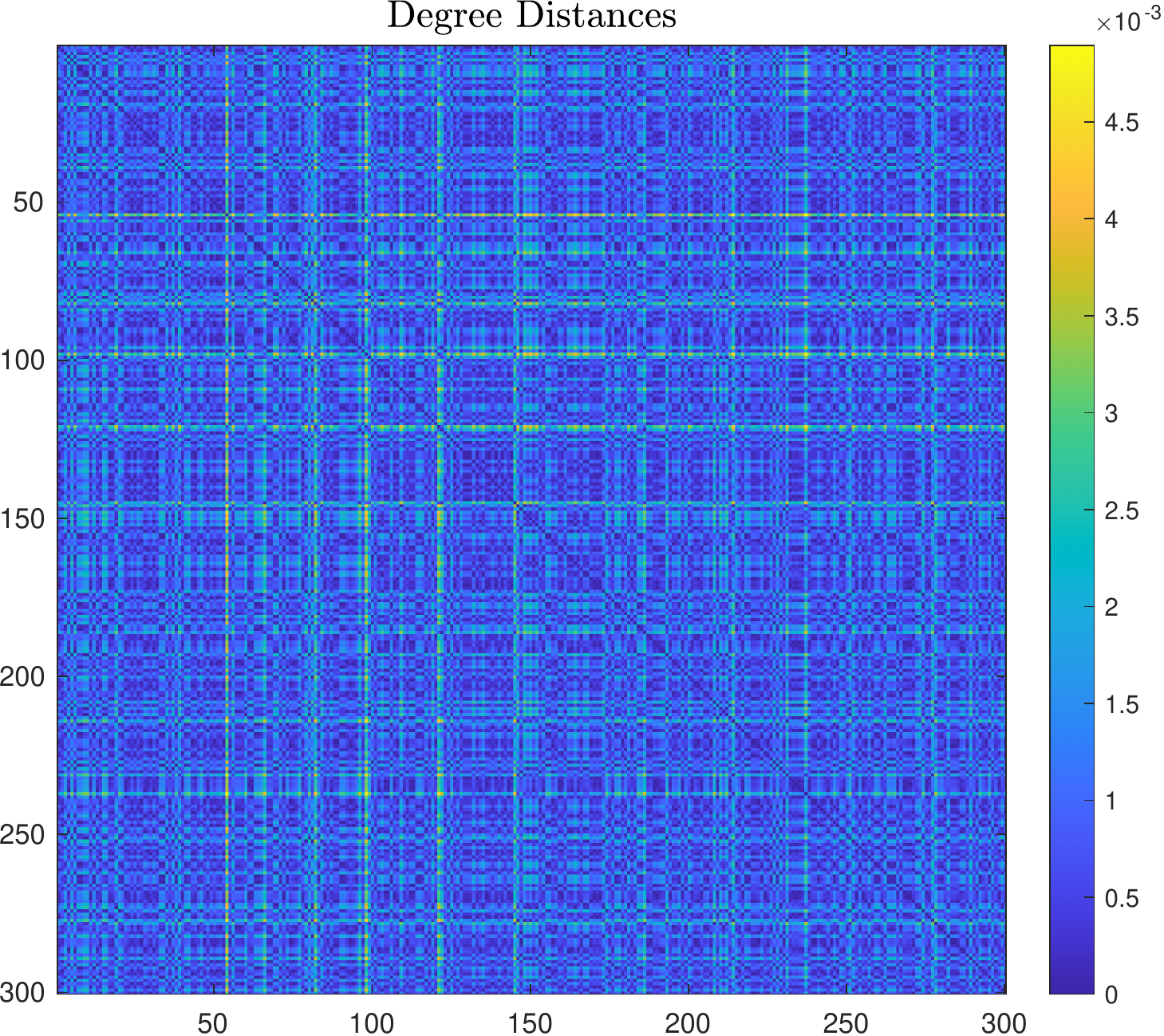}
	\includegraphics[width=\textwidth]{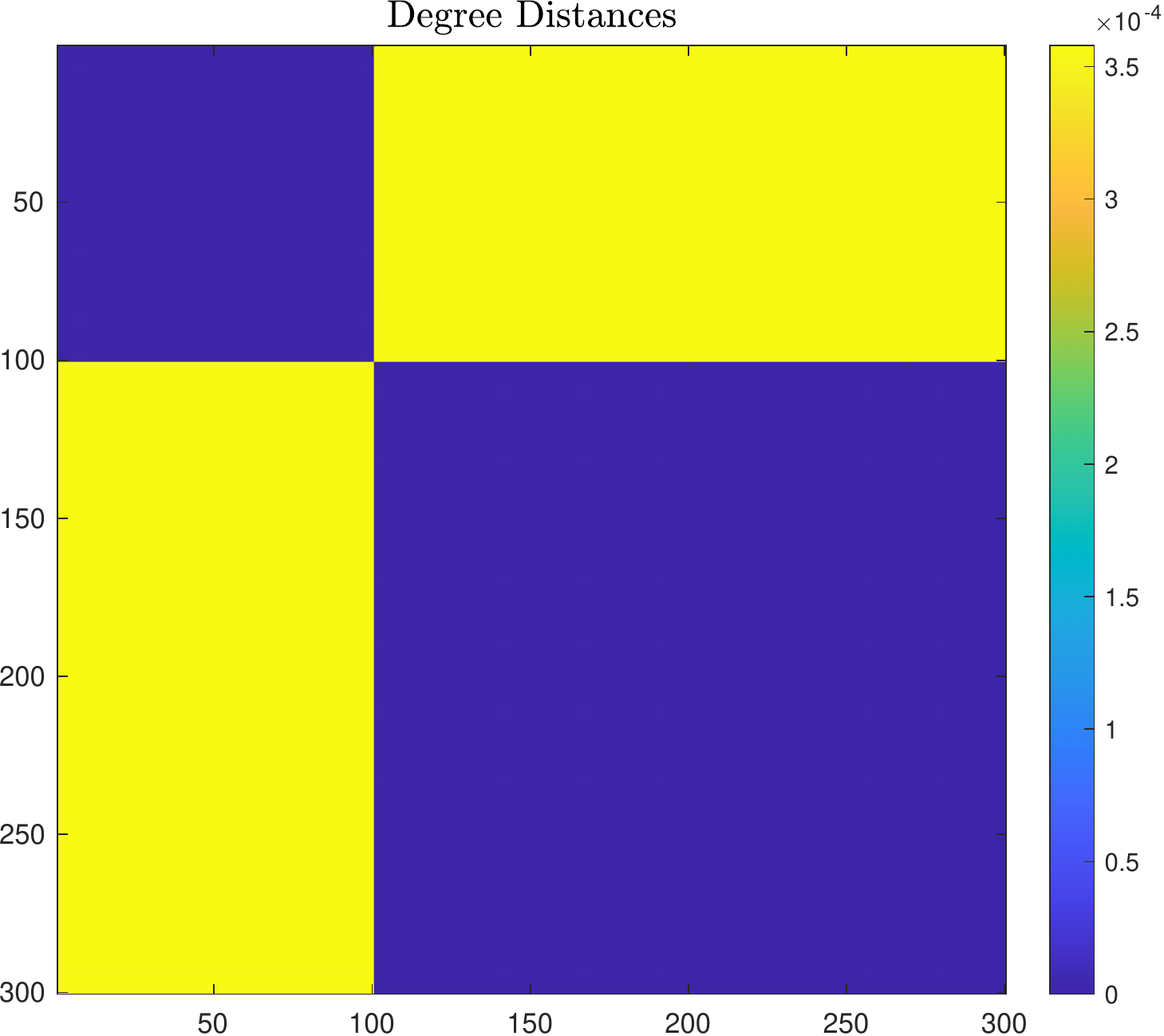}
	\includegraphics[width=\textwidth]{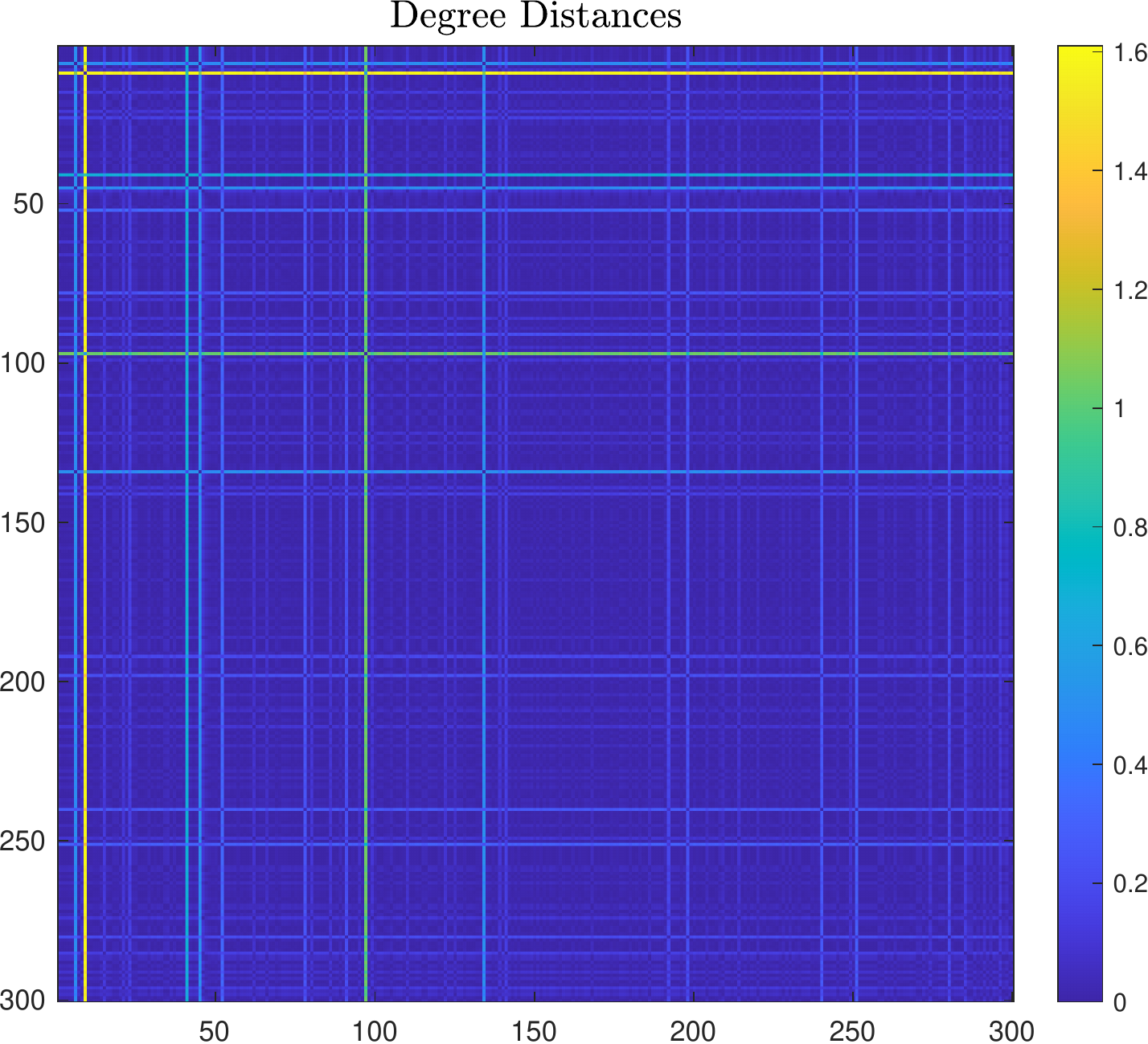}
\subcaption{Degree Distance Matrix}
\end{subfigure}
\caption{ \label{fig:DSD_Matrices}\emph{First row:} HSBM; \emph{second row:} low rank;  \emph{third row:} random geometric.  \emph{(a):}  Exact DSD matrix, computed by inverting the regularized Laplacian.  \emph{(b):}  The truncated spectral DSD, given by only using the first three eigenvectors.  This row shows the benefit of truncating the spectral decomposition when $P$ is intrinsically nearly low-rank.  We see that the distance matrices show more contrast compared to the full DSD, and their computation is much faster.  \emph{(c):}  Truncated LE distances.  We see the distances are quite noisy, and also fail to capture the multiscale structure in the data.  \emph{(d):}  Degree distances are largely uninformative, as expected.  Note that the DSD and the degree distances are quite different, possibly due to the fact that the underlying diffusion matrices mix slowly, and consequence of a small spectral gap.}
\end{figure}

\subsection{Experiments on Biological Networks}\label{subsec:BiologicalNetworks}
In humans, as well as several well-studied model organisms, biologists study different types of gene-gene or protein-protein association networks, where the vertices represent genes or proteins, and two proteins are connected by an edge based on different criteria (depending on the network). For example, the classical protein-protein interaction networks connect two proteins if there is experimental evidence that they bind in the cell, with an edge weight that corresponds to either the confidence in the experimental evidence, or (in some networks) the predicted strength of the interaction. Other networks connect two genes if they are typically expressed in the same human tissues, or if their genetic sequence is sufficiently similar. The DREAM Disease Module Identification Challenge~\cite{Choobdar2019_Assessment} collected a heterogeneous set of six different human  protein-protein association networks; here we consider the first two of these networks, DREAM1 and DREAM2. DREAM1 is derived from STRING \cite{Szklarczyk2014_String} which  integrates known and predicted protein-protein associations from various sources, including those aggregated from databases that collect experimental evidence for protein-protein interactions, those derived from co-expression, and others either inferred by literature annotation or transferred from interactions of sequence-similar proteins in other species.   Both physical protein interactions (direct) and functional associations (indirect) are included. DREAM1 includes all the different interaction types in STRING v.10.0 excluding the interactions derived from text mining.   Edge weights correspond to the STRING association score (after removing evidence from text mining). DREAM2 intends to represent a more classical protein-protein interaction network, where the presence of an edge between proteins indicates that there is evidence that proteins physically bind in the cell. It is derived from the InWeb database~\cite{li2017scored}, where edges come from either humans or are supported by the presence of a corresponding interaction in multiple different model organism databases (InWeb requires each such interaction in four separate model organism databases before it includes it in its network as also a human interaction edge with confidence).  These networks are dense and highly connected (see Table~\ref{tab: DREAM largest_conn_comp}), though they can have some isolated nodes. To ensure the assumption of  irreducibility for convergence, we restrict to the largest connected component of the network if it is not originally connected. The edge weights reflect the confidence of the interaction (See Table \ref{tab: DREAM largest_conn_comp}).

\begin{table}[!htb]
 \centering
\begin{tabular}{|c|c|c|c|}
\hline
 & Vertices & Edges  & Average Degree                         \\ \hline
DREAM1 & 17,388   & 2,232,398 & 128.4                        \\
DREAM2 & 12,325   & 397,254   & 32.2               \\
\hline
\end{tabular}
\caption{Number of vertices, number of edges, and average degrees for the largest connected components in the DREAM1 and DREAM2 networks.}
\label{tab: DREAM largest_conn_comp}
\end{table}

To illustrate the value of DSD for the biological networks, as well as the virtue of the DSD approximation via eigenvector truncation, we consider two classical problems in computational biology: \emph{link prediction} and \emph{function prediction}.

\subsubsection{Link Prediction}
Despite the large number of edges, it is assumed that many remaining true edges are missing from the PPI network, because some of the underlying interactions have not yet been experimentally observed. The goal of \emph{link prediction} is to estimate these missing links, based on the already observed ones.  A solution to the link prediction problem is an ordered ranking of all possible pairs of vertices that are not directly connected by an edge in the network. A classifier predicts the edges ranked above a threshold to be ``true edges" and below the threshold to be ``non-edges" where varying the threshold allows a tradeoff between true and false positive edge predictions. Link prediction algorithms help prioritize new pairs of proteins to experimentally test for interactions.

Each of the DREAM1 and DREAM2 networks are very large and have dense, highly connected cores.  Because of these densely connected cores, the overwhelming majority of missing links in the graphs are between nodes the core region. As the core is highly connected, a simple heuristic measure of ranking the node pairs based on the number of common neighbors will produce an excellent ranking when applied to the entire network, since we expect every node pair in the core to have some neighbors in common.  However, we are often most interested in predicting new interactions outside the common core.  In order to avoid simply discovering the common core, and also to work with smaller and more computationally tractable examples, we consider the link prediction problem on sparser and smaller samples of the DREAM1 and DREAM2 networks. More precisely, we randomly sample 400 nodes from the network, and consider the ``full" network of their induced subgraph.  We then compute a ``partial" network, consisting of the same subgraph of nodes with 10\% of the edges removed uniformly at random. While removing the edges, we ensure graph connectedness using Algorithm \ref{alg:ConnComp}.

\begin{algorithm}[htp!]
	\caption{Generate of random connected subgraphs} \label{alg:ConnComp}
	\begin{algorithmic}[1]
	    \STATE Input a graph $G=(X,W)$ as a list of edges, with $M = |W|$.
		\STATE Randomly permute the edge indices $\{1,\dots, M\}$.
		\STATE Create a spanning tree from this ordered edge list, prioritizing the edges with smaller assigned number. This ensures that the subgraph is connected and has all the nodes of the graph.
        \STATE Protect the edges of the resulting tree; remove $0.10M$ of the edges from the remaining graph uniformly at random.
	\end{algorithmic}
\end{algorithm}

Link prediction ranking methods can then be compared against each other by starting with the partial network, and assuming the missing edges from the full network are the true positive links, and all unobserved links from the full network are negative.  In this sense, we may understand the partial network as a training set, and the goal is to predict the full network.  DSD-based link prediction is done by ranking every node pair in the graph in ascending order based on their DSD distances, with the interpretation that the lower the DSD distance between two nodes in the graph, the more likely the nodes are to have a link between them.  The heuristic comparison methods considered---weighted common neighbor, Jaccard's coefficient, and the weighted Adamic-Adar index (defined in \cite{Liben2007_Link})---are local scoring methods, where the link scores between two distinct nodes in the graph are computed based on their neighborhood affinity. Since a larger score implies that the two nodes are closely associated in the network, a link between them becomes more likely. By ranking all the edges in the ascending order of the score and providing a threshold, we get a score-based method that predicts the missing link in the graph.  Table \ref{table:link_pred_heuristics} shows how the heuristic scores are computed for two arbitrary nodes $x_{i},x_{j} \in X$ in a graph $G = (X,W)$.  We also compare to using \emph{diffusion distances} \cite{Coifman2006_Diffusion} (see Definition \ref{defn:DD}) at several time scales, to analyze the importance of the multitemporal aggregation performed by DSD, but not diffusion distances.

\begin{table}[!htb]
\renewcommand\arraystretch{1.5}
\centering
\begin{tabular}{|
    >{\centering\arraybackslash}m{7.5cm}
    |>{\centering\arraybackslash}m{7.5cm}
    |}
\hline
	Score & Definition \\ \hline 
	Weighted Common Neighbor & $\displaystyle\sum_{x_{k} \in \mathcal{N}(x_{i}) \cap \mathcal{N}(x_{j})} (W_{ik} + W_{jk})$ \\ [10pt]
	Jaccard's Coefficient & $\displaystyle|\mathcal{N}(x_{i}) \cap \mathcal{N}(x_{j})|/|\mathcal{N}(x_{i}) \cup \mathcal{N}(x_{j})|$ \\ [6pt] 
	Weighted Adamic-Adar Index & $\displaystyle\sum_{x_{k} \in \mathcal{N}(x_{i}) \cap \mathcal{N}(x_{j})}1/\log(1 + \displaystyle\sum_{x_{\ell} \in \mathcal{N}(y)} W_{\ell k})$\\\hline
\end{tabular}
\caption{\label{table:link_pred_heuristics} Definitions of different heuristic method scores for two nodes $x_{i},x_{j} \in X$  in an undirected graph $G = (X,W)$ with neighbor sets $\mathcal{N}(x_{i})$ and $\mathcal{N}(x_{j})$ respectively.}
\end{table}
  
Figures \ref{fig:Link_Prediction_DREAM1}, \ref{fig:Link_Prediction_DREAM2} are obtained by averaging the results generated from 100 iterations of sampling uniformly at random subgraphs of 400 nodes from the DREAM1 and DREAM2 networks and generating training and test sets, as described in Algorithm \ref{alg:ConnComp}.  DSD-based  ranking  is compared against the heuristic methods and diffusion distances.  By varying the threshold and mapping the true positive versus the false positive rate of the various methods, we compute the F1 score, which is the harmonic mean of precision and recall.  Similarly, we constructed the partial receiver operating characteristic (ROC) \cite{Fawcett2006_Introduction} and precision-recall curves  from the top 20000 ranked edges for the reduced graphs. The ROC plots show that the top edges predicted by DSD-based ranking are more accurate, compared to both the heuristic methods and the diffusion distance based methods for DREAM1. In DREAM2, the heuristic methods slightly outperform the DSD method. 

Figures \ref{fig:Link_Prediction_DREAM1},  \ref{fig:Link_Prediction_DREAM2} (g-i) shows the results from using approximate DSD computations in which the spectral decomposition is truncated to denoise and reduce computational complexity; see (\ref{eqn:DSD_Embedding}). Recall that we are considering 400 node subgraphs; since the subgraphs are comparatively small, it could be that all the eigenvectors have a significant value, and performing the eigenvector truncation damages accuracy in a significant way.
However, we find that approximate DSD is nearly comparable to exact DSD in performance in DREAM1; in DREAM2, it actually has a denoising effect and improves link prediction performance. For both DREAM1 and DREAM2, approximate DSD significantly decreases run time.

\begin{figure}[!htb]
\centering
\begin{subfigure}[t]{.32\textwidth}
	\captionsetup{width=.95\linewidth}
	\includegraphics[width=\textwidth]{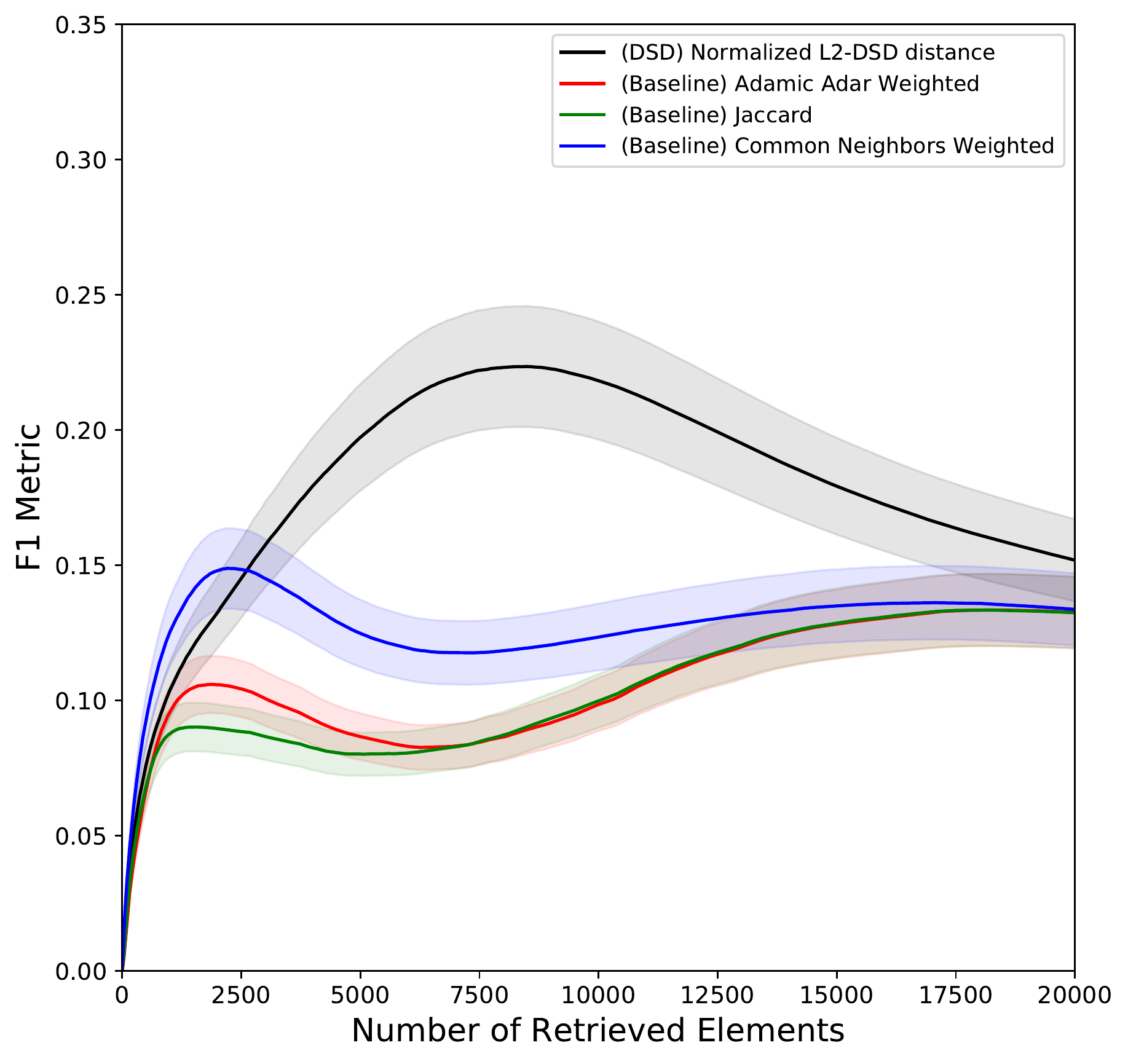}
	\subcaption{F1, heuristics}
\end{subfigure}
\begin{subfigure}[t]{.32\textwidth}
	\captionsetup{width=.95\linewidth}
	\includegraphics[width=\textwidth]{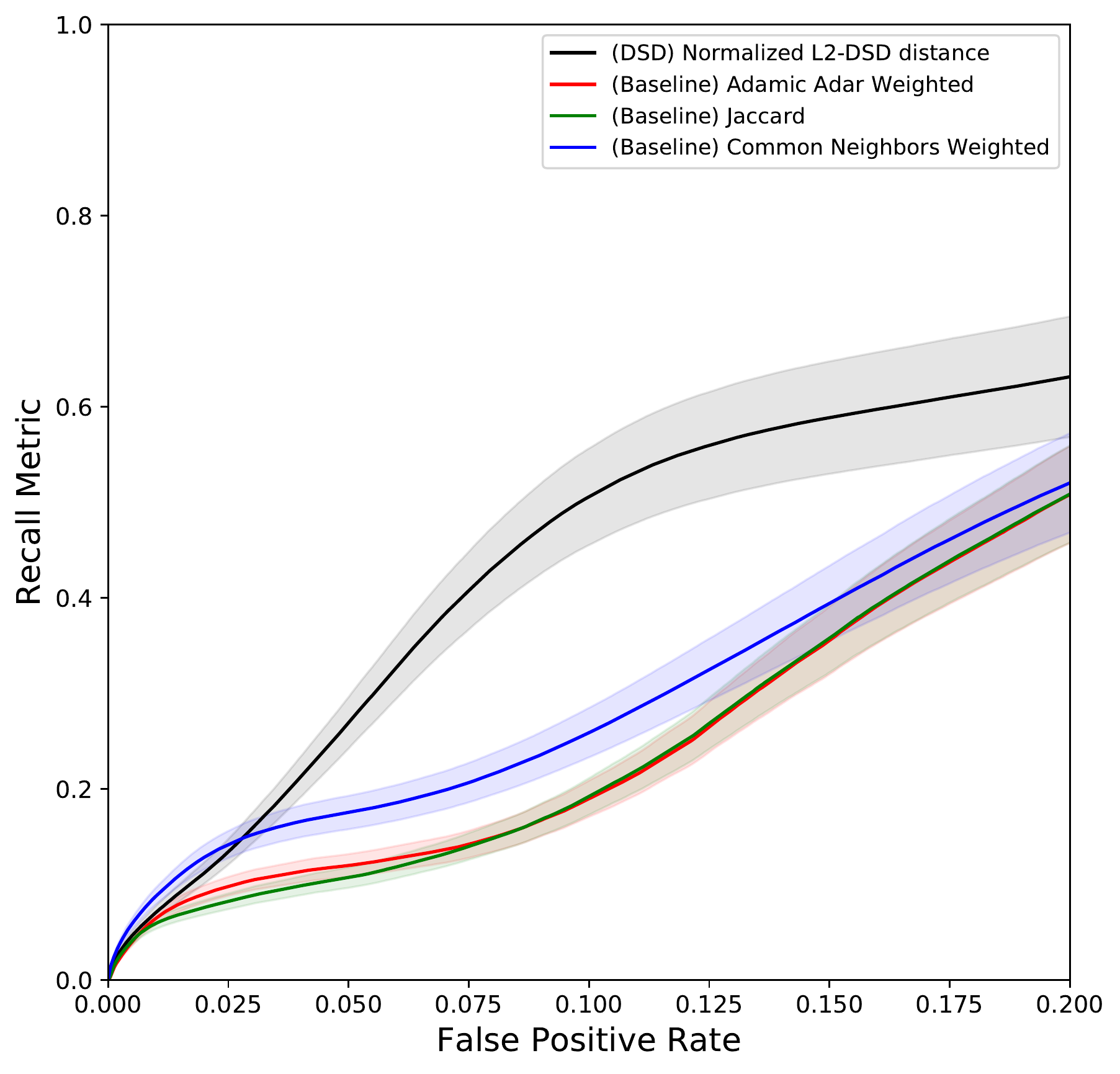}
	\subcaption{ROC, heuristics}
\end{subfigure}
\begin{subfigure}[t]{.32\textwidth}
	\captionsetup{width=.95\linewidth}
	\includegraphics[width=\textwidth]{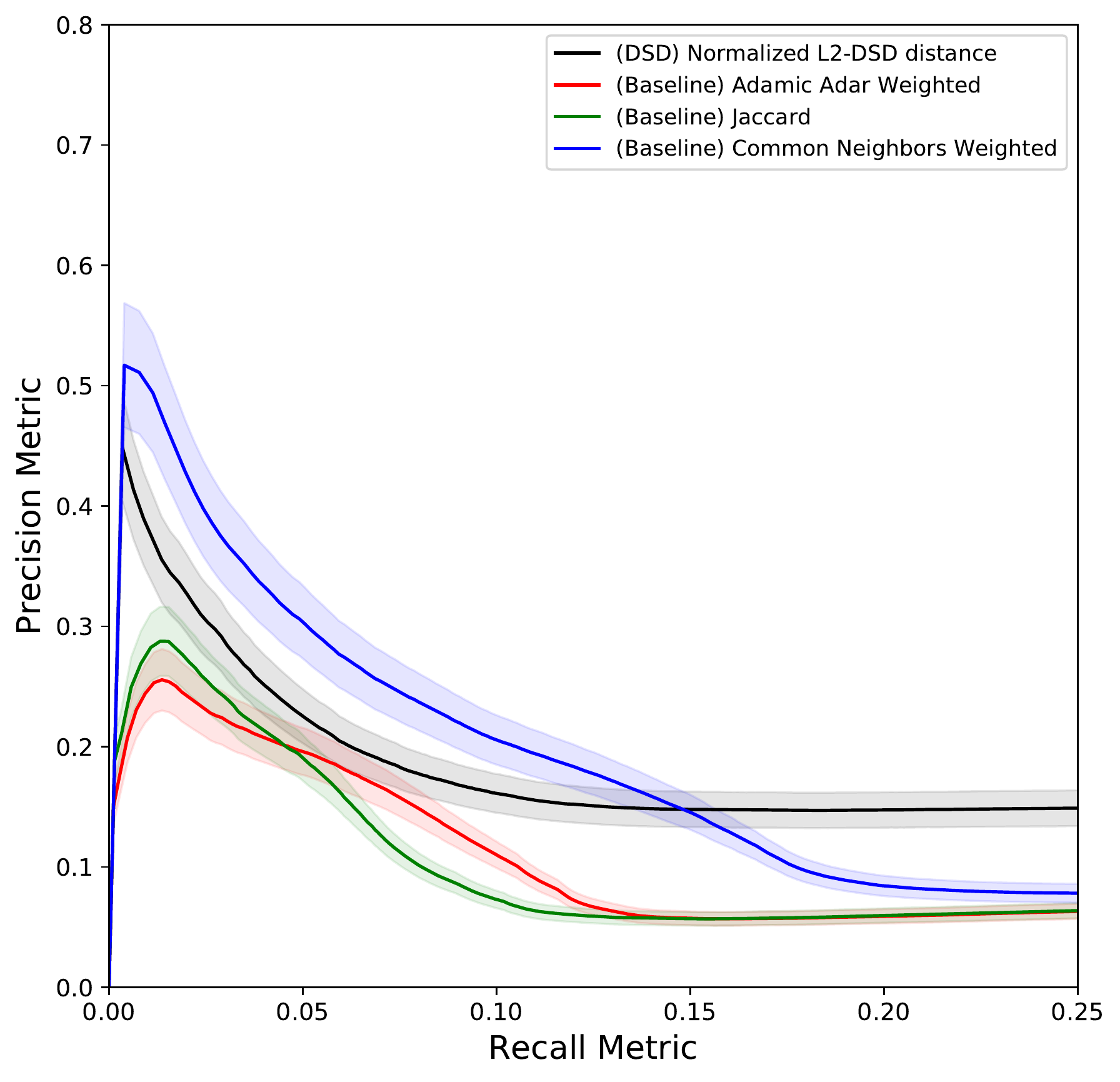}
	\subcaption{Precision-recall, heuristics}
\end{subfigure}
\\~\\
\begin{subfigure}[t]{.32\textwidth}
	\captionsetup{width=.95\linewidth}
	\includegraphics[width=\textwidth]{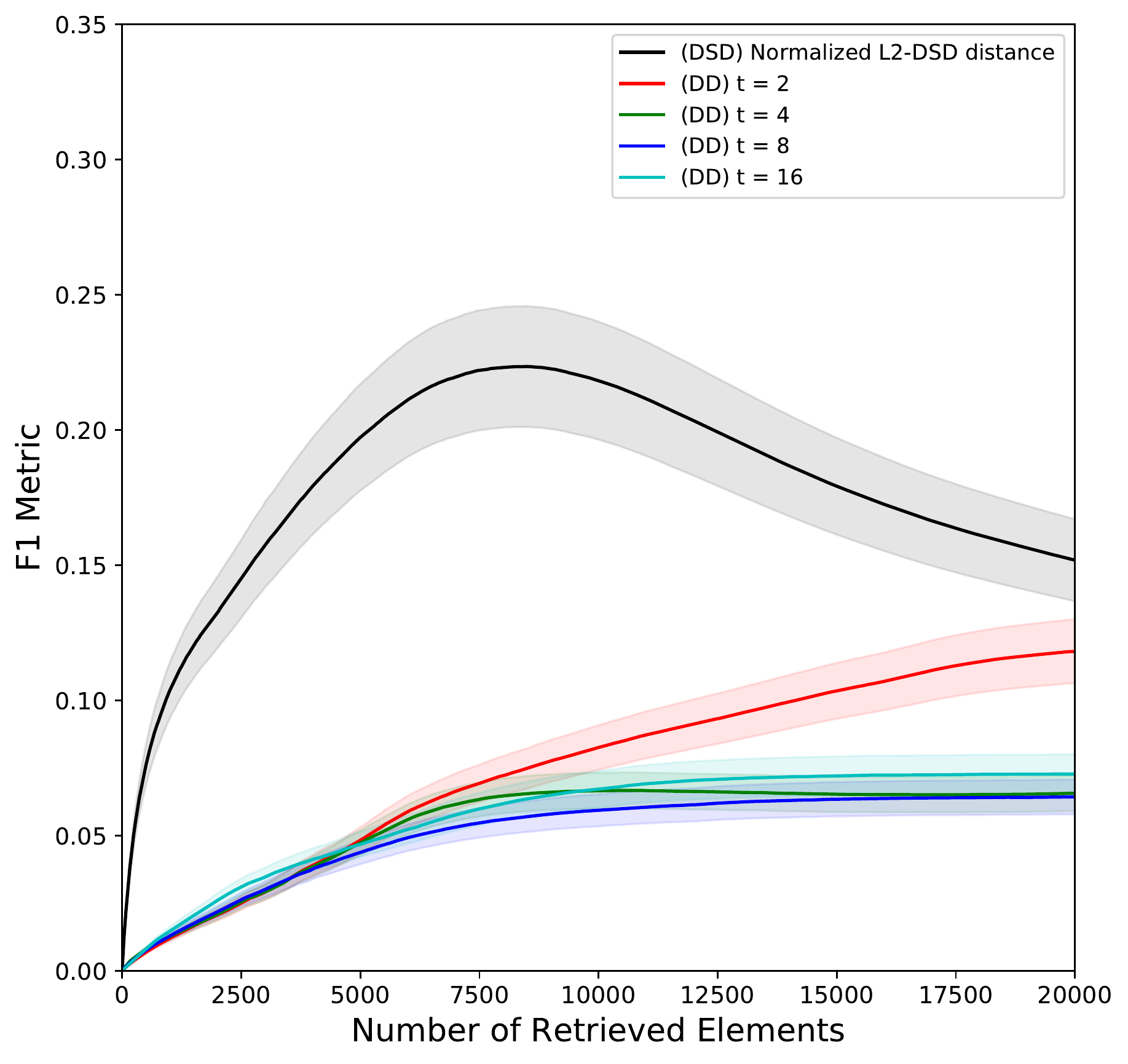}
	\subcaption{F1, diffusion distances}
\end{subfigure}
\begin{subfigure}[t]{.32\textwidth}
	\captionsetup{width=.95\linewidth}
	\includegraphics[width=\textwidth]{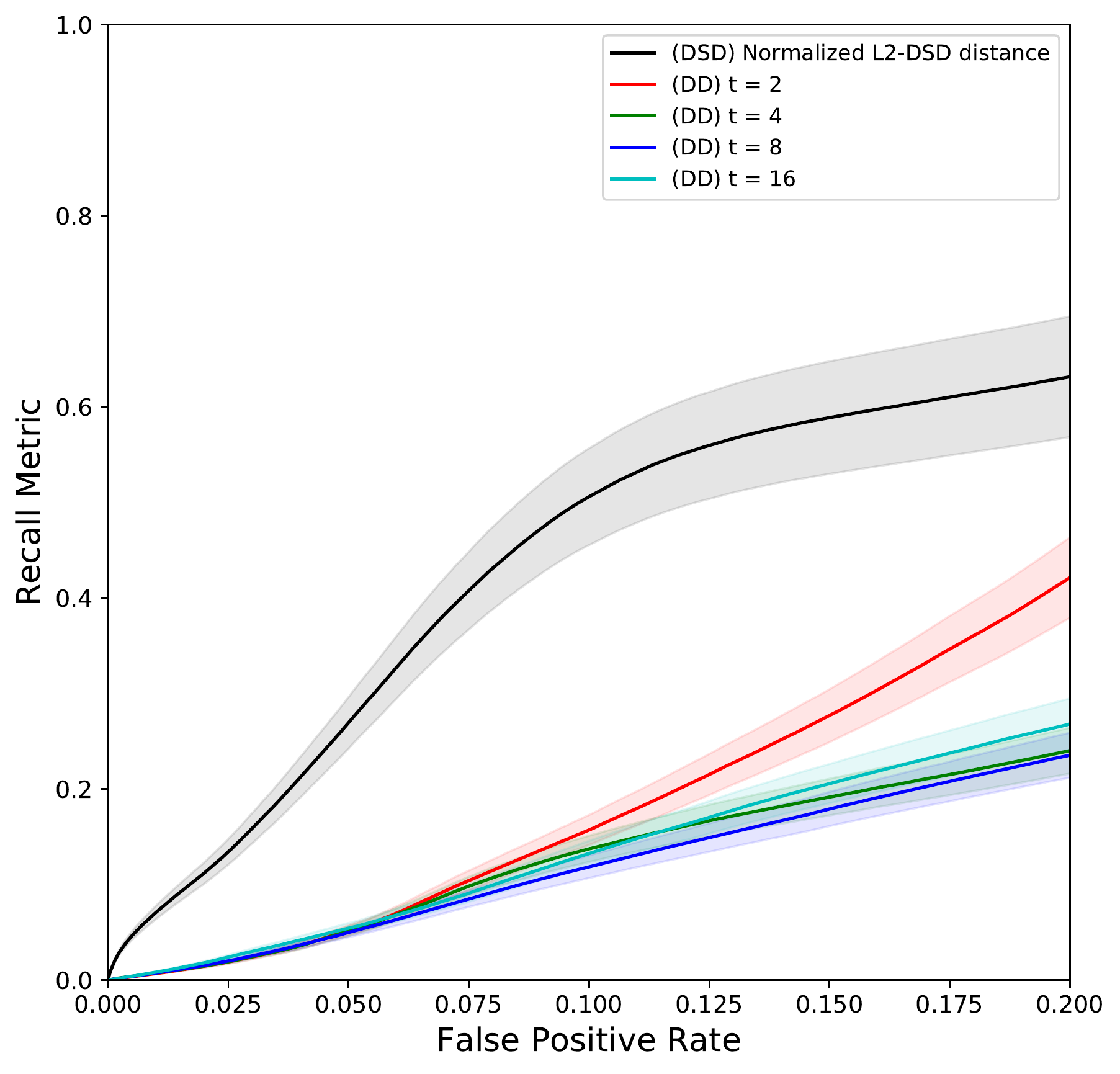}
	\subcaption{ROC, diffusion distances}
\end{subfigure}
\begin{subfigure}[t]{.32\textwidth}
	\captionsetup{width=.95\linewidth}
	\includegraphics[width=\textwidth]{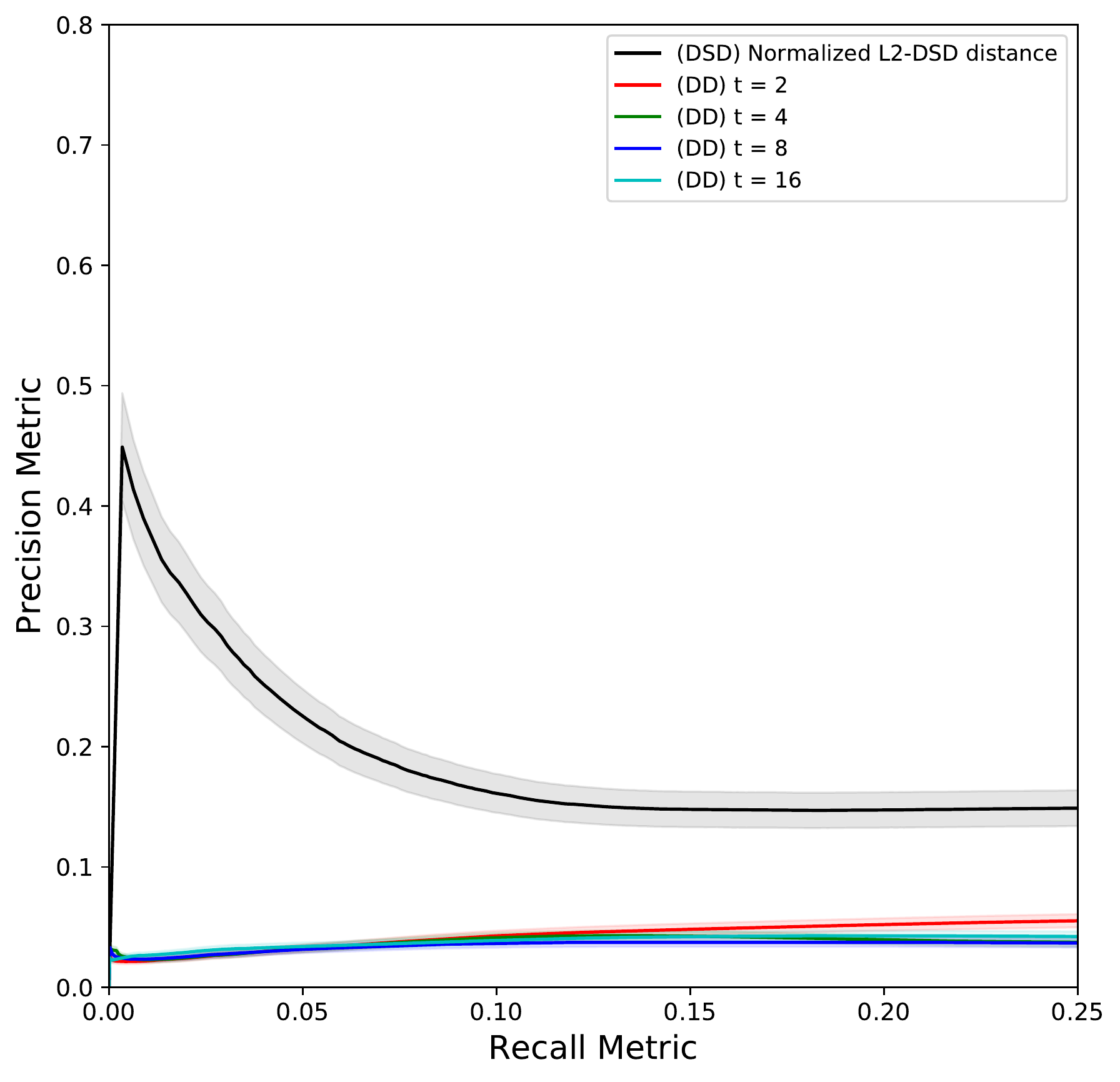}
	\subcaption{Precision-recall, diffusion distances}
\end{subfigure}
\\~\\
\begin{subfigure}[t]{.32\textwidth}
	\captionsetup{width=.95\linewidth}
	\includegraphics[width=\textwidth]{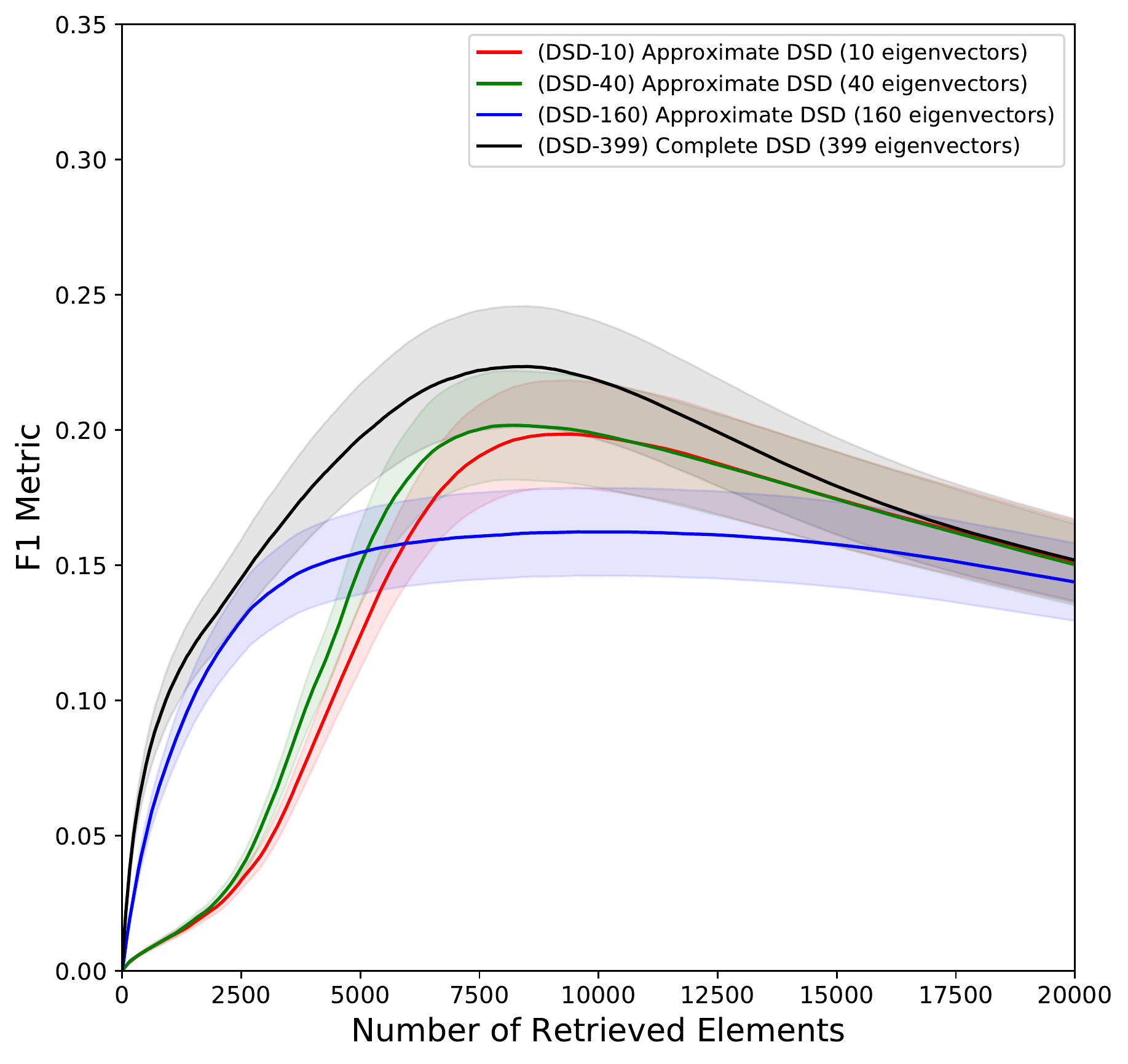}
	\subcaption{F1, Reduced DSD}
\end{subfigure}
\begin{subfigure}[t]{.32\textwidth}
	\captionsetup{width=.95\linewidth}
	\includegraphics[width=\textwidth]{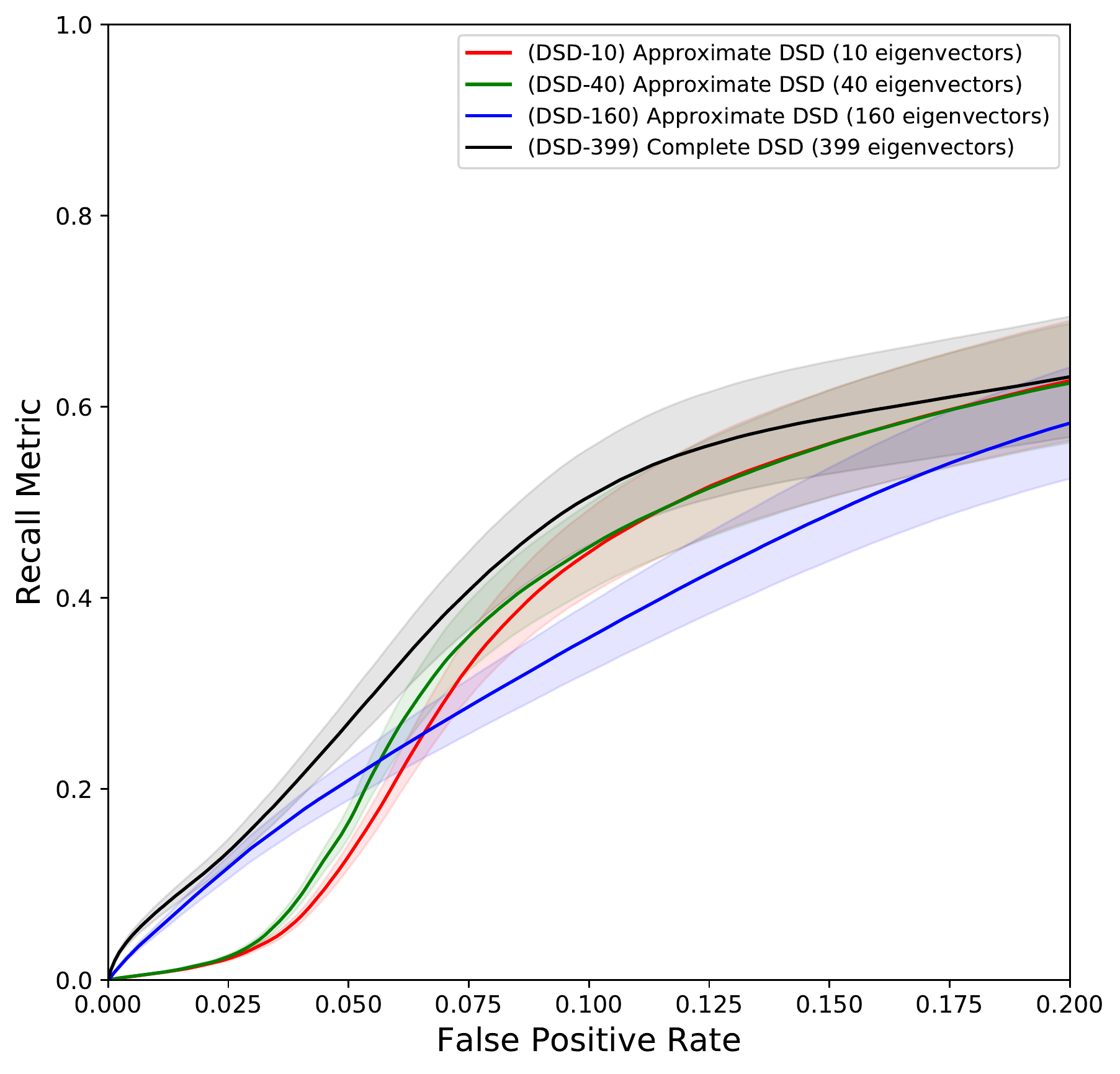}
	\subcaption{ROC, Reduced DSD}
\end{subfigure}
\begin{subfigure}[t]{.32\textwidth}
	\captionsetup{width=.95\linewidth}
	\includegraphics[width=\textwidth]{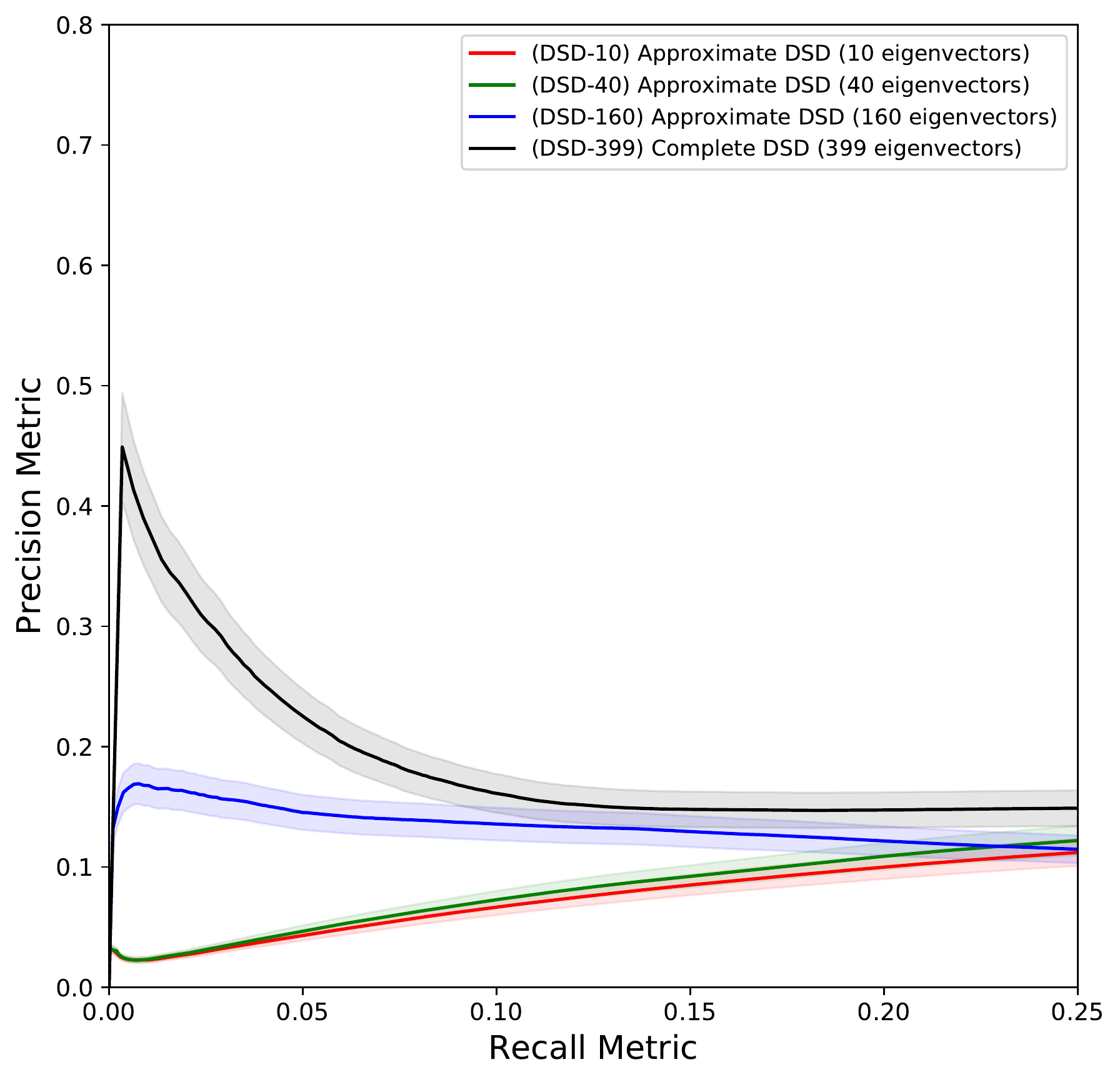}
	\subcaption{Precision-recall, Reduced DSD}
\end{subfigure}
\caption{ \label{fig:Link_Prediction_DREAM1} \emph{(a), (b), (c):}  Comparison of DSD to the different heuristic ranking methods on 100 randomly sampled graphs from DREAM1, each having 400 nodes.  On all of the average F1, ROC and precision-recall curves, we see DSD outperforms the heuristic methods.  \emph{(d), (e), (f):} Comparison of DSD to the diffusion distance ranking methods on 100 randomly sampled graphs from DREAM1, each having 400 nodes.  On all of the average F1, ROC and precision-recall curves, we see DSD outperforms the diffusion distances for an exponential range of time scales.  This indicates the importance of aggregating time scales when consider DREAM1. \emph{(g), (h), (i):} Comparison of DSD to the approximate DSD methods on 100 randomly sampled graphs from DREAM1, each having 400 nodes. Limiting the number of eigenvectors used significantly decreased the run time while not substantially degrading the empirical performance.  \textbf{Note:} The shaded regions demarcate plus and minus one standard deviation of the score.}
\end{figure}

\begin{figure}[!htb]
\centering
\begin{subfigure}[t]{.32\textwidth}
	\captionsetup{width=.95\linewidth}
	\includegraphics[width=\textwidth]{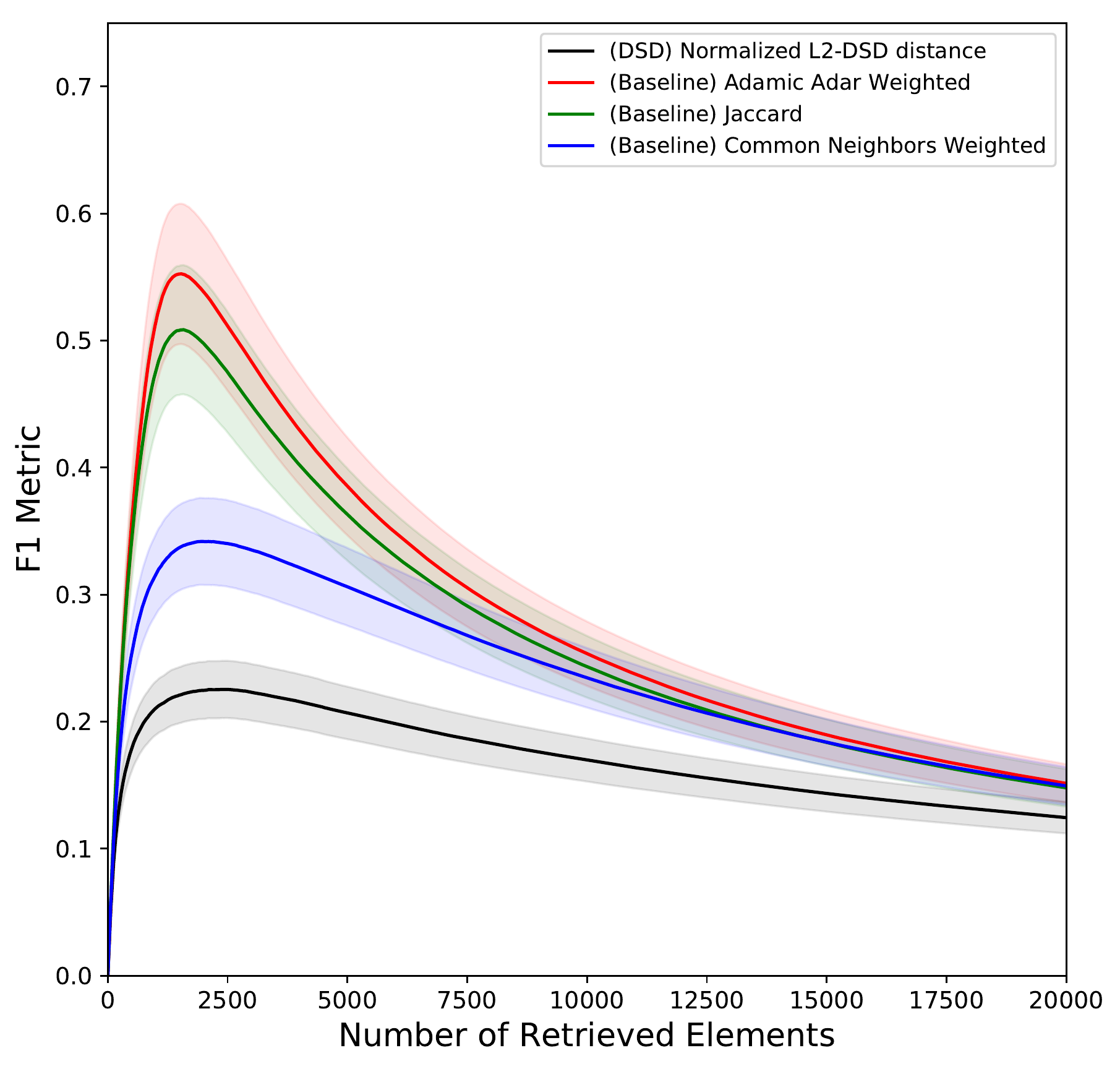}
	\subcaption{F1, heuristics}
\end{subfigure}
\begin{subfigure}[t]{.32\textwidth}
	\captionsetup{width=.95\linewidth}
	\includegraphics[width=\textwidth]{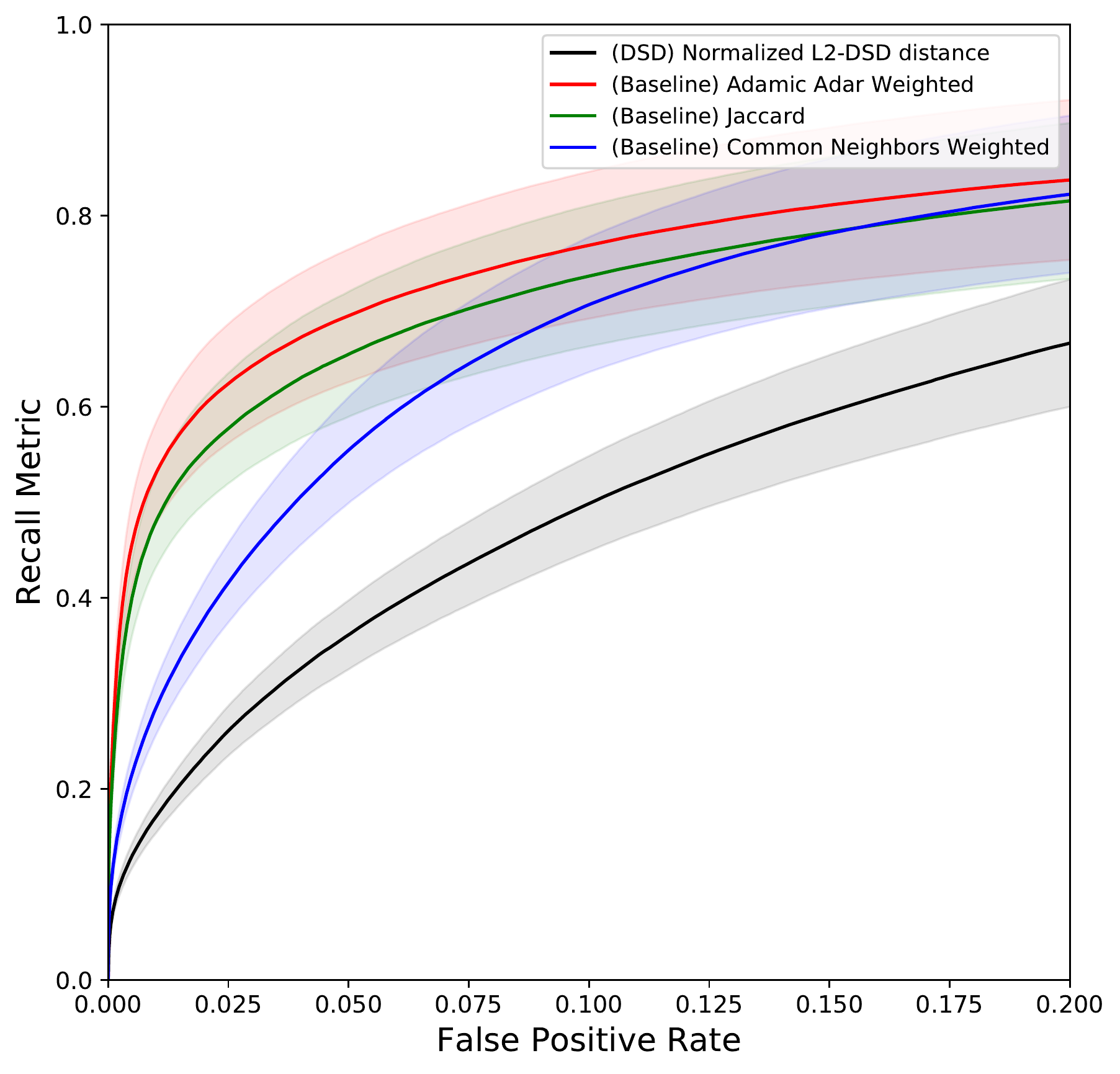}
	\subcaption{ROC, heuristics}
\end{subfigure}
\begin{subfigure}[t]{.32\textwidth}
	\captionsetup{width=.95\linewidth}
	\includegraphics[width=\textwidth]{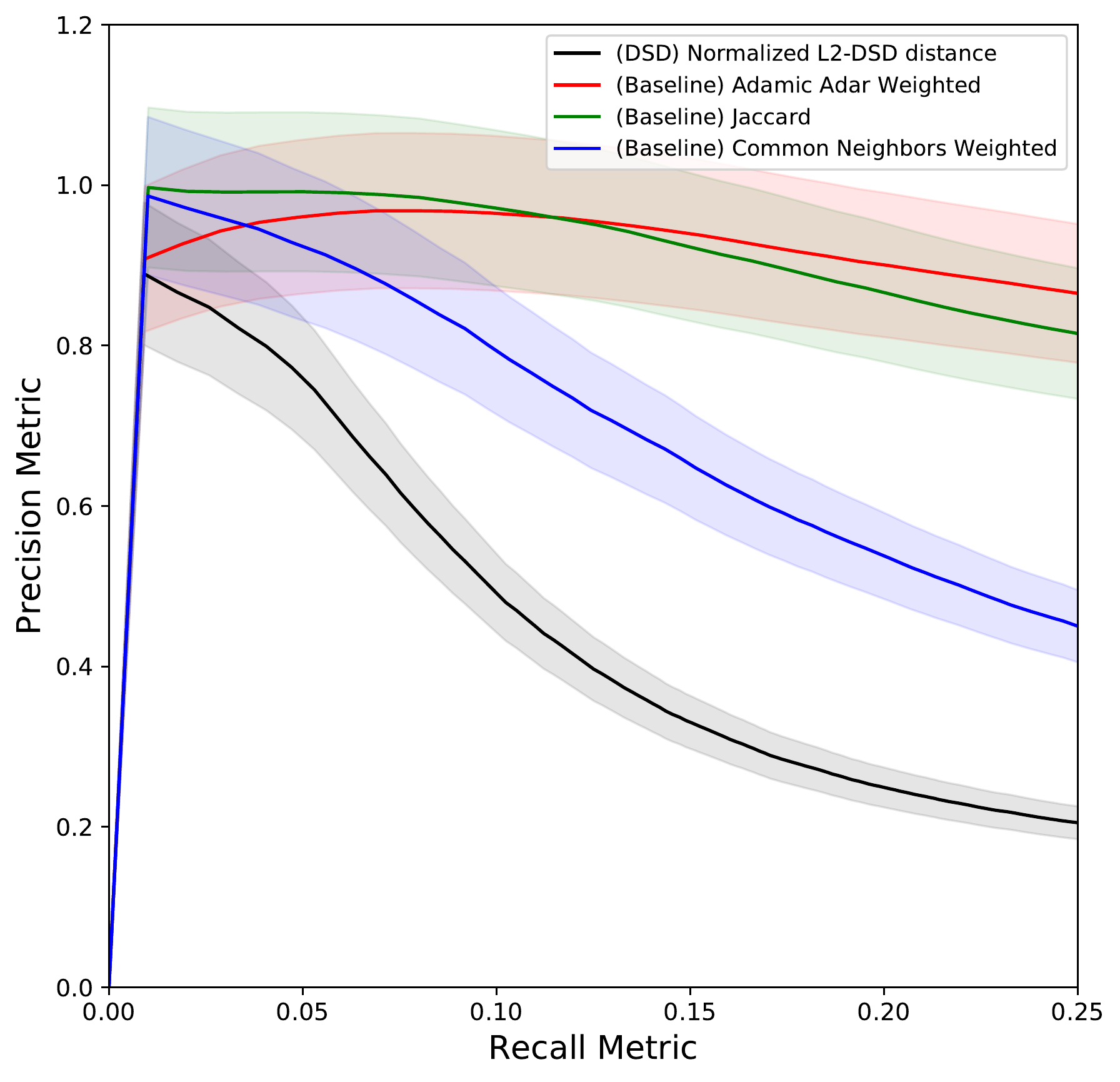}
	\subcaption{Precision-recall, heuristics}
\end{subfigure}
\\~\\
\begin{subfigure}[t]{.32\textwidth}
	\captionsetup{width=.95\linewidth}
	\includegraphics[width=\textwidth]{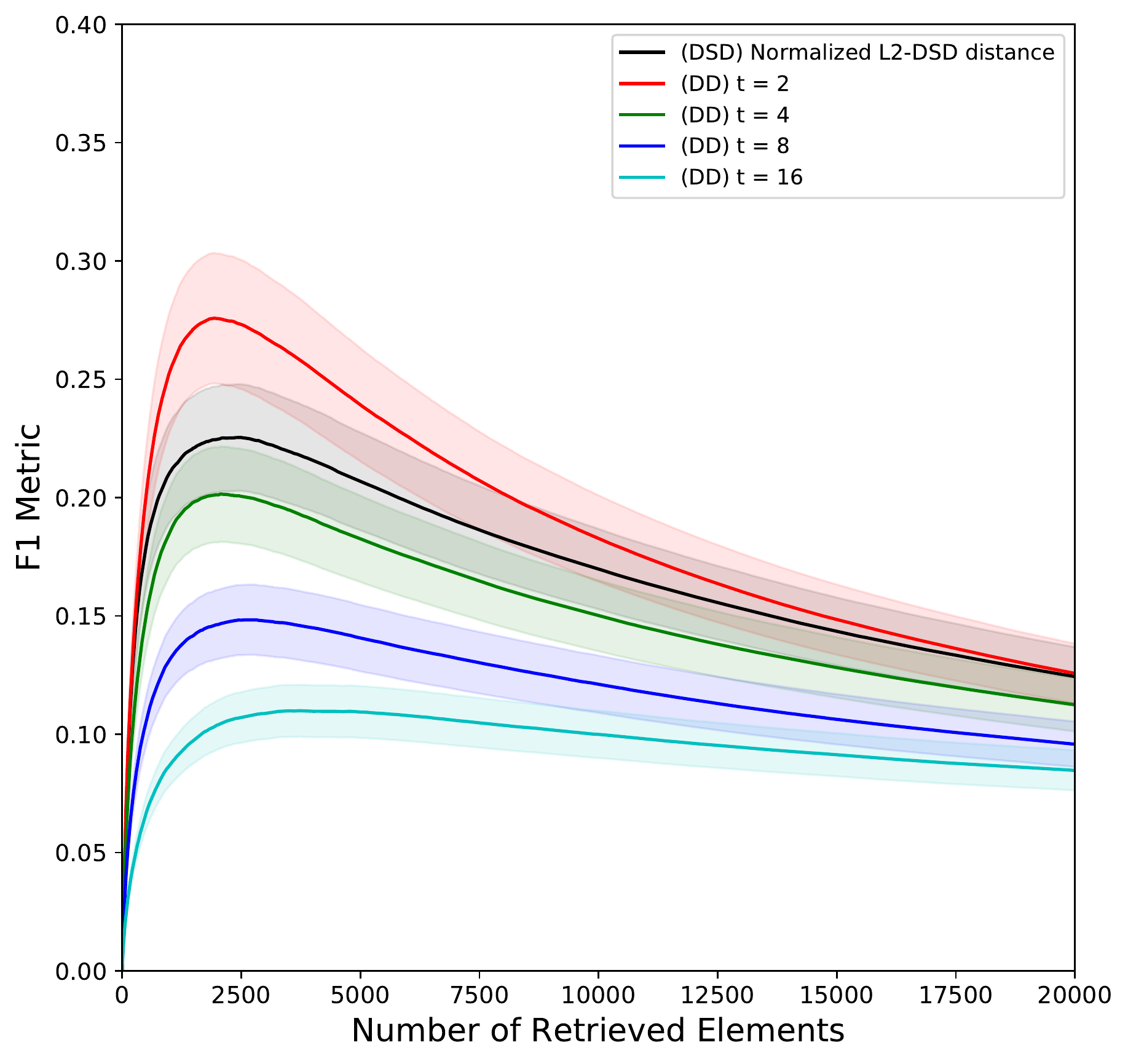}
	\subcaption{F1, diffusion distances}
\end{subfigure}
\begin{subfigure}[t]{.32\textwidth}
	\captionsetup{width=.95\linewidth}
	\includegraphics[width=\textwidth]{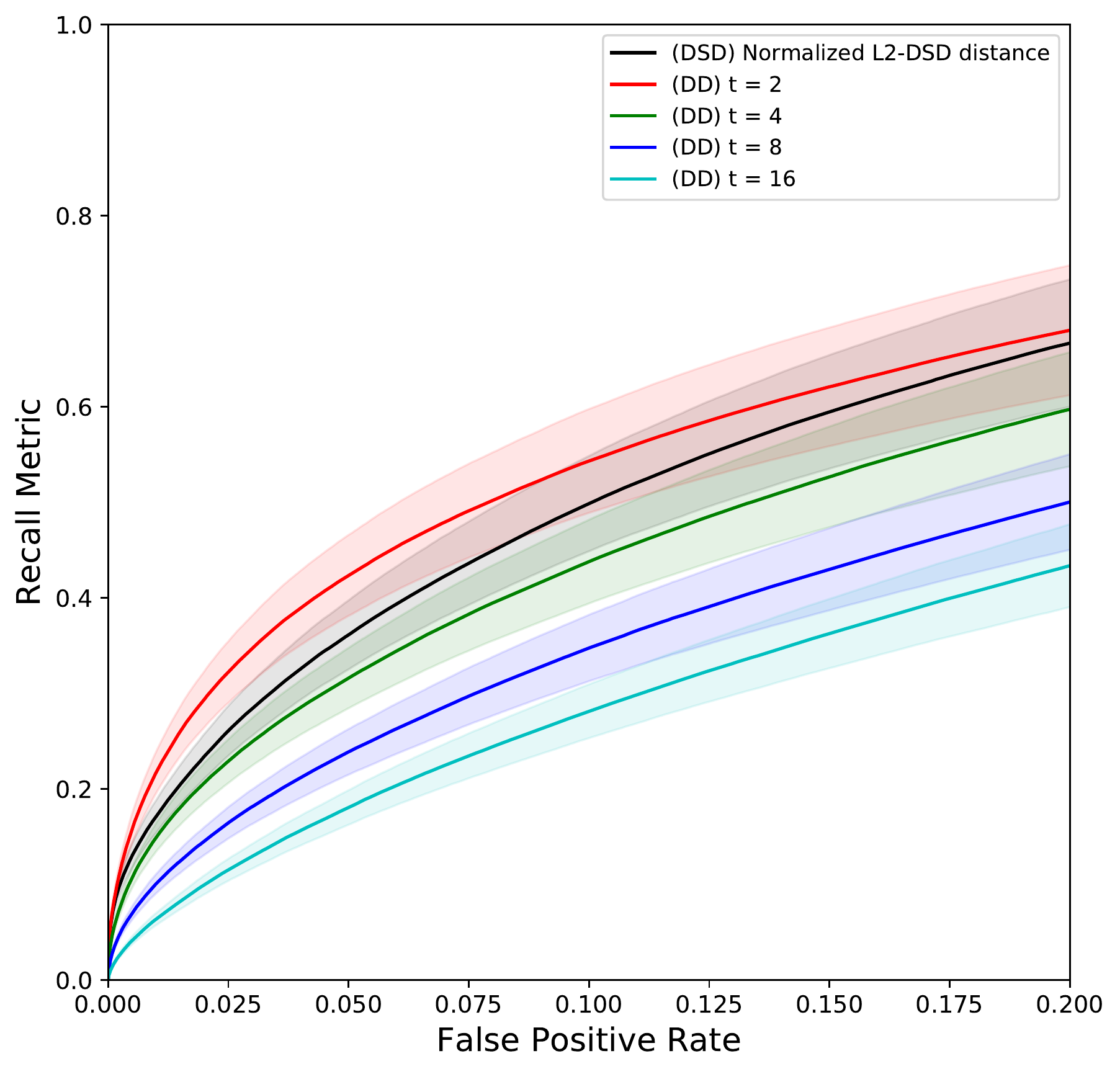}
	\subcaption{ROC, diffusion distances}
\end{subfigure}
\begin{subfigure}[t]{.32\textwidth}
	\captionsetup{width=.95\linewidth}
	\includegraphics[width=\textwidth]{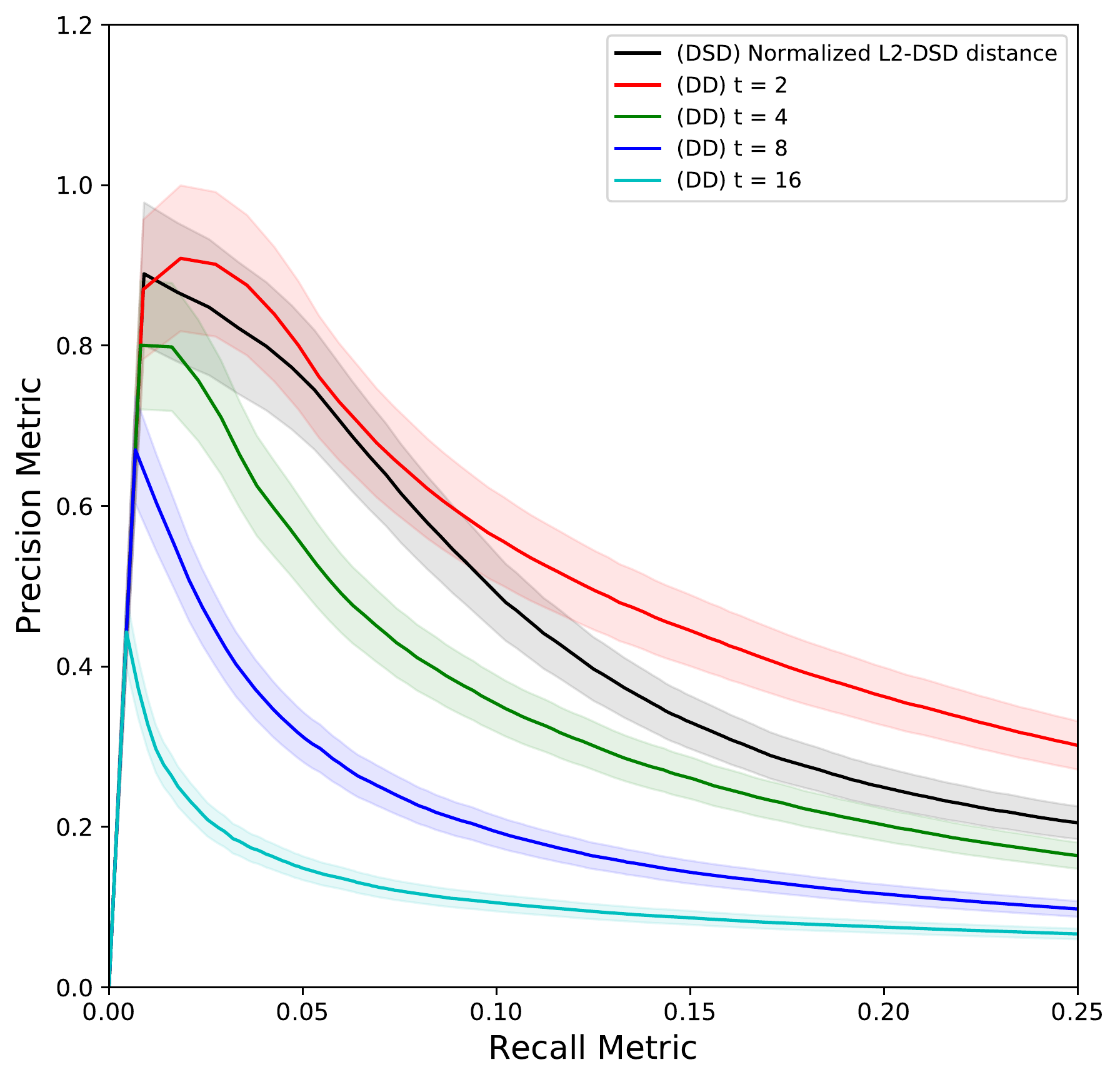}
	\subcaption{Precision-recall, diffusion distances}
\end{subfigure}
\\~\\
\begin{subfigure}[t]{.32\textwidth}
	\captionsetup{width=.95\linewidth}
	\includegraphics[width=\textwidth]{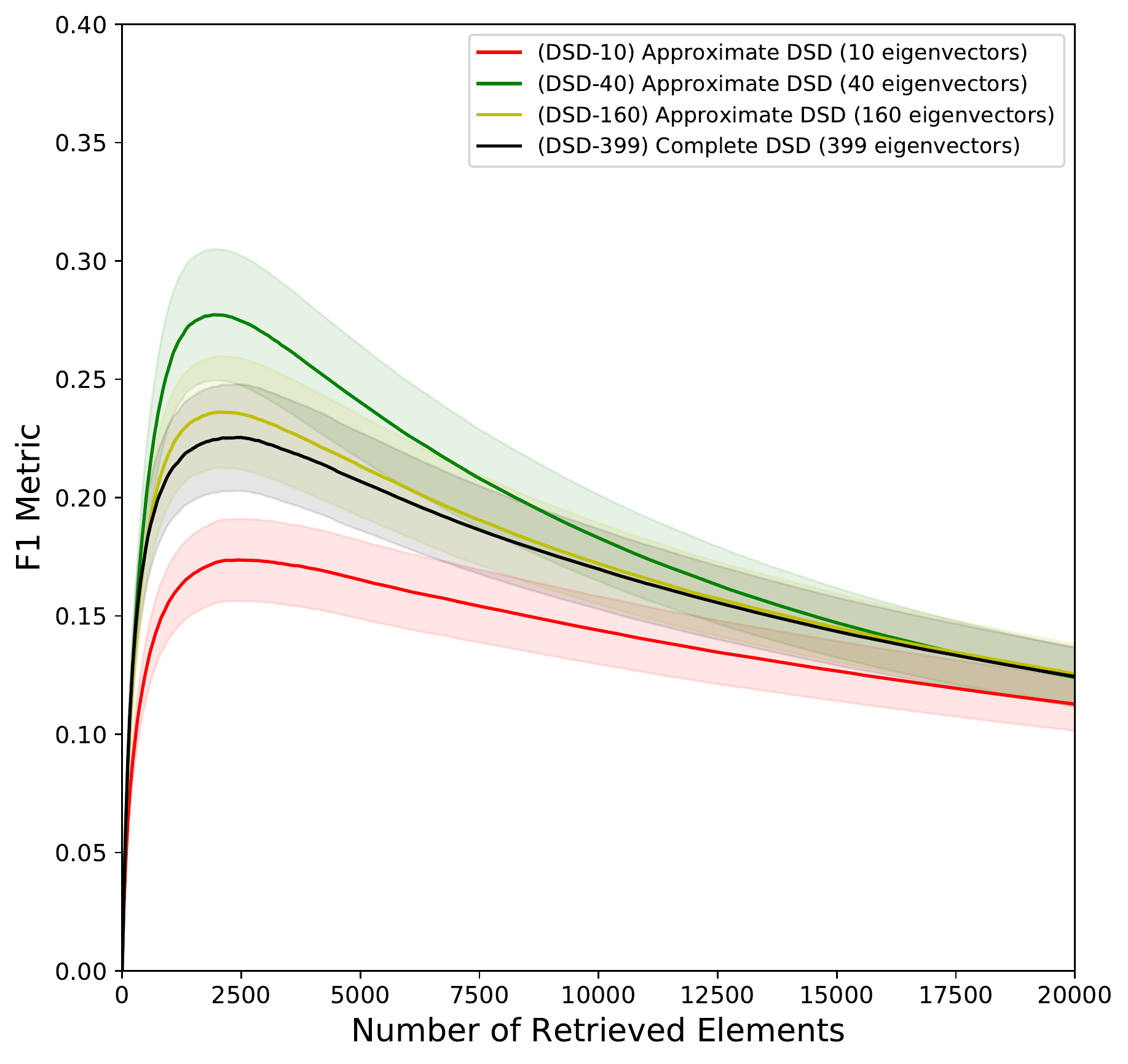}
	\subcaption{F1, Reduced DSD}
\end{subfigure}
\begin{subfigure}[t]{.32\textwidth}
	\captionsetup{width=.95\linewidth}
	\includegraphics[width=\textwidth]{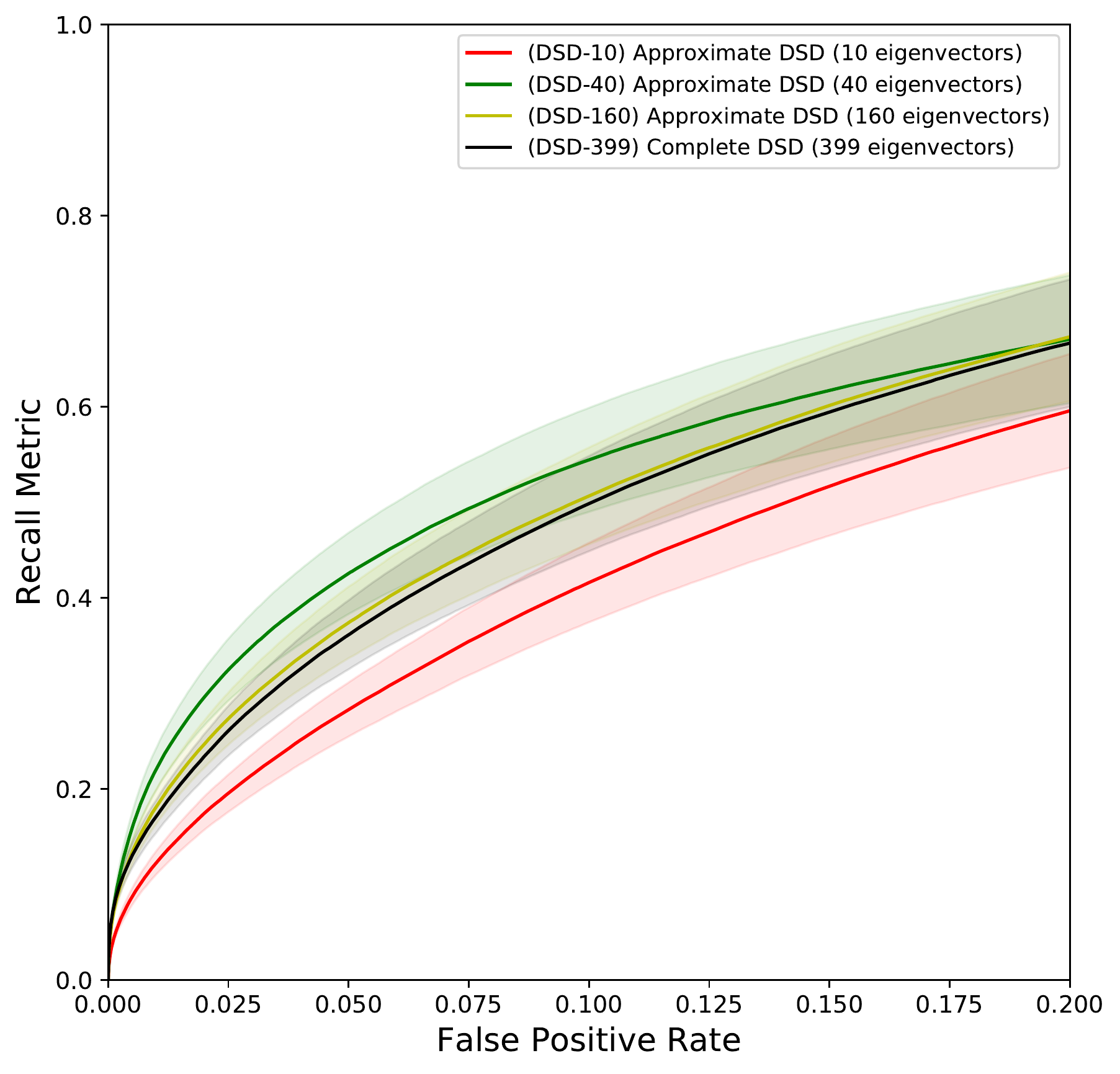}
	\subcaption{ROC, Reduced DSD}
\end{subfigure}
\begin{subfigure}[t]{.32\textwidth}
	\captionsetup{width=.95\linewidth}
	\includegraphics[width=\textwidth]{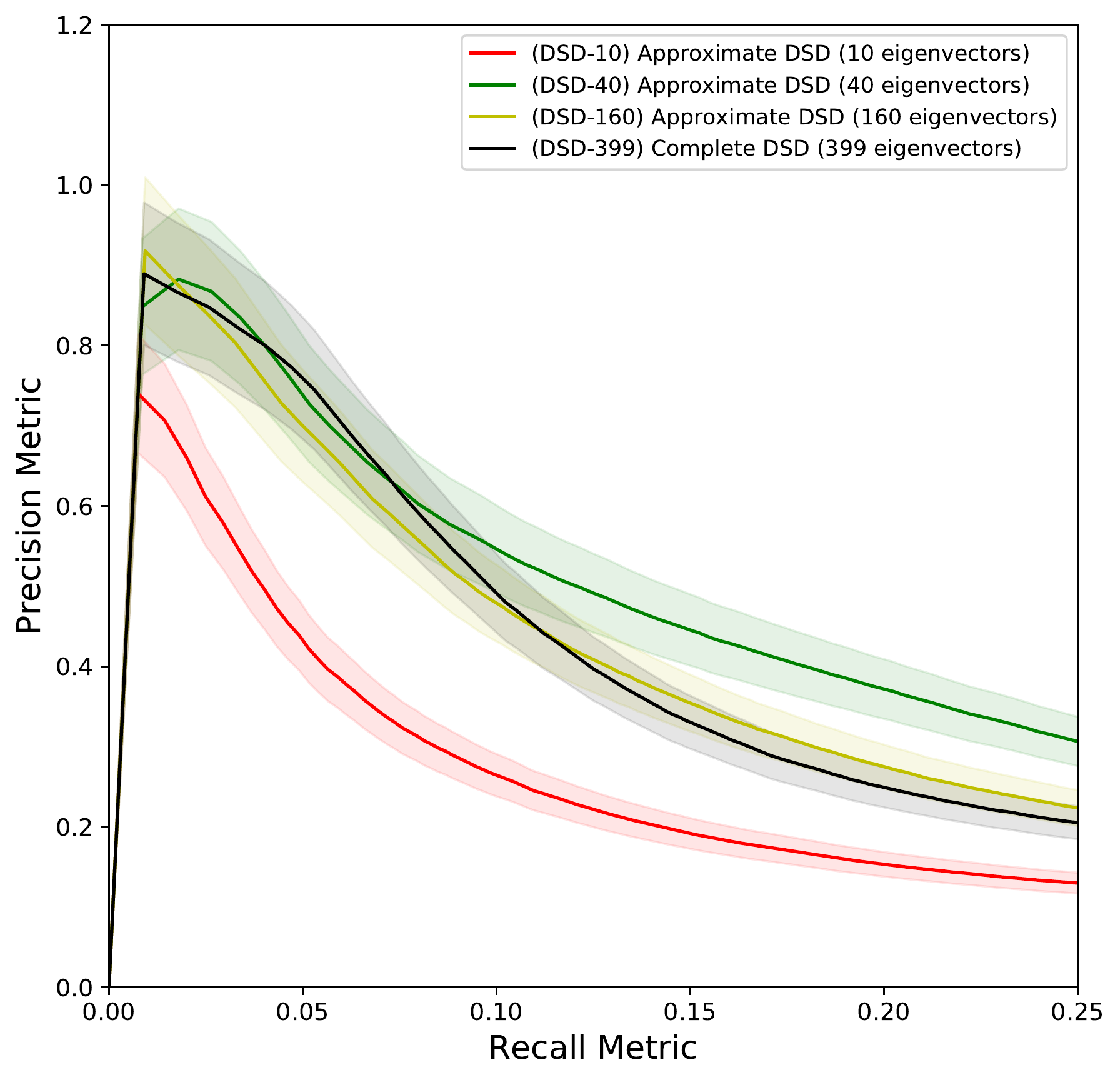}
	\subcaption{Precision-recall, Reduced DSD}
\end{subfigure}
\caption{ \label{fig:Link_Prediction_DREAM2} \emph{(a), (b), (c):}  Comparison of DSD to the different heuristic ranking methods on 100 randomly sampled graphs from DREAM2, each having 400 nodes.  On all of the average F1, ROC and precision-recall curves, we see DSD is outperformed by the heuristic methods.  \emph{(d), (e), (f):} Comparison of DSD to the diffusion distance ranking methods on 100 randomly sampled graphs from DREAM2, each having 400 nodes.  In terms of average F1, ROC and precision-recall curves, we see DSD mostly outperforms the diffusion distances on an exponential range of time scales.  This indicates the importance of aggregating time scales when consider DREAM2. \emph{(g), (h), (i):} Comparison of DSD to the approximate DSD methods on 100 randomly sampled graphs from DREAM2, each having 400 nodes. Limiting the number of eigenvectors improves results while also lowering computational complexity. \textbf{Note:} The shaded regions demarcate plus and minus one standard deviation of the score.}
\end{figure}

\subsubsection{Function Prediction}

Function prediction is a  classical topic in computational biology. While we know the roles of some proteins in the cell, protein-protein association networks can assist in making predictions for functional roles of unknown proteins, since proteins of similar function should cluster (under the appropriate metric) in the network. DSD was introduced in the context of function prediction ~\cite{Cao2013_Going}, and it was shown that a majority vote for the most popular function among the $k$ closest neighbors of a vertex in DSD distance performed well in predicting the function of that vertex in a yeast protein-protein interaction network. Here, we evaluate DSD for function prediction in the DREAM1 and DREAM2 human networks. Functional labels are taken from the popular Gene Ontology (GO) database ~\cite{gene2019gene} (downloaded from FuncAssociate3.0~\cite{berriz2009next} on 02/12/19), where we restrict the GO labels we consider in our cross-validation experiments to those that are neither too general nor too specific: in particular, we consider all labels from the Biological Process hierarchy and Molecular Function hierarchy that annotate between 100 and 500 proteins in the  networks.  

On the largest connected components of the biological networks, we use a five-fold cross validation to demonstrate the accuracy rate of function prediction of Gene Ontology (GO) labels based on the DSD distance metric computed as in Algorithm \ref{alg:DSD}. For each node in the testing set, we generate its $k$ nearest neighbors ($k$NN) based on the DSD metric, then use majority voting weighted by $\frac{1}{\DSD (x_{i}, \, x_{j}) }$ to predict its function; note that other weighting functions could be used, and the selection of optimal weighting functions is the subject of ongoing research. The  most frequently voted functional label is chosen as the predicted function label, and it is marked correct if this is one of the functional labels assigned to the node in the test set. We report the percent of nodes assigned a correct functional label as the accuracy score. This is summarized in Algorithm~\ref{alg:fun_pred}.

\

\begin{algorithm}[htp!]
	\caption{Function prediction with majority voting} \label{alg:fun_pred}
	\begin{algorithmic}[1]
	    \FOR{every node $x$ in the testing set} 
	    \STATE Collect its $k$ nearest neighbors $\{v_i\}_{i=1}^{k}$ based on DSD metric (we use $k=10$ in our experiments). 
	    \STATE Each neighbor that is also in the training set votes for its labels, weighted by $\frac{1}{\DSD (x, \, v_{i}) }$. 
	    \STATE Use the most voted label as the predicted function label $f(x)$. 
	    \STATE Mark the node correct if $f(x)$ is one of the labels of $x$. 
	    \ENDFOR
		\STATE  Compute the accuracy score as the percentage of nodes that get a correct functional label. 
	\end{algorithmic}
\end{algorithm}

\

Next, we investigate the accuracy and efficiency of function prediction on the same networks with approximate DSD computed by Algorithm \ref{alg:approx-DSD} and compare the numerical result to the one computed with exact DSD, using $k = 10$ nearest neighbors.  In our experiments, we use the MATLAB build-in eigensovler \texttt{eigs}, which is based on ARPACK~\cite{lehoucq1998arpack}.

In Figure \ref{fig: func_pred_approx_and_exact_DSD} the accuracy rate of function prediction of DREAM1 and DREAM2 network based on approximated DSD with different numbers of eigenvectors is compared to the accuracy result based on exact DSD (the red horizontal line). Function prediction with the exact or approximated DSD metric demonstrates a substantial increase in accuracy  compared to the best-studied baseline method, majority vote~\cite{Cao2013_Going}, where all direct neighbors in the original graph vote for their functional labels with equal weight for each label (the grey horizontal line). Compared to the exact DSD metric, the accuracy of the  approximated DSD metric initially increases as we include more eigenvectors, and it almost always outperforms exact DSD for both DREAM1 and DREAM2 networks, and also for both the molecular function and biological process GO hierarchies.  After that, the prediction accuracy decreases in all cases when we include the high-frequency eigenvectors, which are possibly corrupted by noise. We note that the approximate DSD methods that we employ are not immune to numerical error, which prevents the approximate DSD from perfectly converging to the exact DSD when all the eigenvalues are used (note the numerical instability at the far right of each of the graphs of Figure \ref{fig: func_pred_approx_and_exact_DSD}).
These numerical errors in the eigendecomposition depend on the machine epsilon, but also on the condition number of the matrix, the magnitude of the eigenvalues, as well as the stopping criterion for the eigensolver, and are most acute for the highest frequency eigenvectors. Thus when we truncate the 
number of eigenvectors used to the lowest frequency eigenvectors, the numerical instability is much less of a factor, and the same picture would likely be observed if we were computing exact DSD projected to the lower-dimensional eigenspace. 

\begin{figure}[!htb]
\centering
\begin{subfigure}[t]{.24\textwidth}
	\captionsetup{width=.95\linewidth}
	\includegraphics[width=\textwidth]{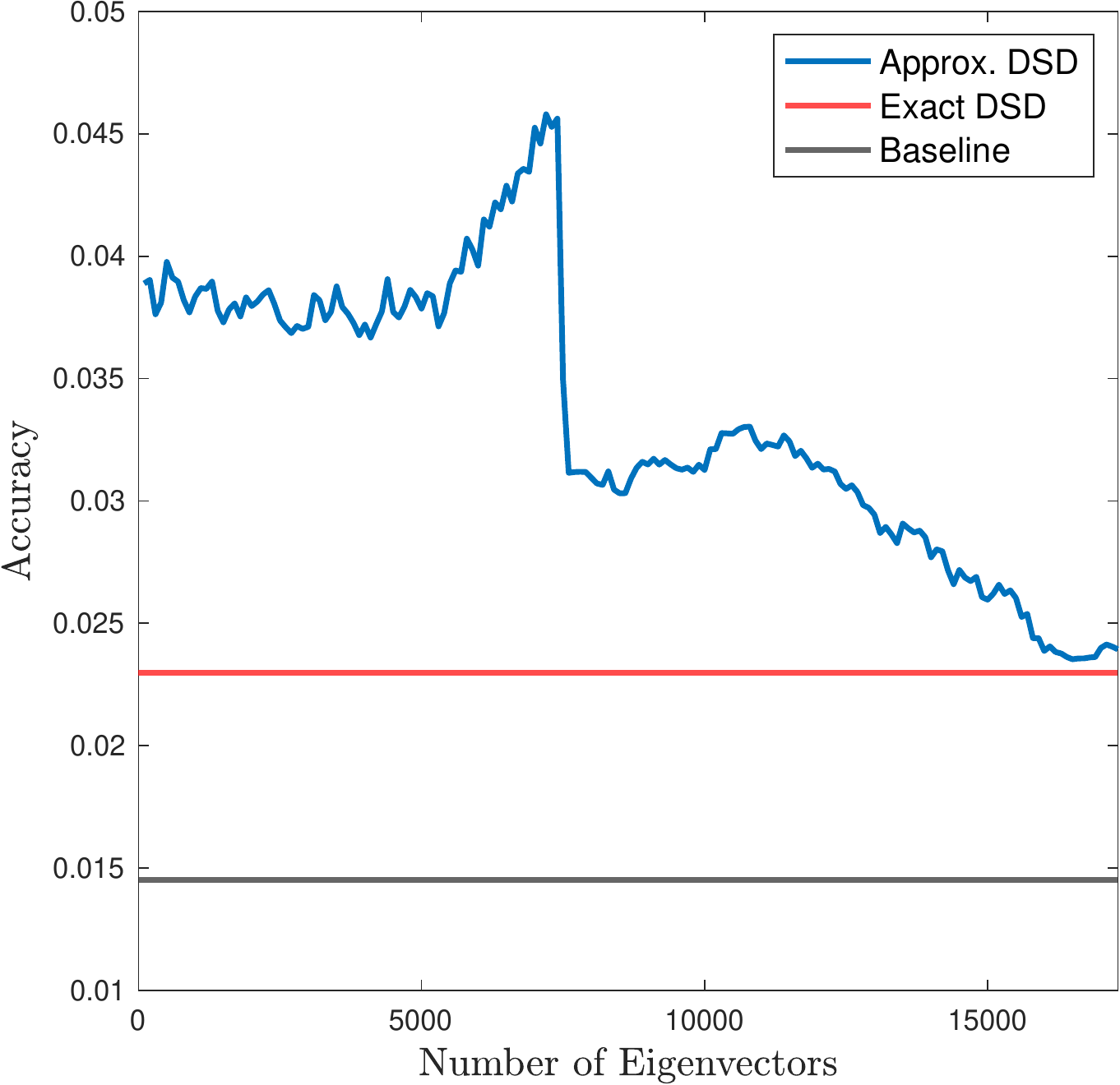}
	\subcaption{DREAM1, Biological Process Labels}
\end{subfigure}
\begin{subfigure}[t]{.24\textwidth}
	\captionsetup{width=.95\linewidth}
	\includegraphics[width=\textwidth]{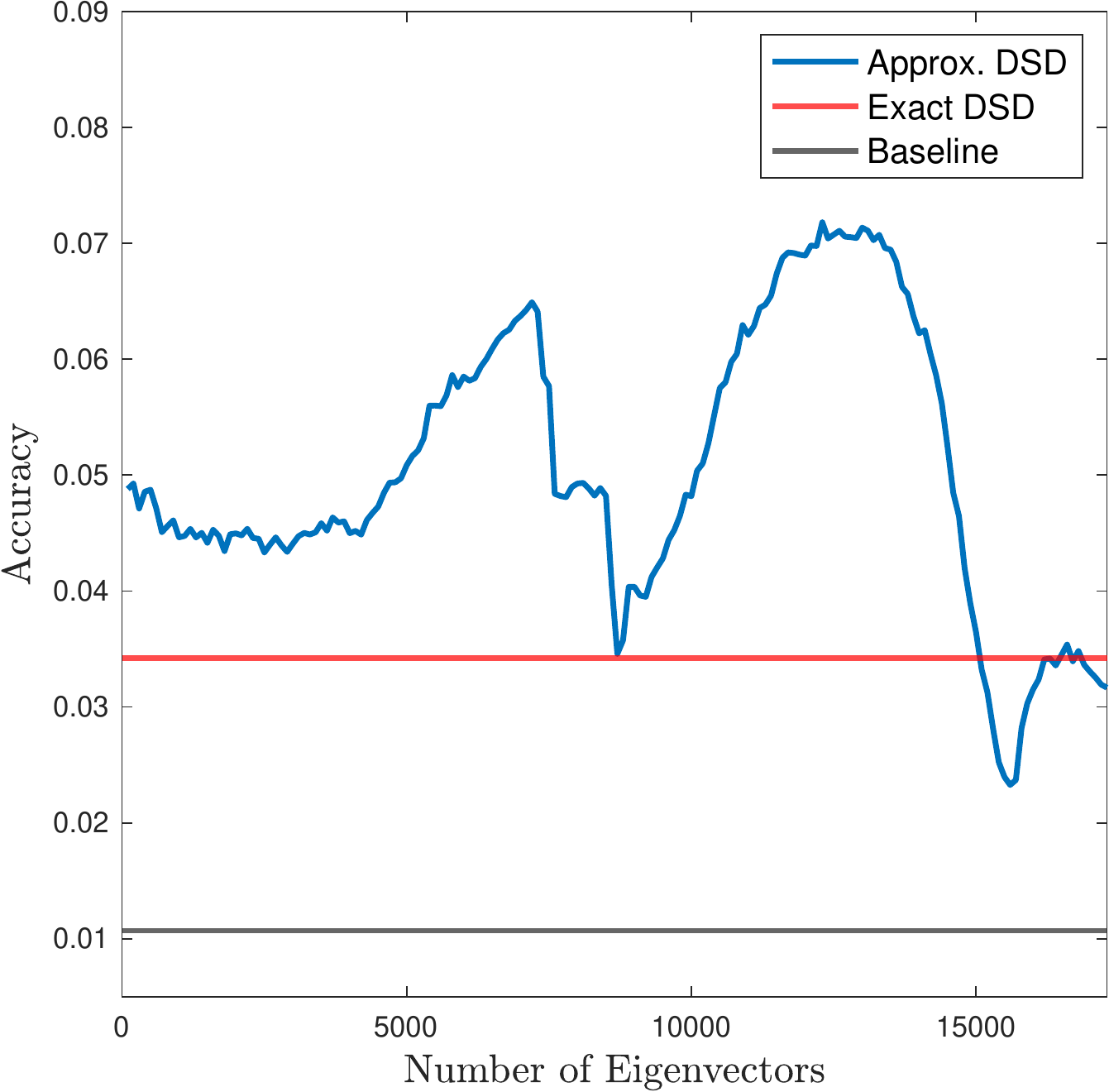}
	\subcaption{DREAM1, Molecular Function Labels}
\end{subfigure}
\begin{subfigure}[t]{.24\textwidth}
	\captionsetup{width=.95\linewidth}
	\includegraphics[width=\textwidth]{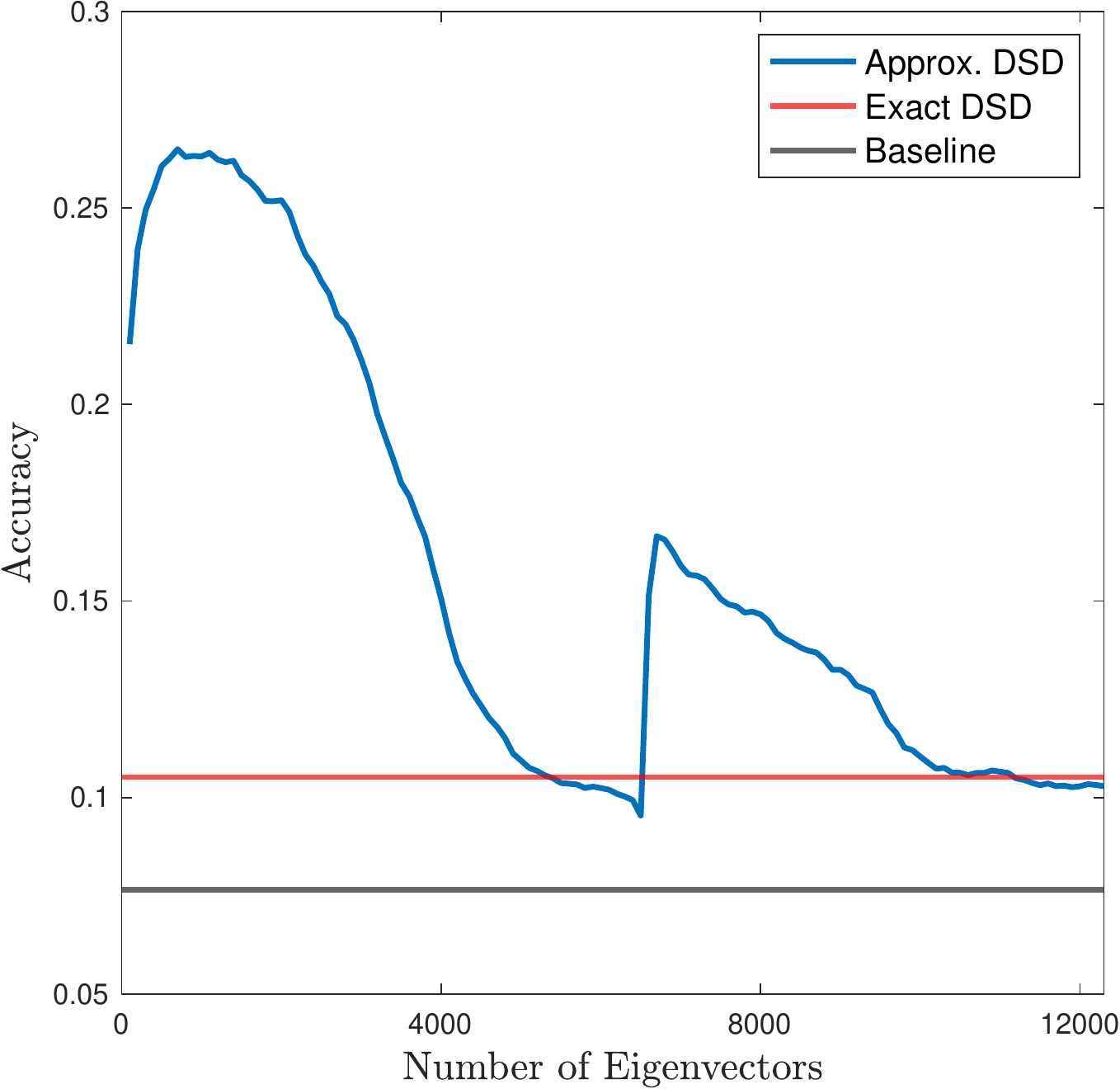}
	\subcaption{DREAM2, Biological Process Labels}
\end{subfigure}
\begin{subfigure}[t]{.24\textwidth}
	\captionsetup{width=.95\linewidth}
	\includegraphics[width=\textwidth]{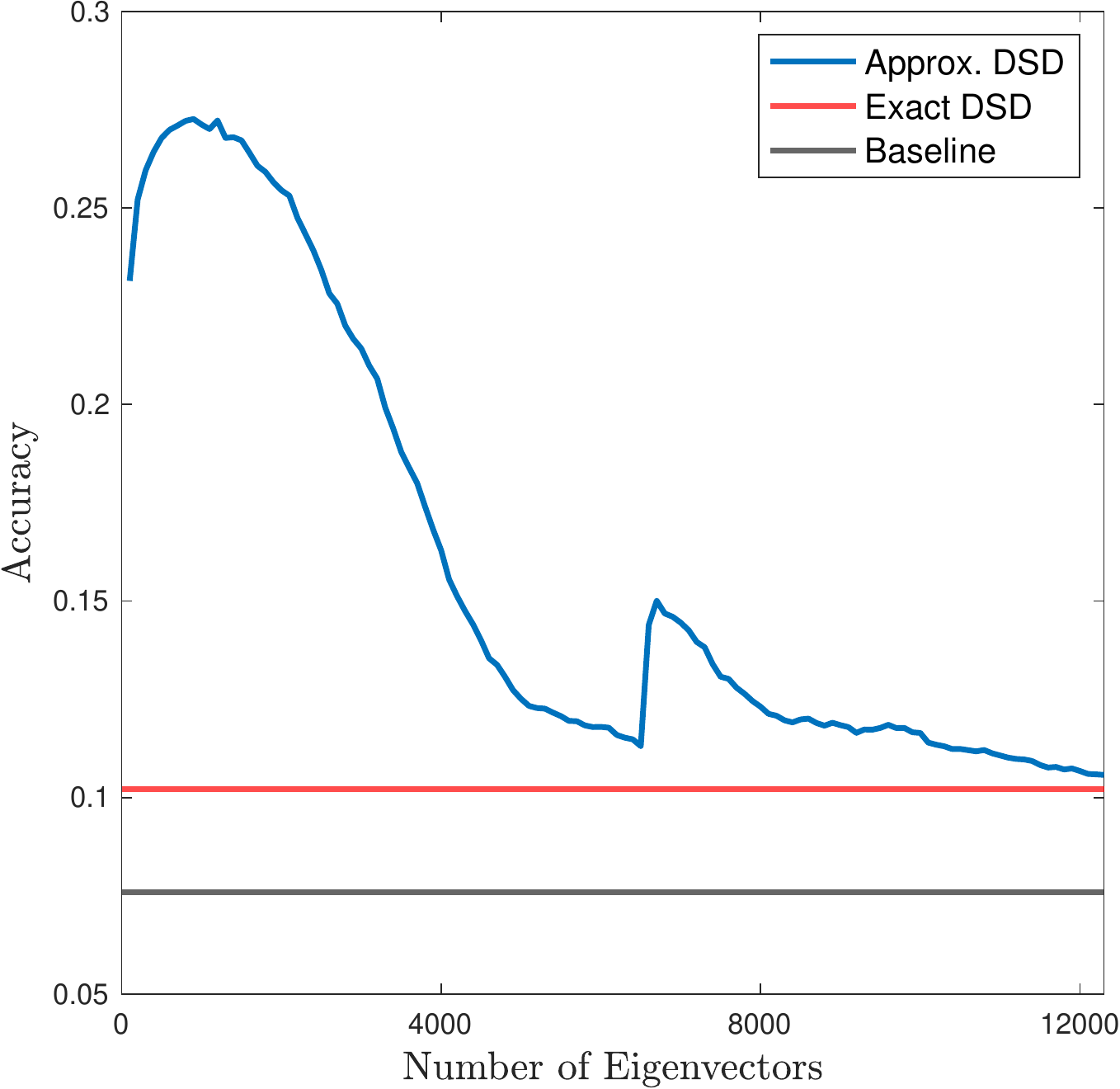}
	\subcaption{DREAM2, Molecular Function Labels}
\end{subfigure}
\caption{Percentage of accuracy of function prediction for DREAM1 and DREAM2 networks using GO Biological Process hierarchy and Molecular Function hierarchy labels by exact DSD and approximated DSD. The improved function prediction results illustrate the efficiency in denoising the DSD metric by using only the top eigenvectors.  We note that both exact and approximate DSD-based function prediction strongly outperform the baseline method.
\label{fig: func_pred_approx_and_exact_DSD}
}
\end{figure}

From Figure \ref{fig: func_pred_approx_and_exact_DSD} the performance of function prediction peaks when we restrict to the first 6000 eigenvectors to compute the approximate DSD for DREAM1 network in the Biological Process hierarchy, and there is a peak also at the first 6000 eigenvectors but a slightly higher second peak at around 12000 eigenvectors for the Molecular Function hierarchy. For DREAM 2, restricting to the first 1000 eigenvectors for both gives the best performance.  With only these eigenvectors we are able to compute an accurate low dimensional embedding of the DSD metric, reflected by the function prediction performance in Figure \ref{fig: func_pred_approx_and_exact_DSD}.
We compare the CPU time for both approaches in Table \ref{tab: DREAM_comp_time_compare} and it is observed that computing the exact DSD metric is much more computationally expensive.
\begin{table}[!htb]
\centering
\begin{tabular}{|c|c|c|}
\hline
          & Exact DSD & Approximate DSD  \\ \hline
DREAM1 &  17084.17         & 3090.918                    \\ \hline
DREAM2 &   2759.05        &     87.049                   \\ \hline
  
\end{tabular}
\caption{CPU time in seconds used to computed exact DSD and approximate DSD for each DREAM network. The approximate DSD for DREAM1 is computed with the first 6000 eigenvectors and the one for DREAM2 uses the first 1000 eigenvectors. We see that the computation of the approximate DSD is much faster, since only a small number of eigenvectors of the underlying diffusion matrix (or equivalently, of the symmetric normalized graph Laplacian) need to be computed.}
\label{tab: DREAM_comp_time_compare}
\end{table}

\section{Conclusions and Future Directions}
\label{sec:FutureDirections}

This article provides a multitemporal theory for diffusion state distances, and suggests exploiting spectral decompositions of the underlying random walk to reduce computational complexity and denoise the metric.  Experiments on synthetic and real data illustrate that the DSD is typically more effective for clustering, link prediction, and function predictions tasks, compared to classical graph-based metrics and also to heuristic methods used for biological network analysis.  Moreover, the spectral formulation of DSD is shown not only to improve efficiency, but to improve accuracy in link and function prediction by denoising.  

The analysis of Section \ref{sec:MultitemporalAnalysis} is for the discrete random walk matrix $P$.  A natural question is to understand the connection between the discrete DSD and a notion of continuum DSD, involving the continuous inverse Laplacian \cite{Liben2007_Link}.  Indeed, the underlying diffusion matrix corresponds, in the large sample limit, to a Fokker-Planck SPDE, and it is of interest to develop continuum formulations of the DSD.  A related question is to understand how the key parameters of Section \ref{sec:MultitemporalAnalysis} change hierarchically, that is, how $\delta_{r}, \lambda_{r}^{*}, \kappa_{r}$ change across scales.  As $r$ increases and the clustering coarsens, $\delta_{r}$ is expected to decrease, but $\lambda_{r}^{*}$ should increase, since the random walk needs to mix on larger clusters.  Analyzing the precise tradeoff in these parameters is the subject of ongoing research.  

As discussed in Section \ref{subsec:LargeSampleLimits}, commute distances, related to DSD, degenerate into highly localized notions of distance as $n\rightarrow\infty$.  Our method is most interesting when there is small gap between the first and second eigenvalue of the underlying diffusion process.  It is of interest to consider new models of random graphs for which this gap grows smaller as $n\rightarrow\infty$ at such a rate so that the analysis of Section \ref{sec:MultitemporalAnalysis} has a meaningful continuum analogue.  

From the computational biology perspective, while random-walk based measures of proximity have long been understood to illuminate relationships in the networks that are studied~\cite{cowen2017network}, DSD-distance measures in particular have recently been increasingly adopted in a variety of settings, from de-noising networks for the study of cancer driver genes~\cite{hristov2017network} to the method employed for hierarchical clustering in the latest version of the STRING network~\cite{szklarczyk2019string}, as well as featuring prominently in the DREAM Disease Module challenge itself as both a stand-alone method as well as the basis of a novel consensus method to integrate diverse different clusterings~\cite{Choobdar2019_Assessment}.

Link prediction and function prediction remain two of the most important inference problems for biological networks. In addition to inference based on single sources of network data,  as described in this paper, how best to take advantage of multiplex networks remains an active area of research~\cite{yu2015predicting,Choobdar2019_Assessment}. It would also be interesting to explore further the approximate DSD measures described in this paper, and understand how to automatically choose the right number of dimensions, based simple parameters of network structure, or network density.  

Moreover, DSD-based link prediction performs somewhat poorly on the DREAM2 network, compared to its performance in link prediction on DREAM1, and its function prediction performance.  This is hypothesized to be due to the especially dense core of DREAM2.  Developing methods that interpolate between the heuristic methods (which perform well on the dense core) and DSD-based methods (which perform well off the dense core) is the topic of ongoing research.

\section*{Acknowledgements}

This research was partially supported by NSF grants DMS-1812503, DMS-1912737, DMS-1924513, CCF-1934553, and OAI-1937095.  JMM is grateful to Mauro Maggioni and Nicolas Garcia Trillos for interesting and insightful conversations.   

\bibliographystyle{siamplain}
\bibliography{DSD_Theory}

\appendix

\section{Proof of Theorem \ref{thm:NearReducibility}}

\begin{proof}
Notice $\|P^{t}-S^{\infty}\|_{\infty}\le \|P^{t}-S^{t}\|_{\infty}+\|S^{t}-S^{\infty}\|_{\infty}$.  For all $t\ge0$, $P^{t}-S^{t}=\sum_{i=1}^{t}S^{t-i}(P-S)P^{i-1},$ so that $$\|P^{t}-S^{t}\|_{\infty}=\left\|\sum_{i=1}^{t}S^{t-i}(P-S)P^{i-1}\right\|_{\infty}\le \sum_{i=1}^{t}\|S^{t-i}\|_{\infty}\|(P-S)\|_{\infty}\|P^{i-1}\|_{\infty}=t\|(P-S)\|_{\infty}\le t\delta.$$

Now, observe that after diagonalizing $S$, 

\[S^{t}=Z\begin{bmatrix}
    I_{\numclust}  & \0   \\
   \0  & D^{t}  \\
\end{bmatrix}Z^{-1}, \ \
S^{\infty}=Z\begin{bmatrix}
    I_{\numclust}  & \0   \\
   \0  & \0 \\
\end{bmatrix}Z^{-1},\]
where  $D$ is a diagonal matrix with $\lambda_{\numclust+1}, \lambda_{\numclust+2}, \dots, \lambda_{n}$ on the diagonal.  We may thus estimate \[\|S^{t}-S^{\infty}\|_{\infty}\le \|Z\|_{\infty}\lambda_{\numclust+1}^{t}\|Z^{-1}\|_{\infty}=\kappa\lambda_{\numclust+1}^{t},\] as desired.  The second result of the theorem follows similarly.  
\end{proof}

\end{document}